\documentclass{article}

\usepackage{microtype}
\usepackage{graphicx}
\usepackage{subfigure}
\usepackage{booktabs} 

\usepackage{amsmath}
\usepackage{amsthm}
\usepackage{amssymb}
\usepackage{bm}
\usepackage{graphicx}
\usepackage{overpic}
\usepackage{color}
\usepackage{array}
\usepackage[rightcaption]{sidecap}
\sidecaptionvpos{figure}{c}
\usepackage{enumitem}

\usepackage{dsfont}
\usepackage{color-edits}
\addauthor{jm}{blue}
\addauthor{pa}{red}
\addauthor{ss}{green}
\addauthor{mk}{magenta}

\newtheorem{lemma}{Lemma}

\newtheorem{theorem}{Theorem}

\theoremstyle{definition}

\theoremstyle{definition}

\newcounter{assump}
\renewcommand{\theassump}{\Alph{assump}}

\newcommand{\R}{\mathbb{R}}
\newcommand{\N}{\mathbb{N}}
\newcommand{\indi}{\mathds{1}}
\newcommand{\Ex}{\mathbb{E}}

\newcommand{\Var}{\text{Var}}

\DeclareMathOperator{\RC}{RatioCut}
\DeclareMathOperator{\NC}{NCut}
\DeclareMathOperator{\vol}{vol}
\DeclareMathOperator{\cut}{Cut}
\DeclareMathOperator{\trace}{Tr}
\DeclareMathOperator{\bal}{balance}
\DeclareMathOperator{\rank}{rank}

\DeclareMathOperator{\const}{const}
\newcommand{\charfct}{\indi}

\usepackage{makecell}

\usepackage{hyperref}

\usepackage[accepted]{my_arx_style}

\icmltitlerunning{Guarantees for Spectral Clustering with Fairness Constraints}

\begin{document}

\twocolumn[
\icmltitle{Guarantees for Spectral Clustering with Fairness Constraints}



\icmlsetsymbol{equal}{*}

\begin{icmlauthorlist}
\icmlauthor{Matth\"{a}us Kleindessner}{ru}
\icmlauthor{Samira Samadi}{gt}
\icmlauthor{Pranjal Awasthi}{ru}
\icmlauthor{Jamie Morgenstern}{gt}
\end{icmlauthorlist}

\icmlaffiliation{ru}{Department of Computer Science, Rutgers University, NJ}
\icmlaffiliation{gt}{College of Computing, Georgia Tech, GA}

\icmlcorrespondingauthor{Matth\"{a}us Kleindessner}{matthaeus.kleindessner@rutgers.edu}
\icmlcorrespondingauthor{Samira Samadi}{ssamadi6@gatech.edu}
\icmlcorrespondingauthor{Pranjal Awasthi}{pranjal.awasthi@rutgers.edu}
\icmlcorrespondingauthor{Jamie Morgenstern}{jamiemmt@cs.gatech.edu}



\icmlkeywords{Machine Learning, Spectral Clustering, Fairness, ICML}

\vskip 0.3in
]



\printAffiliationsAndNotice{}  

\begin{abstract}
  Given the widespread popularity of spectral clustering (SC) for
  partitioning graph data, we study a version of constrained SC in
  which we try to incorporate the fairness notion
proposed by
\citet{fair_clustering_Nips2017}. According to this notion, a clustering
is fair if every demographic group is approximately
proportionally
 represented in each cluster.
 To this end,
  we develop
variants of both normalized and unnormalized  constrained SC
and show that they
help find
fairer clusterings on both synthetic and real data.
We also provide a rigorous theoretical analysis of our
algorithms
on
a natural variant
of the stochastic block model, where $h$ groups have strong inter-group
connectivity, but also exhibit a ``natural'' clustering structure
which is fair. We prove that our algorithms can recover this fair
clustering with high probability.
\end{abstract}

\section{Introduction}\label{sec_introduction}

Machine learning (ML) has recently seen an explosion of applications
in settings to guide or make choices directly affecting
people. Examples include applications in lending, marketing,
education, and many more. Close on the heels of the adoption of ML
methods in these everyday domains have been any number of examples of
ML methods displaying unsavory behavior towards certain demographic
groups. These have spurred the study of \emph{fairness} of
ML
algorithms.  Numerous mathematical formulations of fairness
have been proposed for supervised learning settings, each with their
strengths and shortcomings in terms of what they disallow and how
difficult they may be to
satisfy~\citep[e.g., ][]{fta,hardt2016equality,kleinberg2017,zafar2017}. Somewhat
more recently, the community has begun to study appropriate notions of
fairness for unsupervised learning settings
\citep[e.g., ][]{fair_clustering_Nips2017,celis2018,celis_fair_ranking,samira2018,fair_k_center_2019}.

In particular, the
work of~\citet{fair_clustering_Nips2017}
proposes a notion of fairness for clustering: namely, that
each cluster has proportional representation from different
demographic groups.
Their paper
provides approximation algorithms that incorporate
this fairness
notion for
$k$-center and $k$-median clustering. The follow-up work of~\citet{sohler_kmeans} extends this
to
 $k$-means clustering.
 These papers open up an important
line of work
that aims at
studying the following
questions for clustering: a) {\it How to incorporate
  fairness constraints into popular clustering objectives and
  algorithms?} and b) {\it What is the price of fairness?} For
example, the
experimental
results of~\citet{fair_clustering_Nips2017} indicate that
fair
clusterings
come with
a significant increase
in the $k$-center or $k$-median objective value. While the above works
focus on clustering data sets in Euclidean / metric spaces,
many
clustering problems involve graph data.
On such
 data,
spectral
clustering~\citep[SC; ][]{Luxburg_tutorial} is the method of choice in
practice. In this paper, we extend the above line of work by studying
the implications of incorporating the fairness notion
of~\citet{fair_clustering_Nips2017}
into~SC.


The contributions of this paper are as follows:
\begin{itemize}[wide, labelwidth=!, labelindent=0pt]
\setlength{\itemsep}{-1pt}

\item We show how to incorporate the constraints
 that in each cluster, every group should be represented with the  same proportion  as in the original data set
 into the SC framework.
  For continuity with prior work (as discussed above; also see Section~\ref{sec_related_work}), we
  refer to these constraints as \emph{fairness}
  constraints and speak of \emph{fair} clusterings.  However,
  the terms \emph{proportionality} and \emph{proportional} would be a more formal description of our goal.
 Our approach to incorporate the fairness constraints is analogous to existing versions of constrained SC that try to incorporate must-link constraints (see Section~\ref{sec_related_work}).
In contrast to the work of \citet{fair_clustering_Nips2017},
which always yields  a fair clustering no matter how much the
objective value increases
compared to an unfair clustering,
our approach does not guarantee that we end up with a fair clustering. Rather, our approach  guides SC to
find a good
\emph{and}
fair clustering if such a one exists.

\item Indeed, we prove that our algorithms find a good and fair clustering in a natural variant of the
  famous
   stochastic block model that we propose. In our variant, $h$ demographic groups have strong inter-group
  connectivity, but also exhibit a ``natural'' clustering structure
 that is fair. We provide a
  rigorous analysis of our algorithms showing that
they
  can
  recover this fair
  clustering with high~probability. To the best of our knowledge, such an  analysis has not been done before
for constrained versions of SC.

\item We conclude by giving experimental
  results on real-world
  data sets
  where proportional clustering can be
  a desirable goal, comparing the proportionality and objective value
  of standard SC to our methods.
Our experiments confirm that our algorithms tend to find fairer clusterings compared to standard SC.
  A surprising finding
  is that in many real
  data sets {\it achieving higher proportionality often comes at minimal cost}, namely,
  that our methods produce clusterings that are fairer, but have
  objective values very close to those of clusterings produced by
  standard SC. This complements the results
  of~\citet{fair_clustering_Nips2017}, where achieving fairness constraints \emph{exactly} comes at a significant cost in the objective value, and indicates that in some
  scenarios fairness and objective value need not be at odds
  with one other.
\end{itemize}


\vspace{1pt}
\textbf{Notation~~}
For $n\in \N$, we use
$[n]=\{1,\ldots,n\}$.
 $I_n$ denotes the
 $n\times n$-identity matrix
 and $0_{n\times m}$ is the
 $n\times m$-zero matrix.
 $\mathbf{1}_n$ denotes a vector of length~$n$
 with all entries equaling  1.
 For
 a matrix
 $A\in \R^{n\times m}$, we denote the transpose of $A$ by $A^T \in \R^{m\times n}$.
 For $A\in \R^{n\times n}$, $\trace(A)$ denotes the trace of $A$,
 that is
  $\trace(A)=\sum_{i=1}^n A_{ii}$.
  If we say that a matrix is positive (semi-)definite, this implies
  that~the~matrix~is~symmetric.

\section{
Spectral Clustering
}\label{sec_SC}

To set the ground and introduce terminology,
we review
spectral clustering~(SC).
There are several versions of SC
\citep{Luxburg_tutorial}.
For ease of presentation, here we focus on unnormalized
SC
\citep{hagen1992}. In Appendix~\ref{appendix_fair_SC_normalized}, we
adapt
all
findings
of this section and the following Section~\ref{sec_Fair_SC} to normalized
SC
\citep{shi2000}.

Let
$G=(V,E)$ be an undirected graph
on~$V=[n]$.
 We assume that each edge between two vertices~$i$ and
$j$ carries a positive weight $W_{ij}>0$ encoding the strength of
similarity between the verices. If there is no edge between $i$ and
$j$, we set $W_{ij}=0$.  We assume that $W_{ii}=0$ for all $i\in[n]$.
Given $k\in \N$, unnormalized SC aims to partition
$V$ into $k$ clusters
with minimum value of the RatioCut objective function as follows
\citep[see][for details]{Luxburg_tutorial}: for a clustering
$V=C_1\dot{\cup}\ldots \dot{\cup}C_k$ we have
\begin{align}\label{def_ratio_cut}
\RC(C_1,\ldots,C_k)=\sum_{l=1}^k \frac{\cut(C_l,V\setminus C_l)}{|C_l|},
\end{align}
where
\begin{align*}
\cut(C_l,V\setminus C_l)=\sum_{i\in C_l, j\in V\setminus C_l}W_{ij}.
\end{align*}
Let $W=(W_{ij})_{i,j\in[n]}$ be the weighted adjacency matrix of~$G$ and $D$ be the degree matrix,
that is a diagonal matrix with the vertex degrees
$d_i=\sum_{j\in[n]}W_{ij}$, $i\in [n]$, on the diagonal. Let $L=D-W$
denote the unnormalized graph Laplacian matrix.  Note that $L$ is
positive semi-definite.
A key insight is that if we encode a clustering $V=C_1\dot{\cup}\ldots \dot{\cup}C_k$ by a matrix
$H\in \R^{n\times k}$
with
\begin{align}\label{enc_clustering}
H_{il}= \begin{cases}
     1/\sqrt{|C_l|}, & i\in C_l, \\
0,& i\notin C_l
   \end{cases},
\end{align}
then $\RC(C_1,\ldots,C_k)=\trace(H^TLH)$. Hence, in order to minimize
the RatioCut function over all possible clusterings, we could instead
solve
\begin{align}\label{ratiocut_problem}
 \min_{H\in\R^{n\times k}}\trace(H^T L H) ~~ \text{subject to~$H$ is of form \eqref{enc_clustering}}.
\end{align}
Spectral clustering relaxes this minimization problem by replacing the
requirement that $H$ has to be of form \eqref{enc_clustering} with the
weaker requirement that $H^TH=I_k$, that is it solves
\begin{align}\label{relaxed_problem}
 \min_{H\in\R^{n\times k}}\trace(H^T L H) ~~ \text{subject to~}H^TH=I_k.
\end{align}
Since $L$ is symmetric, it is well known that a solution to
\eqref{relaxed_problem} is given by a matrix $H$ that contains
some orthonormal eigenvectors corresponding to the $k$ smallest
eigenvalues (respecting multiplicities) of $L$ as columns
\citep[][Section 5.2.2]{luetkepohl1996}. Consequently, the first step
of SC
is to compute such an optimal~$H$ by computing the $k$ smallest
eigenvalues
and corresponding eigenvectors.  The second step is to infer a
clustering from~$H$. While there is a one-to-one correspondence
between a clustering and a matrix
of the form \eqref{enc_clustering}, this is not the case for a
solution $H$ to the relaxed problem \eqref{relaxed_problem}.  Usually,
a clustering of $V$ is inferred from $H$ by applying $k$-means
clustering to the rows of $H$.  We summarize unnormalized~SC
as Algorithm~\ref{SC_alg}. Note that, in general, there is no
guarantee
on how close the RatioCut value of the clustering obtained by Algorithm~\ref{SC_alg} to the RatioCut value of an optimal clustering (solving~\eqref{ratiocut_problem})  is.

\begin{algorithm}[t!]
   \caption{Unnormalized
   SC
   }\label{SC_alg}
\begin{algorithmic}
   \STATE {\bfseries Input:} weighted adjacency matrix $W\in\R^{n\times n}$; $k\in\N$

\vspace{1mm}
   \STATE {\bfseries Output:} a clustering of $[n]$ into $k$ clusters

   \begin{itemize}[leftmargin=*]
   \setlength{\itemsep}{-2pt}
   \item compute
     the Laplacian matrix $L=D-W$
\item compute the $k$ smallest (respecting multiplicities) eigenvalues of $L$ and the corresponding  orthonormal eigenvectors (written as columns of $H\in \R^{n\times k}$)
  \item apply $k$-means clustering to the rows of $H$
   \end{itemize}
\end{algorithmic}
\end{algorithm}

\section{Adding  Fairness Constraints}\label{sec_Fair_SC}

We now extend the above setting to incorporate fairness
constraints. Suppose that the  data set $V$ contains $h$
groups~$V_s$ such that $V=\dot{\cup}_{s\in[h]}
V_s$. \citet{fair_clustering_Nips2017} proposed a notion of fairness
for clustering asking that every cluster contains approximately the same
number of elements from each group~$V_s$. For a
clustering $V=C_1\dot{\cup}\ldots \dot{\cup}C_k$, define the balance of
cluster $C_l$ as
\begin{align}\label{def_balance}
\bal(C_l)=\min_{s\neq s'\in[h]}\frac{|V_s \cap C_l|}{|V_{s'} \cap C_l|}\in[0,1].
\end{align}
The higher the balance of  each cluster, the fairer is the
 clustering according to the notion of \citet{fair_clustering_Nips2017}. For any clustering, we have
$\min_{l\in [k]} \bal(C_l)\leq \min_{s\neq s'\in[h]} |V_s|/|V_{s'}|$,
so that this fairness notion is actually asking for
a clustering in which in every
cluster, each group is (approximately) represented with the same fraction as in the
whole data set $V$.  The following lemma shows how to
incorporate this goal into
the RatioCut minimization problem
\eqref{ratiocut_problem} using a linear constraint on $H$.

\begin{lemma}[Fairness constraints as linear constraint on $H$]\label{lemma_most_gen_const}
For $s\in[h]$, let $f^{(s)}\in \{0,1\}^n$ be the group-membership vector of $V_s$, that is $f^{(s)}_i=1$ if $i\in V_s$ and $f^{(s)}_i=0$ otherwise. Let $V=C_1\dot{\cup}\ldots \dot{\cup}C_k$ be a clustering that is encoded as in \eqref{enc_clustering}.
We have,
for every $l\in[k]$,
 \begin{align*}
 \forall s\in[h-1]: \sum_{i=1}^n \left(f^{(s)}_i-\frac{|V_s|}{n}\right) H_{il}=0  ~~~\Leftrightarrow~~~\\
 ~~~~~~~~~~~~~~~~~~~~ \forall s\in[h]:\frac{|V_s\cap C_l|}{|C_l|}=\frac{|V_s|}{n}.
 \end{align*}
\end{lemma}

\begin{proof}
This
simply
follows from
\begin{align*}
 \sum_{i=1}^n \left(f^{(s)}_i-\frac{|V_s|}{n}\right) H_{il}=\frac{|V_s\cap C_l|}{\sqrt{|C_l|}}-\frac{|V_s|\cdot|C_l|}{n\sqrt{|C_l|}}.
  \end{align*}
and  $|C_l|=\sum_{s=1}^h |V_s\cap C_l|$.
\end{proof}

Hence,
if we want to find a clustering that minimizes the RatioCut objective function and is as fair as possible, we have to solve
\begin{align}\label{ratiocut_problem_fair}
\begin{split}
 \min_{H\in\R^{n\times k}}\trace(H^T L H) ~~ \text{subject to~$H$ is of form \eqref{enc_clustering}}\\[-5pt]
 \text{and $F^TH=0_{(h-1)\times k}$},
\end{split}
\end{align}
where $F\in\R^{n\times (h-1)}$ is the matrix that has the vectors $f^{(s)}-(|V_s|/n)\cdot\mathbf{1}_n$, $s\in[h-1]$, as columns. In the same way as we have
relaxed \eqref{ratiocut_problem} to \eqref{relaxed_problem}, we
may
relax the minimization problem \eqref{ratiocut_problem_fair} to
\begin{align}\label{ratiocut_problem_fair_relaxed}
\begin{split}
 \min_{H\in\R^{n\times k}}\trace(H^T L H) ~~ \text{subject to~$H^TH=I_k$}\\[-5pt]
 \text{and $F^TH=0_{(h-1)\times k}$}.
\end{split}
\end{align}
Our proposed  approach to incorporate the fairness notion by \citet{fair_clustering_Nips2017} into
the SC framework
 consists of solving \eqref{ratiocut_problem_fair_relaxed} instead of \eqref{relaxed_problem}
 (and, as before,  applying $k$-means clustering to the rows of an optimal~$H$ in order to infer a clustering).
 Our approach is analogous to the numerous versions of constrained SC that try to incorporate
must-link
 constraints (``vertices A and B should end up in the same cluster'') by
 putting
 a linear constraint on~$H$
(e.g., \citealp{stella_journal_2004,kawale2013}; see Section~\ref{sec_related_work}).

%
%

\begin{algorithm}[t!]
   \caption{
Unnormalized SC with fairness constraints
   }\label{fair_SC_alg}
\begin{algorithmic}
   \STATE {\bfseries Input:} weighted adjacency matrix $W\in\R^{n\times n}$; $k\in\N$; group-membership vectors  $f^{(s)}\in \{0,1\}^n$, $s\in [h]$

\vspace{1mm}
   \STATE {\bfseries Output:} a clustering of $[n]$ into $k$ clusters

   \begin{itemize}[leftmargin=*]
   \setlength{\itemsep}{-2pt}
   \item compute
the Laplacian matrix $L=D-W$
\item Let
$F$
be a matrix with columns $f^{(s)}-\frac{|V_s|}{n}\cdot\mathbf{1}_n$, $s\in[h-1]$
\item compute a matrix $Z$  whose columns form an orthonormal basis of the nullspace of $F^T$
  \item compute the $k$ smallest (respecting multiplicities) eigenvalues of $Z^TLZ$ and the corresponding  orthonormal eigenvectors (written as columns of
$Y$)
  \item apply $k$-means clustering to the rows of $H=ZY$
   \end{itemize}
\end{algorithmic}
\end{algorithm}

Next, we describe a straightforward way to solve \eqref{ratiocut_problem_fair_relaxed}, which is also discussed by \citet{stella_journal_2004}.
It is easy to see that
$\rank(F)=\rank(F^T)=h-1$.
%
We need to assume that $k\leq n-h+1$ since otherwise \eqref{ratiocut_problem_fair_relaxed} does not have any solution.
Let $Z\in \R^{n\times(n-h+1)}$ be a matrix whose columns form an orthonormal basis of the nullspace of $F^T$.
We can substitute $H=ZY$ for  $Y\in \R^{(n-h+1)\times k}$, and then, using that $Z^TZ=I_{(n-h+1)}$, problem \eqref{ratiocut_problem_fair_relaxed} becomes
\begin{align}\label{ratiocut_problem_fair_relaxed_transformed}
 \min_{Y\in \R^{(n-h+1)\times k}}\trace(Y^TZ^T LZY) ~~ \text{subj. to~$Y^TY=I_k$}.
\end{align}
Similarly to problem \eqref{relaxed_problem}, a solution to \eqref{ratiocut_problem_fair_relaxed_transformed} is given by a matrix $Y$ that contains some
orthonormal eigenvectors corresponding to the $k$ smallest eigenvalues
(respecting multiplicities) of $Z^TLZ$ as columns.
We then set $H=ZY$.

This way of solving \eqref{ratiocut_problem_fair_relaxed}
gives rise to our ``fair'' version of unnormalized~SC
as stated in Algorithm~\ref{fair_SC_alg}.
Note that just as there is no guarantee on the RatioCut value of the output of Algorithm~\ref{SC_alg} or Algorithm~\ref{fair_SC_alg}
compared to the RatioCut value of an optimal clustering,
in general, there is also no guarantee on how fair the output of Algorithm~\ref{fair_SC_alg} is.
We still refer to Algorithm~\ref{fair_SC_alg} as our \emph{fair} version of unnormalized SC.
Similarly to how we proceeded here,
in Appendix~\ref{appendix_fair_SC_normalized}, we
incorporate the fairness constraints into normalized SC
 and state
 our fair version of normalized SC
 as Algorithm~\ref{fair_SC_alg_normalized}.

 One might wonder why we do not simply run standard SC on each group $V_s$ separately in order to derive a fair version. In Appendix~\ref{appendix_baseline_strategy} we show why such an idea does not work.

 \vspace{1pt}
 \textbf{Computational complexity~~}
We provide a complete discussion
of the complexity of our algorithms
in Appendix~\ref{appendix_complexity}.
With the implementations as stated, the complexity of both
Algorithm~\ref{fair_SC_alg} and Algorithm~\ref{fair_SC_alg_normalized} is
$\mathcal{O}(n^3)$ regarding time and $\mathcal{O}(n^2)$ regarding space, which is the same as the worst-case complexity of standard SC when the number of clusters can be arbitrary. One could apply one of the techniques suggested in the existing literature
on constrained spectral clustering to speed up computation  (e.g., \citealp{stella_journal_2004}, or \citealp{xu2009}; see Section~\ref{sec_related_work}), but most of these techniques only work for $k=2$ clusters.

\section{Analysis on  Variant of
the
Stochastic Block Model}\label{section_SBmodel}

In this section, our goal is to model data sets that have two or
more meaningful ground-truth clusterings, of which only one is fair,
and show that our algorithms recover the fair ground-truth
clustering. If there was only one meaningful ground-truth clustering
and this clustering was fair,
then any clustering algorithm that is able
to recover the ground-truth clustering (e.g., standard SC) would be a fair algorithm.
To this end, we define a variant of the famous stochastic block model
\citep[SBM; ][]{Holland1983}. The
SBM is a random graph model that has been widely used to study the performance of
clustering algorithms, including standard SC (see
Section~\ref{sec_related_work} for related work). In the traditional
SBM there is a ground-truth clustering of the vertex set $V=[n]$ into
$k$ clusters, and in a random graph generated from the model, two
vertices $i$ and $j$ are connected with a probability that only
depends on which clusters $i$ and $j$ belong to.

In our variant of the SBM we assume that $V=[n]$ comprises $h$ groups $V=V_1\dot{\cup}\ldots \dot{\cup}V_h$ and is partitioned
into $k$ ground-truth clusters  $V=C_1\dot{\cup}\ldots \dot{\cup}C_k$ such that
$|V_s\cap C_l|/|C_l|=\eta_s$, $s\in[h],l\in[k]$,
for some $\eta_1,\ldots,\eta_h\in (0,1)$ with $\sum_{s=1}^h\eta_s=1$.
Hence, in every cluster each group is represented with the same fraction as in the whole data set $V$ and this ground-truth clustering is fair.
Now we define a random graph on $V$ by connecting two vertices $i$ and $j$
with a certain probability~$\Pr(i,j)$
that
only depends on whether $i$ and $j$ are in the same cluster (or not) and on whether $i$ and $j$ are in the same group (or not). More specifically, we have
\begin{align}\label{prob_stoch_mod}
\begin{split}
&\Pr(i,j)=\\
&\begin{cases}
  a, & \text{$i$ and $j$  in  same cluster and in same group}, \\
b, & \text{$i$ and $j$ not in same cluster, but in same group}, \\
c, & \text{$i$ and $j$ in same cluster, but not in same group}, \\
d, & \text{$i$ and $j$ not in same cluster and not in same group},
    \end{cases}
\end{split}
    \end{align}
and assume that $a>b>c>d$.
As in the
ordinary
SBM, connecting $i$ and $j$ is independent of connecting $i'$ and $j'$ for $\{i,j\}\neq\{i',j'\}$. Every edge is assigned a weight of~$+1$,
that is no two connected vertices are considered more similar to each other than any two other connected vertices.

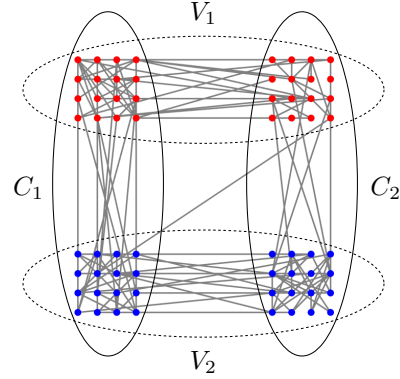
\begin{figure}[t]
\centering
\begin{overpic}[scale=0.25]{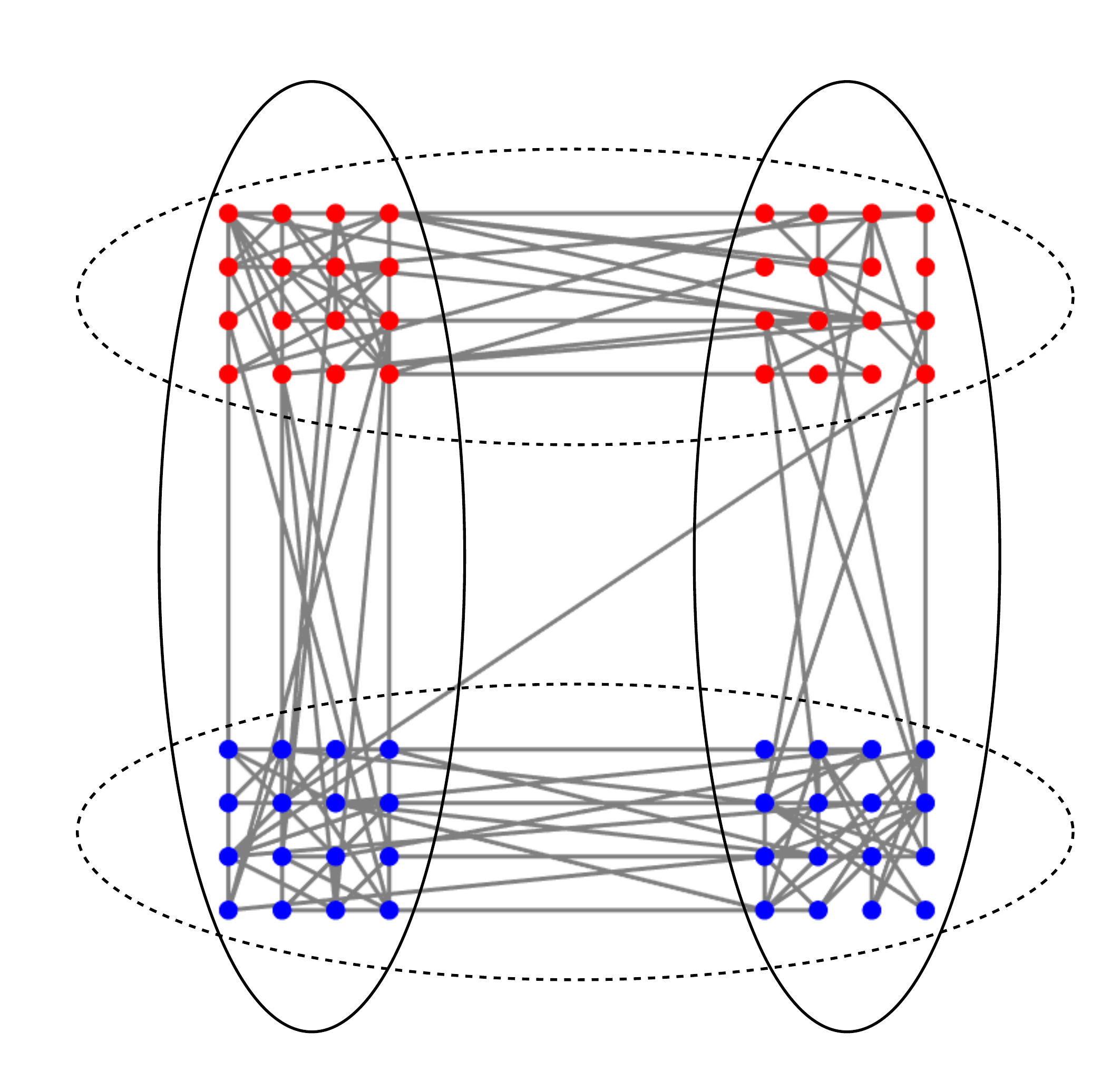}
\put(48,88){$V_1$}
\put(92.5,45){$C_2$}
\put(4.5,45){$C_1$}
\put(48,2){$V_2$}
\end{overpic}
\caption{Example of a graph generated from our variant of the
SBM.
There are two meaningful ground-truth clusterings into two clusters: $V=C_1\dot{\cup}C_2$ and $V=V_1\dot{\cup}V_2$.
Only the first one is fair.}\label{example_graph}
\end{figure}

An example of a graph generated from our model (with $h=k=2$ and $\eta_1=\eta_2=1/2$) can be seen in Figure~\ref{example_graph}.
We can see that there are two meaningful ground-truth clusterings into two clusters: $V=C_1\dot{\cup}C_2$ and $V=V_1\dot{\cup}V_2$. Among these two clusterings,
only $V=C_1\dot{\cup}C_2$ is fair since $\bal(C_i)=1$ while $\bal(V_i)=0$ for
$i\in\{1,2\}$.
Note that
the clustering $V=V_1\dot{\cup}V_2$ has a smaller RatioCut value than
 $V=C_1\dot{\cup}C_2$ because there are more edges between $V_s\cap C_1$ and $V_s\cap C_2$ ($s=1$ or $s=2$) than between $V_1\cap C_l$ and $V_2\cap C_l$ ($l=1$ or $l=2$).
As we will see in the experiments in Section~\ref{sec_experiments} (and can also be seen from the proof of the following Theorem~\ref{theorem_SB_model}), for such a graph, standard SC is very likely to
return the unfair clustering $V=V_1\dot{\cup}V_2$ as output. In contrast, our fair versions of SC return the fair clustering $V=C_1\dot{\cup}C_2$ with high probability:

\begin{theorem}[SC with fairness constraints succeeds on
variant of
stochastic block model]\label{theorem_SB_model}
Let $V=[n]$ comprise $h=h(n)$ groups $V=V_1\dot{\cup}\ldots \dot{\cup}V_h$ and be partitioned
into $k=k(n)$ ground-truth clusters  $V=C_1\dot{\cup}\ldots \dot{\cup}C_k$ such that for all $s\in[h]$ and $l\in[k]$
\begin{align}\label{bal_assum_theorem}
|V_s|=\frac{n}{h},\quad |C_l|=\frac{n}{k},\quad \frac{|V_s\cap C_l|}{|C_l|}=\frac{1}{h}.
\end{align}
Let $G$ be a random graph constructed according to our variant of the stochastic block model~\eqref{prob_stoch_mod} with probabilities
$a=a(n)$, $b=b(n)$, $c=c(n)$, $d=d(n)$ satisfying
$a>b>c>d$ and $a\geq C\ln n/n$
for some $C>0$.

Assume that we run Algorithm \ref{fair_SC_alg} or Algorithm \ref{fair_SC_alg_normalized} (stated in Appendix~\ref{appendix_fair_SC_normalized}) on $G$, where we apply a $(1+M)$-approximation algorithm
to the $k$-means problem
encountered
in the last step of Algorithm \ref{fair_SC_alg} or Algorithm \ref{fair_SC_alg_normalized},
for some $M>0$. Then, for every $r>0$, there exist
constants
$\widehat{C}_i=\widehat{C}_i(C,r)$ and $\widetilde{C}_i=\widetilde{C}_i(C,r)$, $i\in\{1,2\}$, such that the following is true:

\begin{itemize}[leftmargin=*]
\item \textbf{
Unnormalized SC with fairness constraints
}

If
\begin{align}\label{condition_on_probabilities}
\frac{a \cdot k^3 \cdot \ln n}{(c-d)^2 \cdot n}< \frac{\widehat{C}_1}{1+M},
\end{align}
then with probability at least $1-n^{-r}$, the clustering returned by Algorithm~\ref{fair_SC_alg} misclassifies at most
\begin{align}\label{quant_guaran_theo}
\widetilde{C}_1\cdot(1+M)\cdot\frac{a\cdot k^2\cdot \ln n}{(c-d)^2}
\end{align}
many vertices.

\item \textbf{
Normalized SC with fairness constraints}

Let $\lambda_1=\frac{n}{kh}(a+(h-1)c)+\frac{n(k-1)}{kh}(b+(h-1)d)$.
If
\begin{align}\label{condition_on_probabilities_norm}
\frac{\sqrt{k\cdot a\cdot n\ln n}}{\lambda_1-a}< \frac{\widehat{C}_2}{1+M}~~\text{and}~~ \frac{a \cdot k^4 \cdot \ln n}{(c-d)^2 \cdot n}< \frac{\widehat{C}_2}{1+M},
\end{align}
then
with probability at least $1-n^{-r}$, the clustering returned by Algorithm~\ref{fair_SC_alg_normalized} misclassifies at most
\begin{align}\label{quant_guaran_theo_norm}
\widetilde{C}_2\cdot(1+M)\cdot\left[\frac{a\cdot k^3\cdot \ln n}{(c-d)^2} +\frac{a\cdot n^2\cdot \ln n}{(\lambda_1-a)^2}\right]
\end{align}
many vertices.

\end{itemize}

\end{theorem}

We
 make
several
remarks on Theorem~\ref{theorem_SB_model}:
\begin{enumerate}[wide, labelwidth=!, labelindent=0pt]
\setlength{\itemsep}{-1pt}
\item By
``misclassifies
at most
$x$ many
vertices'' we mean that, considering the index $l$ of the cluster $C_l$ that a vertex belongs to as the vertex's class label, there exists a permutation
of cluster indices $1,\ldots, k$ such that up to this permutation the clustering returned by our algorithm predicts the correct class label for all but $x$ many vertices.

 \item The condition \eqref{condition_on_probabilities} is satisfied, for $n$ sufficiently large and assuming that $M\in \mathcal{O}(\ln k)$ (see the next remark),  in various regimes: assuming that $k\in \mathcal{O}(n^s)$ for some  $s\in[0,1/3)$,
 it is satisfied
 in the dense regime
  $a,b,c,d\sim \const$,
 but also in the sparse regime
$a,b,c,d\sim \const\cdot (\ln n/n)^q$
 for some $q\in[0,1-3s)$.

 The same is true for condition \eqref{condition_on_probabilities_norm}, but here we require $s\in[0,1/4)$ and $q\in[0,1-4s)$. We suspect that condition \eqref{condition_on_probabilities_norm}, with respect to $k$, is stronger than necessary. We also suspect that the error bound in \eqref{quant_guaran_theo_norm} is not
 tight
  with respect to $k$.
Note that in \eqref{quant_guaran_theo_norm},
both in the dense and in the sparse regime,
the term $a\cdot k^3\cdot \ln n/(c-d)^2$ is dominating over the term $a\cdot n^2\cdot \ln n/(\lambda_1-a)^2$ by the factor $k^3$.

 Both in the dense and in the sparse regime, under these assumptions on $s$, $q$ and $M$, the error bounds~\eqref{quant_guaran_theo} and \eqref{quant_guaran_theo_norm} divided by $n$, that is the fraction of misclassified vertices, tends to zero as $n$ goes to infinity. Using the terminology prevalent in the literature on community detection in
SBMs
 (see Section~\ref{sec_related_work}),
  we may say that our algorithms are \emph{weakly consistent} or solve the \emph{almost exact recovery} problem.

 \item There are efficient approximation algorithms for the $k$-means problem in $\R^l$.
 An algorithm by
\citet{ahmadian_kmeans_2017}
 achieves a constant approximation factor and has running time polynomial in $n$, $k$ and $l$, where $n$ is the number of data points.
 There is also
the famous
$(1+\varepsilon)$-approximation algorithm
 by \citet{kumar_kmeans_2004}
 with running time linear in $n$ and $l$, but exponential in $k$ and $1/\varepsilon$. The algorithm most widely used in practice (e.g., as default method in \textsc{Matlab}) is $k$-means++, which is a
 randomized
 $\mathcal{O}(\ln k)$-approximation algorithm \citep{kmeans_plusplus}.

 \item We
 show empirically in
  Section~\ref{sec_experiments}
 that our algorithms are  also able to
 find
  the fair ground-truth
 clustering in a graph constructed according to our variant of the SBM when
 \eqref{bal_assum_theorem}
 is not satisfied, that is when~the clusters are of different size
 or
 the balance of the fair ground-truth  clustering~is~smaller than $1$
 (i.e.,
 $\eta_s\neq 1/h$ for some $s\in[h]$).
 %
 For Algorithm~\ref{fair_SC_alg_normalized}, the violation of
 \eqref{bal_assum_theorem}
 can be more severe than for Algorithm~\ref{fair_SC_alg}. In general, we observe  Algorithm~\ref{fair_SC_alg_normalized}~to
 outperform Algorithm~\ref{fair_SC_alg}.
 This is in accordance with standard SC, for which normalized SC has been observed to outperform unnormalized SC
 \citep{Luxburg_tutorial,sarkar2015}.
\end{enumerate}

The proof of Theorem~\ref{theorem_SB_model} can be found in Appendix~\ref{appendix_proofs}.
It consists of two technical challenges (described here only for the unnormalized case).
The first one is to compute  the eigenvalues and eigenvectors of the matrix $Z^T\mathcal{L}Z$, where $\mathcal{L}$ is the expected Laplacian matrix of the random graph~$G$ and $Z$ is the matrix computed
in
Algorithm~\ref{fair_SC_alg}. Let $\mathcal{Y}$ be a matrix containing some orthonormal eigenvectors corresponding to the $k$ smallest eigenvalues of $Z^T\mathcal{L}Z$ as columns and $Y$ be a matrix containing orthonormal eigenvectors corresponding to the $k$ smallest eigenvalues of $Z^TLZ$, where $L$ is the
observed
Laplacian matrix of $G$.
The second challenge is to  prove that with high probability,
$ZY$ is close to $Z\mathcal{Y}$.
For doing so we make use of the famous  Davis-Kahan sin$\Theta$ Theorem \citep{davis_kahan_1970}.
After that,
we can use existing results about  $k$-means clustering of perturbed eigenvectors \citep{lei2015} to derive the theorem.

%
%

\section{Related Work}\label{sec_related_work}

\textbf{Spectral clustering and stochastic block model~~}
%
SC
is one of the most prominent clustering techniques, with a long history and an abundance of related papers. See \citet{Luxburg_tutorial} or \citet{Nascimento_2011} for  general introductions and an overview of the literature. There are numerous papers
on constrained SC, where the goal is to incorporate prior knowledge about the target clustering (usually in the form of
must-link and / or cannot-link constraints) into the SC framework \citep[e.g., ][]{stella_nips2001,stella_journal_2004,joachims_icml2003,lu2008,xu2009,wang2010,eriksson2011,maji2011,kawale2013,khoreva14gcpr,wang2014,cucuringu2016}.  Most of these papers are motivated by the use of SC in image or video segmentation.
Closely related to our work are the papers by
\citet{stella_journal_2004,xu2009,eriksson2011,kawale2013},
which incorporate the prior knowledge by
imposing a linear constraint in the RatioCut or NCut optimization problem analogously to how we derived our fair versions of SC. These papers provide efficient algorithms to solve the resulting optimization problems. However, the iterative algorithms by \citet{xu2009,eriksson2011,kawale2013} only work for $k=2$ clusters. The method by  \citet{stella_journal_2004} works for arbitrary $k$ and could be used to speed up the computation of a solution of \eqref{ratiocut_problem_fair_relaxed} or \eqref{relaxed_problem_normalized_fair} compared to our straightforward way as implemented by Algorithm~\ref{fair_SC_alg} and Algorithm~\ref{fair_SC_alg_normalized}, respectively, but
requires to modify the eigensolver in use.

The stochastic block model \citep[SBM; ][]{Holland1983} is the canonical model to study the performance of clustering algorithms.
There
exist  several variants of the
original
model such as the degree-corrected SBM or the labeled SBM.
For a recent survey see \citet{abbe2018}.
In the labeled SBM, vertices can carry a label that is correlated with the ground-truth clustering.  This is quite the opposite of our model, in which the group-membership  information is ``orthogonal'' to the
ground-truth clustering.
Several papers show the consistency (i.e., the capability to recover the ground-truth clustering) of different versions of SC on the SBM or the degree-corrected SBM
under different assumptions \citep{rohe2011,fishkind2013,qin2013,lei2015,joseph2016,su_arxiv_2017}.
%
%
%
%
%
%
%
%
%
For example, \citet{rohe2011} show consistency of normalized SC assuming that the minimum expected vertex degree is in $\Omega(n/\sqrt{\log n})$,
while \citet{lei2015}  show that SC based on the adjacency matrix is consistent requiring only that the maximum expected degree is in $\Omega(\sqrt{\log n})$.
Note that these
papers also make assumptions on the
eigenvalues
of the expected Laplacian or adjacency matrix
while all assumptions and guarantees stated in our Theorem~\ref{theorem_SB_model} directly  depend on   the
connection probabilities $a,b,c,d$ of our model.
We are not aware of any work providing consistency results for constrained SC methods as we do in this paper.

\vspace{1pt}
\textbf{Fairness~~}
By now, there is a huge body of work on fairness in machine learning. For a recent
paper providing
an
overview
of the literature on fair classification
see \citet{donini2018}.
 Our paper adds to the literature on fair methods
 for unsupervised learning tasks
\citep{fair_clustering_Nips2017,celis2018,celis_fair_ranking,samira2018,sohler_kmeans}.
Note that all these papers assume
to know
which demographic group a data point belongs to just as we do.
We discuss the pieces of work
 most
closely
 related to our~paper.

\citet{fair_clustering_Nips2017} proposed
the notion of fairness for clustering underlying our paper. It is based on the fairness notion of disparate impact
\citep{feldman2015} and the $p\%$-rule \citep{zafar2017}, respectively, which essentially say that the output of a machine learning algorithm should be independent of a sensitive attribute.
In their paper, \citeauthor{fair_clustering_Nips2017} focus on $k$-median and
$k$-center clustering. For the case of a binary sensitive attribute, that is there are only two demographic groups, they provide approximation algorithms for the problems of finding a clustering with
minimum $k$-median / $k$-center cost under the constraint that all clusters have some
prespecified level of balance.
Subsequently, \citet{roesner2018} provide an approximation algorithm for such a fair $k$-center problem
with
multiple  groups.
\citet{sohler_kmeans} build upon the fairness notion and techniques of \citeauthor{fair_clustering_Nips2017} and devise an approximation algorithm for the fair $k$-means problem, assuming that there are only two  groups  of the same size.

\section{Experiments}\label{sec_experiments}

In this section, we present a number of experiments. We first study our fair versions of spectral clustering, Algorithm~\ref{fair_SC_alg} and Algorithm~\ref{fair_SC_alg_normalized}, on synthetic data generated according to our variant of the
SBM
and compare our algorithms to standard SC. We also study how robust our algorithms are with respect to a certain perturbation of our model. We then 
compare our algorithms to standard SC on real network~data. We implemented all algorithms in \textsc{Matlab}\footnote{The code
is available on
\url{https://github.com/matthklein/fair_spectral_clustering}.}. We used the built-in function for $k$-means clustering with all parameters set to their default values except for the number of replicates, which we set to 10. In the following, all plots show  average results obtained from
running an experiment for 100 times.

\newcommand{\sizeA}{0.20}
\newcommand{\sizeB}{4mm}

\begin{figure*}[t]
\centering{
\includegraphics[scale=\sizeA]{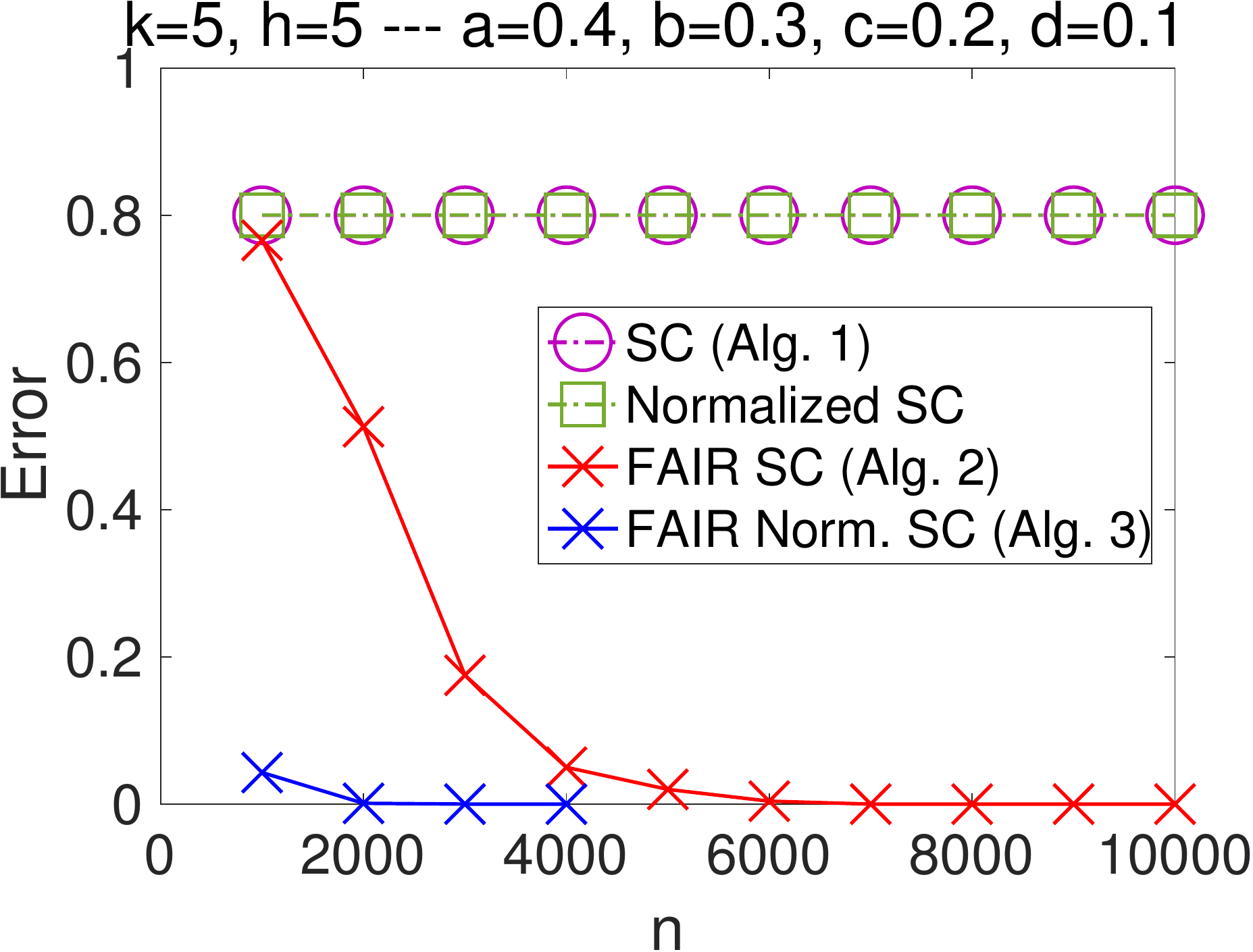}
\hspace{\sizeB}
\includegraphics[scale=\sizeA]{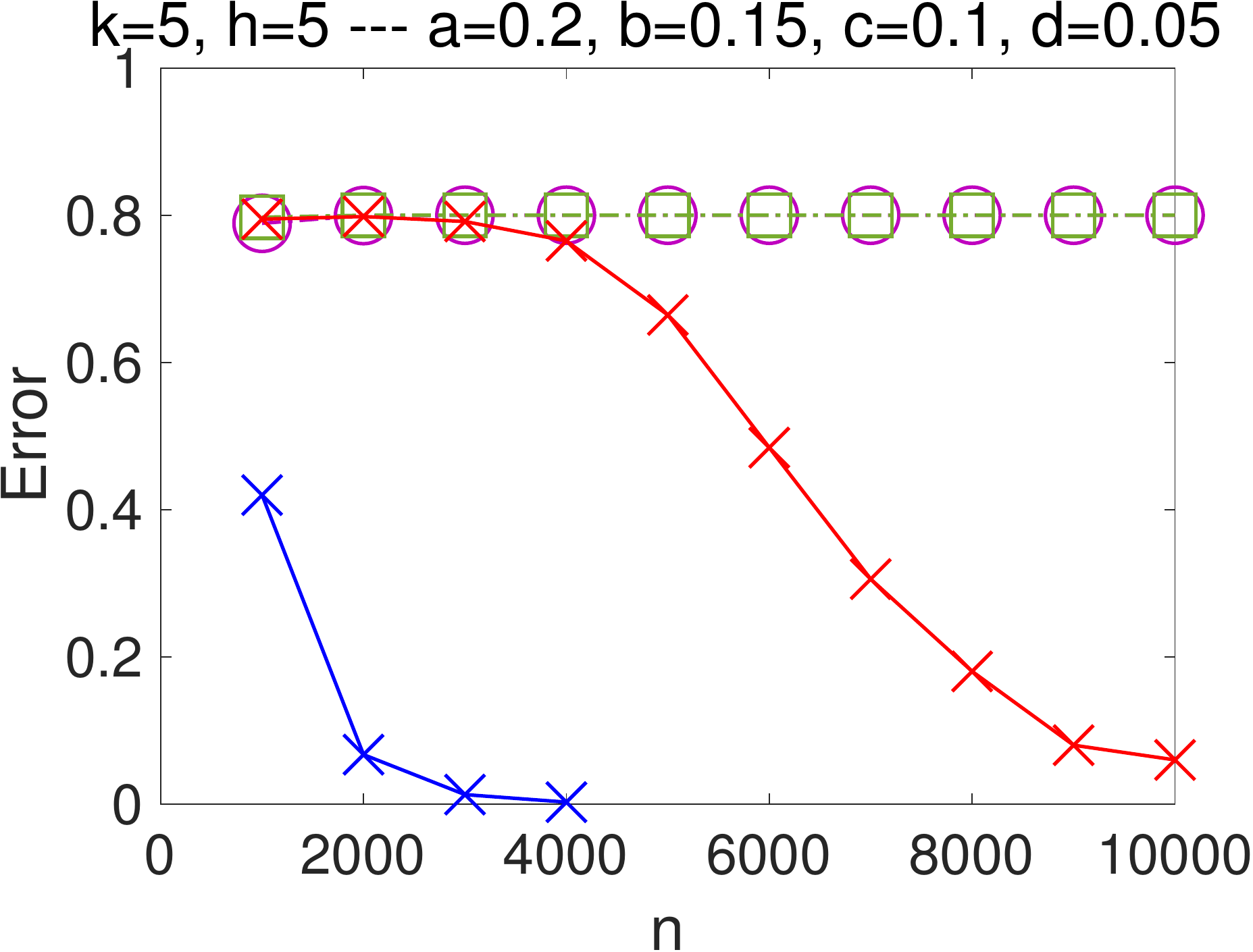}
\hspace{\sizeB}
\includegraphics[scale=\sizeA]{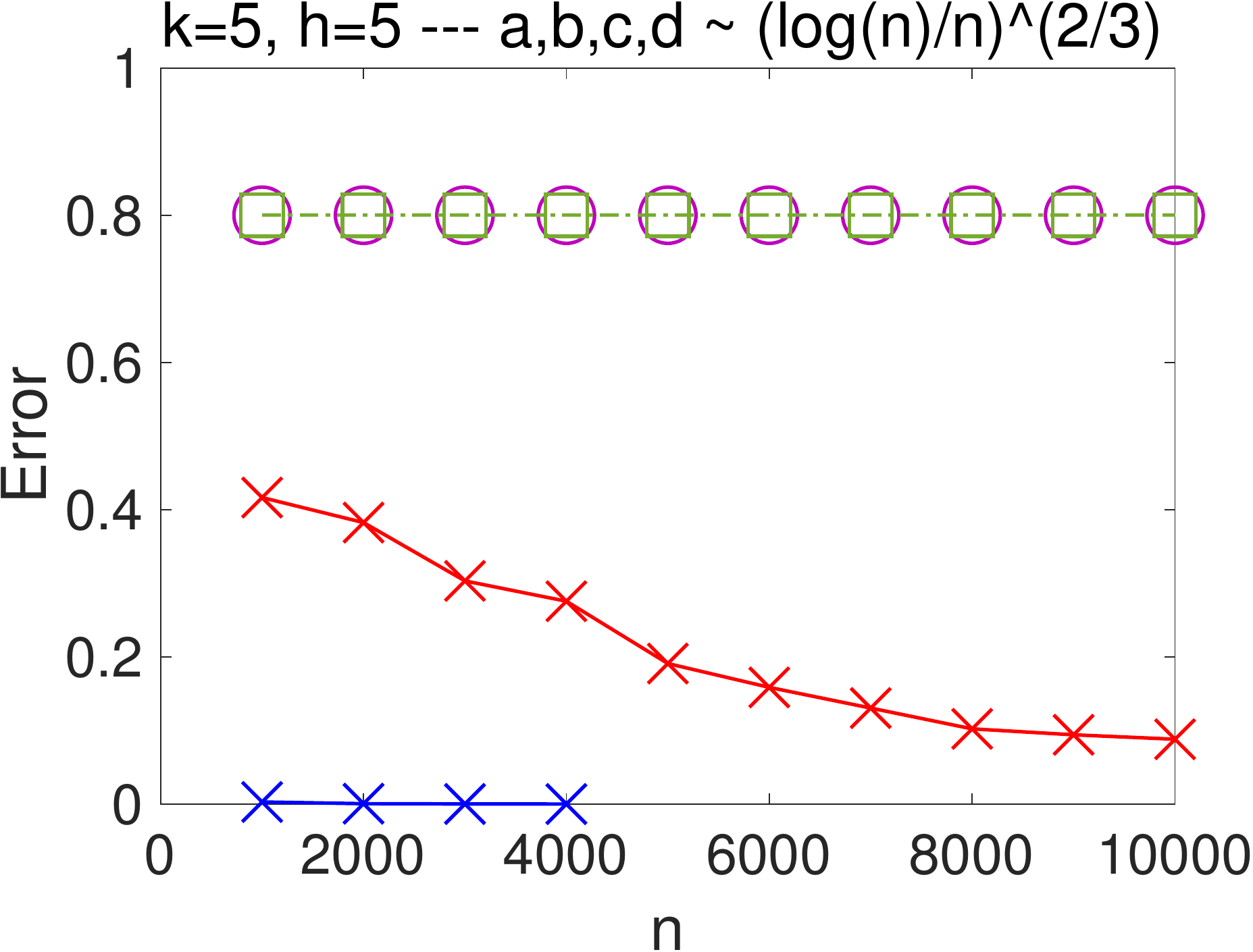}
\hspace{\sizeB}
\includegraphics[scale=\sizeA]{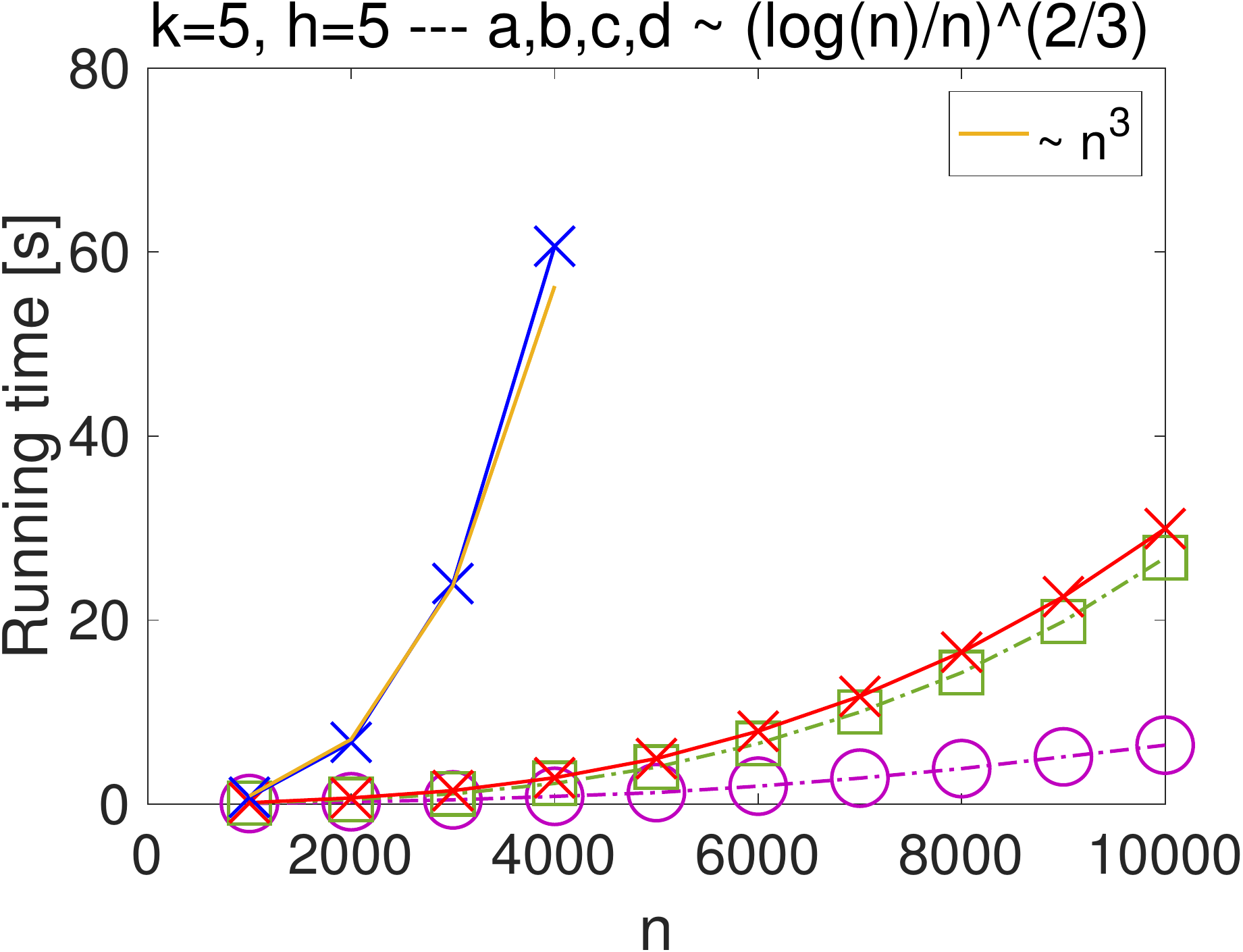}

\vspace{2mm}
\includegraphics[scale=\sizeA]{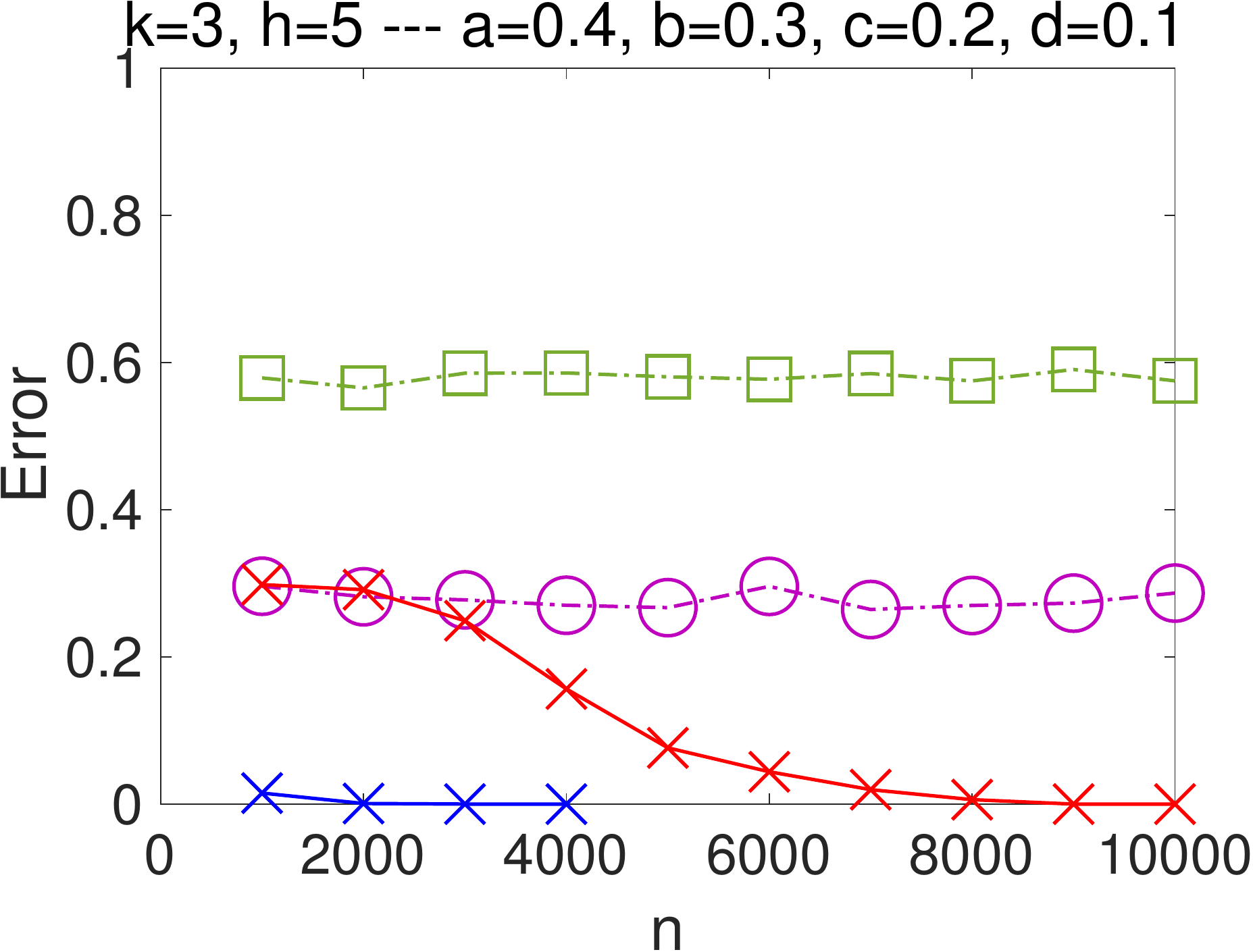}
\hspace{\sizeB}
\includegraphics[scale=\sizeA]{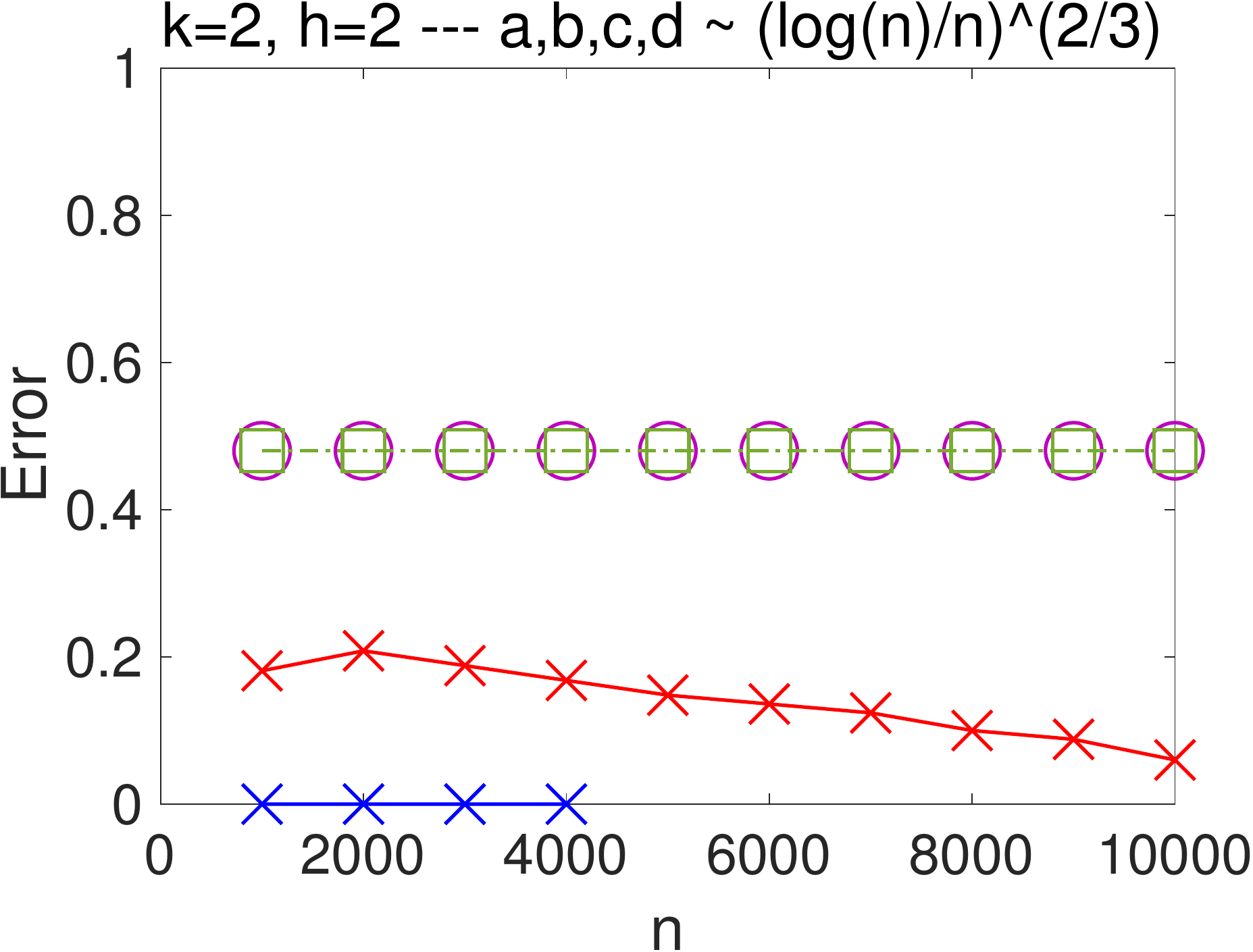}
\hspace{\sizeB}
\includegraphics[scale=\sizeA]{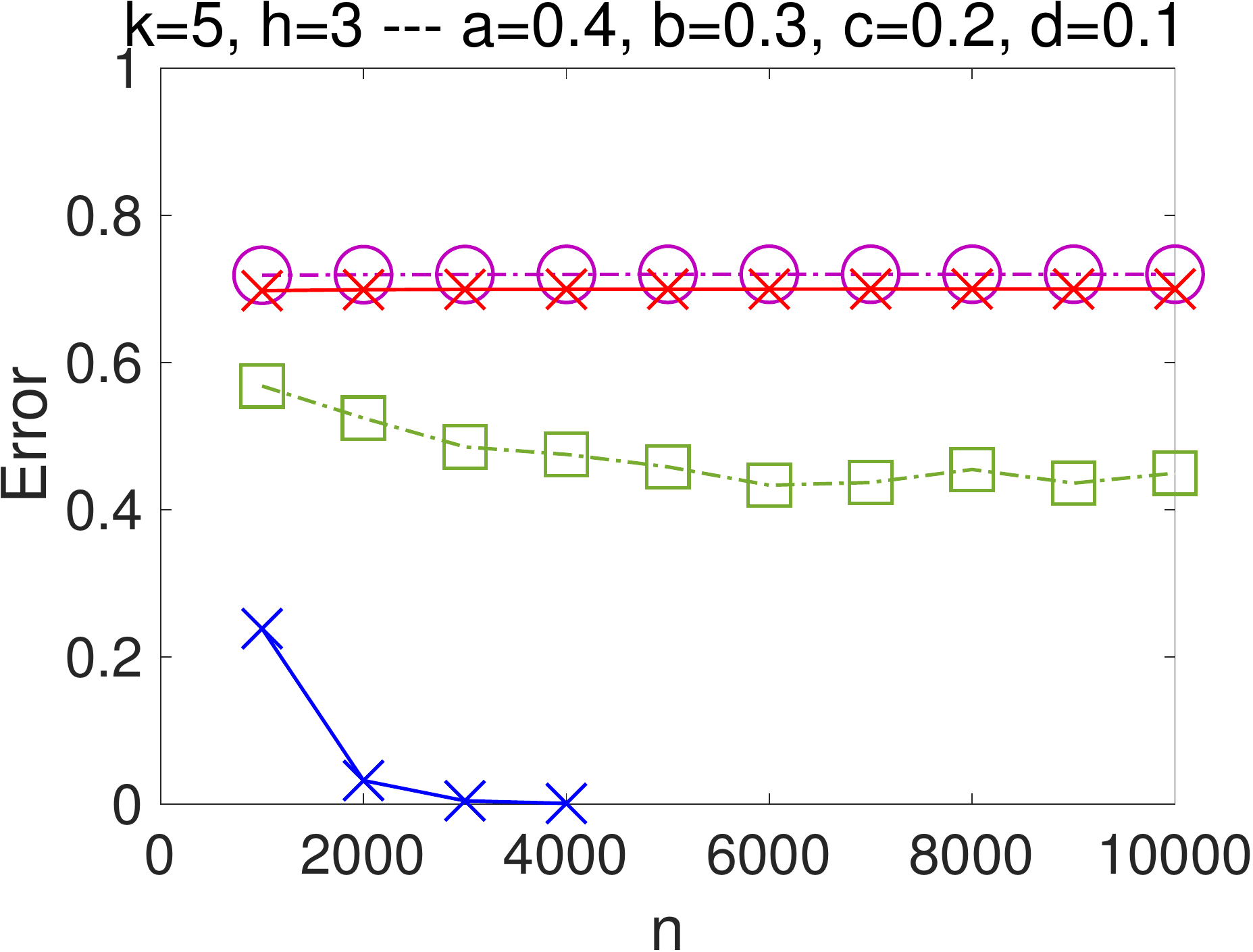}
\hspace{\sizeB}
\includegraphics[scale=\sizeA]{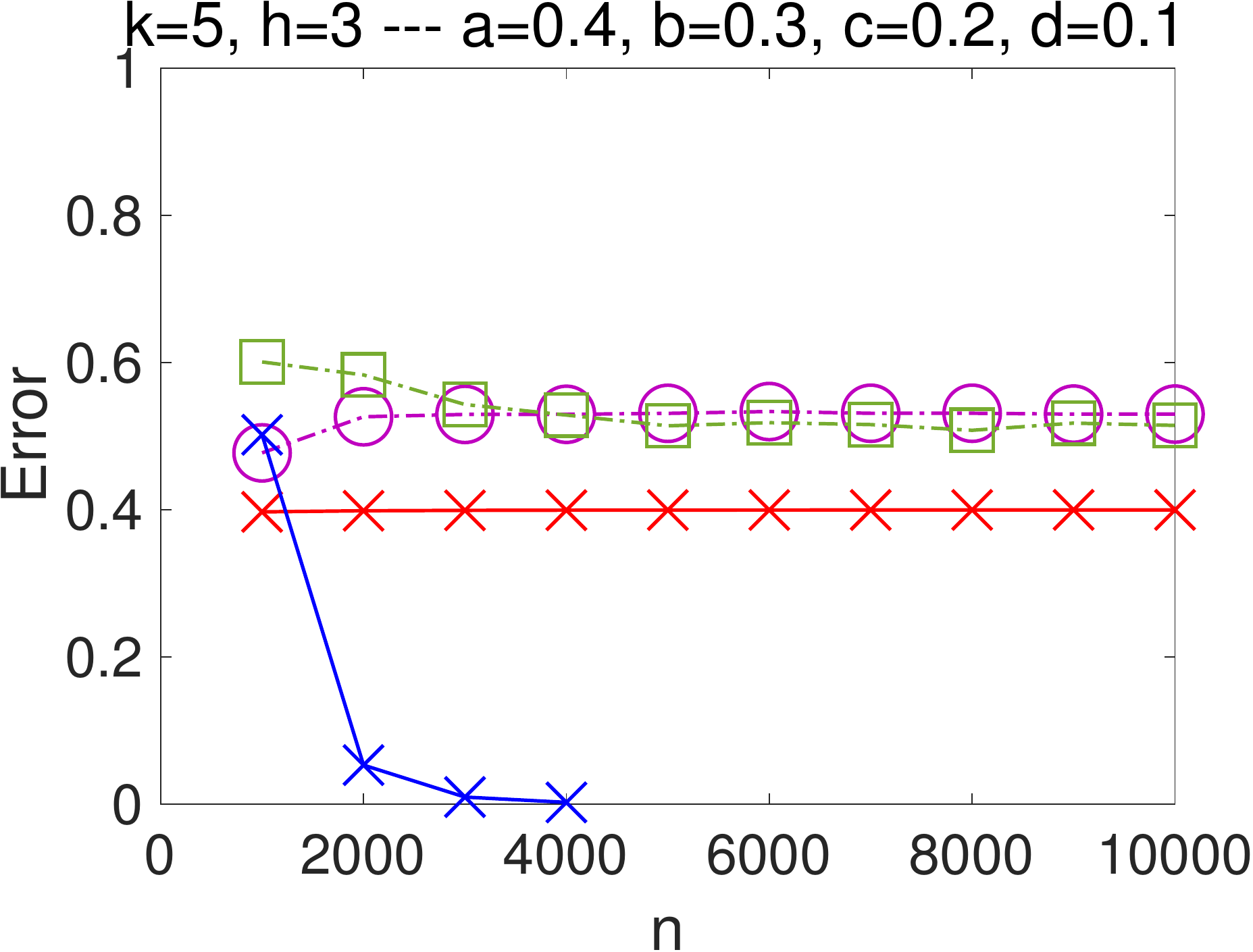}}
\caption{Performance of standard spectral clustering and our fair versions on our variant of the stochstic block model as a function of~$n$
for various
parameter
settings.
Error is the fraction of misclassified vertices
w.r.t.
the fair ground-truth clustering (see Section~\ref{section_SBmodel}).
\textbf{First row:} Assumption~\eqref{bal_assum_theorem}
in Theorem~\ref{theorem_SB_model}
is satisfied, that is
$|V_s\cap C_l|=\frac{n}{kh}$,
$s\in[h]$, $l\in[k]$. \textbf{Second row:} Assumption~\eqref{bal_assum_theorem} is~not~satisfied.\\
\textbf{2nd row, 1st plot:}
$\frac{|C_1|}{n}=\frac{7}{10}$, $\frac{|C_2|}{n}=\frac{|C_3|}{n}=\frac{15}{100}$ and $\frac{|V_s\cap C_l|}{|C_l|}=\frac{1}{5}$, $s\in[5]$, $l\in[5]$;
%
%
\textbf{2nd plot:}
$\frac{|C_1|}{n}=\frac{6}{10}$, $\frac{|C_2|}{n}=\frac{4}{10}$ and $\frac{|V_1\cap C_l|}{|C_l|}=\frac{6}{10}$, $\frac{|V_2\cap C_l|}{|C_l|}=\frac{4}{10}$, $l\in[2]$;
%
%
\textbf{3rd plot:}
$\frac{|C_1|}{n}=\frac{|C_2|}{n}=\frac{3}{10}$, $\frac{|C_3|}{n}=\frac{2}{10}$, $\frac{|C_4|}{n}=\frac{|C_5|}{n}=\frac{1}{10}$ and $\frac{|V_1\cap C_l|}{|C_l|}=\frac{5}{10}$, $\frac{|V_2\cap C_l|}{|C_l|}=\frac{3}{10}$, $\frac{|V_3\cap C_l|}{|C_l|}=\frac{2}{10}$, $l\in[5]$;
%
%
\textbf{4th plot:}
$\frac{|C_1|}{n}=\frac{6}{10}$,
$\frac{|C_2|}{n}=\ldots=\frac{|C_5|}{n}=\frac{1}{10}$
and $\frac{|V_1\cap C_l|}{|C_l|}=\frac{5}{10}$, $\frac{|V_2\cap C_l|}{|C_l|}=\frac{3}{10}$, $\frac{|V_3\cap C_l|}{|C_l|}=\frac{2}{10}$, $l\in[5]$.
}\label{fig_exp_SBmodel_1}
\end{figure*}

\subsection{Synthetic Data}\label{sec_syn_experiments}

We run experiments
on our variant of the
SBM
 introduced~in Section~\ref{section_SBmodel}.
To asses the quality of a clustering we measure the fraction of misclassified vertices
w.r.t.
the fair ground-truth clustering (see Section~\ref{section_SBmodel}), which we refer to as error.

In the experiments
of
Figure~\ref{fig_exp_SBmodel_1},
we study the performance of standard unnormalized and normalized SC and  of our fair versions, Algorithm~\ref{fair_SC_alg} and Algorithm~\ref{fair_SC_alg_normalized},   as a function of $n$.
Due to the high running time of Algorithm~\ref{fair_SC_alg_normalized}
(see
Section~\ref{sec_Fair_SC}),
we only run it up to~$n=4000$. All plots show the error of the methods, except for the fourth plot in the first row, which shows
their
running time.
We study several parameter settings. For the plots in the first row, Assumption~\eqref{bal_assum_theorem} in Theorem~\ref{theorem_SB_model} is satisfied, that is
 $|V_s\cap C_l|=\frac{n}{kh}$ for all $s\in[h]$ and $l\in[k]$. In this case, in accordance with Theorem~\ref{theorem_SB_model}, both Algorithm~\ref{fair_SC_alg} and Algorithm~\ref{fair_SC_alg_normalized} are able to recover the fair ground-truth clustering if $n$ is just large enough while standard SC always fails to
do so.
 Algorithm~\ref{fair_SC_alg_normalized} yields significantly better results than Algorithm~\ref{fair_SC_alg} and requires much smaller values of~$n$ for
achieving
zero error.
This comes at the cost of a
higher running time of Algorithm~\ref{fair_SC_alg_normalized}
(still it is in $\mathcal{O}(n^3)$ as
claimed in Section~\ref{sec_Fair_SC}). The
run-time
of Algorithm~\ref{fair_SC_alg}
is
the same as the
run-time
of standard normalized SC.
\renewcommand{\sizeA}{0.200}
\renewcommand{\sizeB}{0mm}
\begin{figure}[t]
\centering{
\includegraphics[scale=\sizeA]{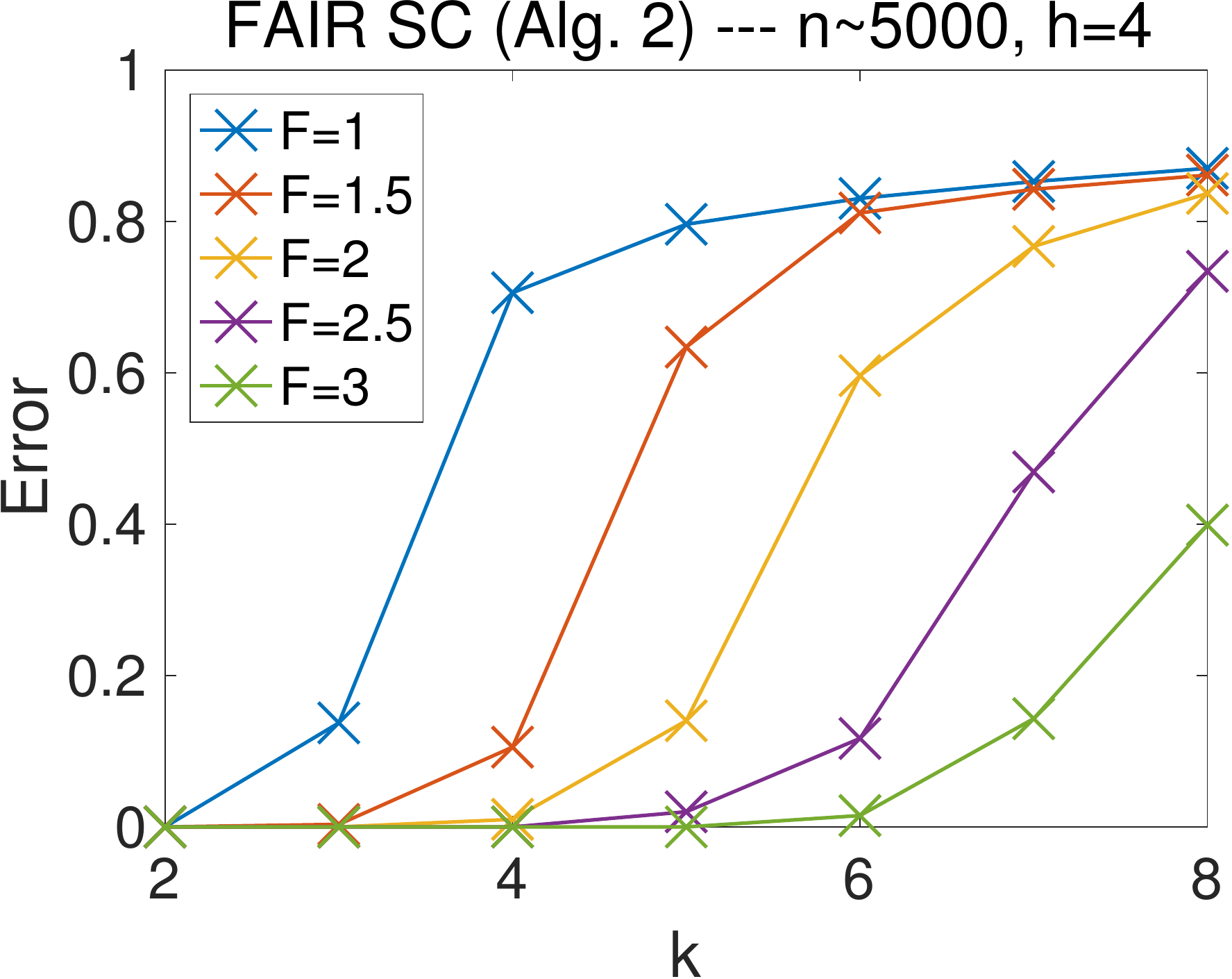}
\hspace{\sizeB}
\includegraphics[scale=\sizeA]{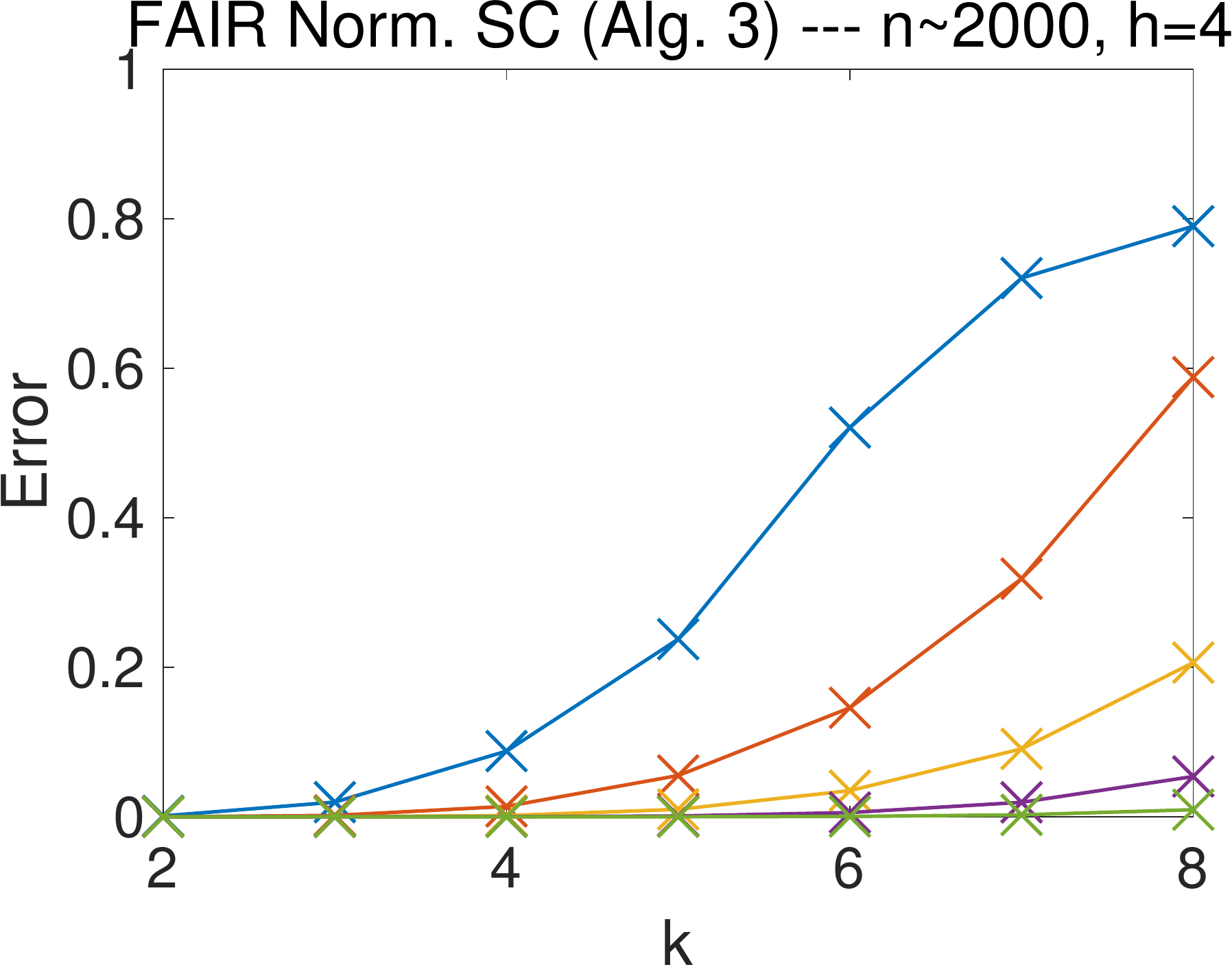}
}
\caption{Error of our algorithms as a function of $k$. We consider
$a=F\cdot \frac{25}{100}$, $b=F\cdot \frac{2}{10}$, $c=F\cdot \frac{15}{100}$, $d=F\cdot \frac{1}{10}$
 for
 various values of $F$.
\textbf{Left:} Alg.~\ref{fair_SC_alg}, $n\approx5000$. \textbf{Right:} Alg.~\ref{fair_SC_alg_normalized}, $n\approx2000$.}
\label{fig_exp_SBmodel_AsFuncOfK}
\end{figure}
\renewcommand{\sizeA}{0.200}
\renewcommand{\sizeB}{0mm}
\begin{figure}[t]
\centering{
\includegraphics[scale=\sizeA]{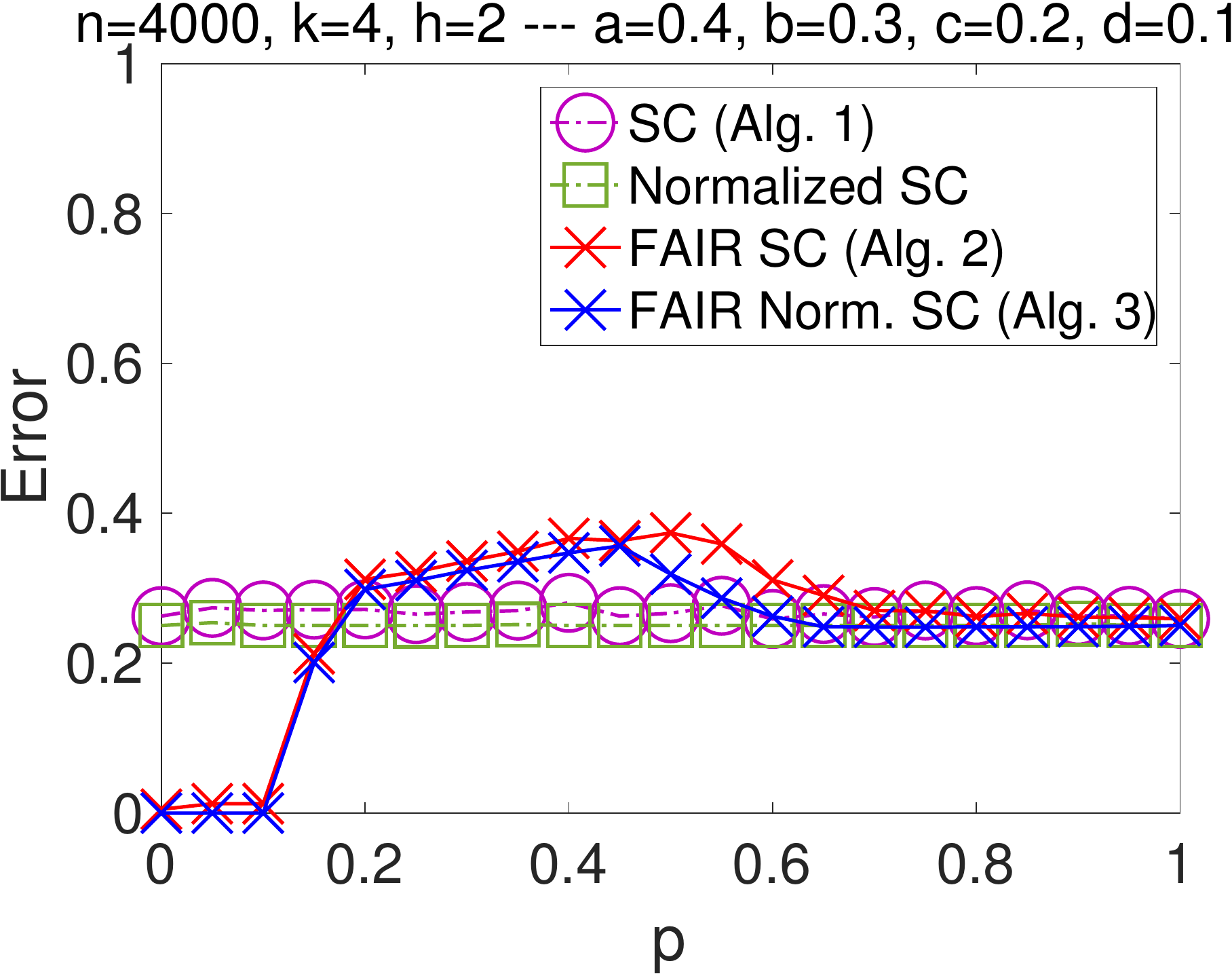}
\hspace{\sizeB}
\includegraphics[scale=\sizeA]{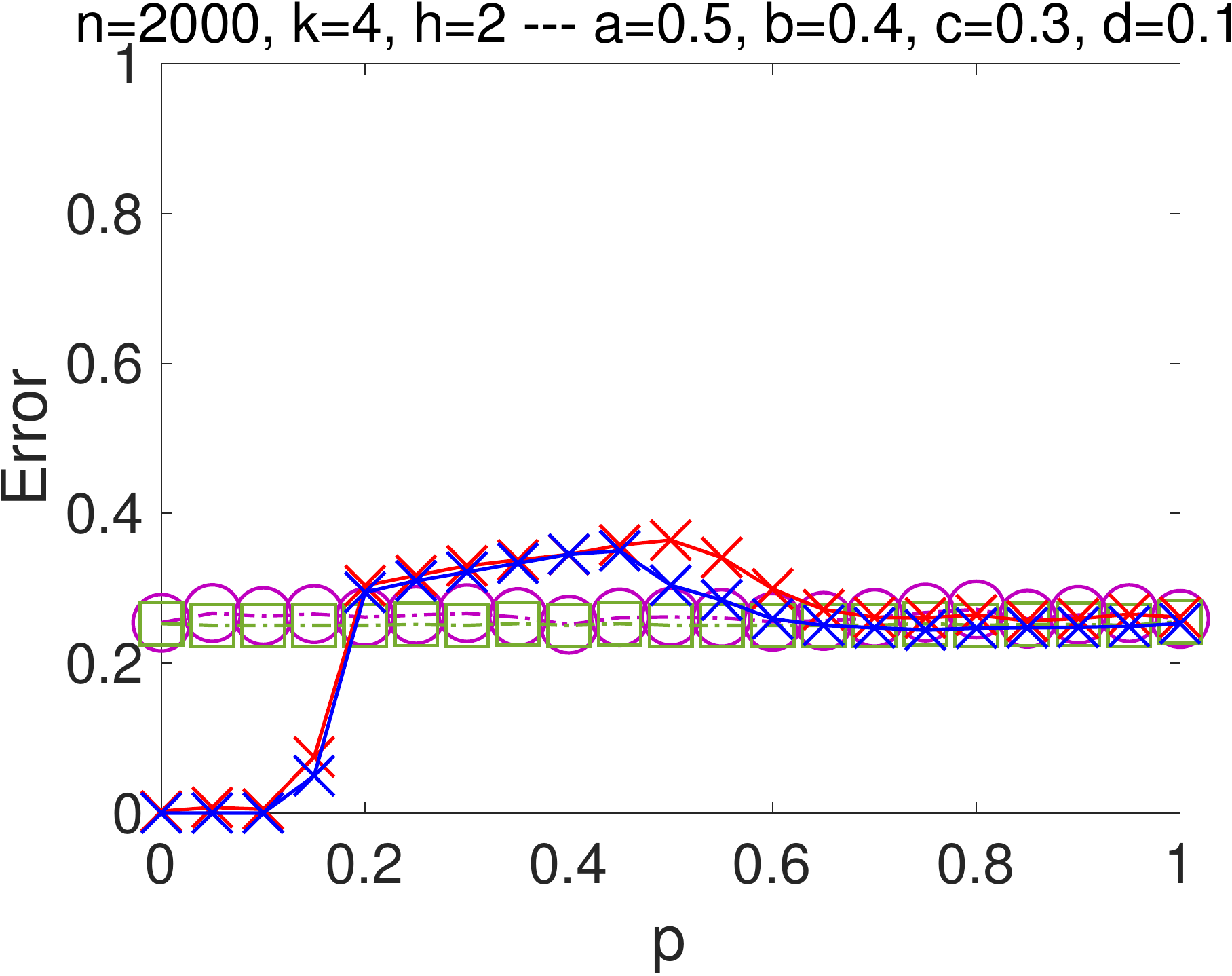}
}
\caption{Error of standard spectral clustering and our fair versions as a function of the perturbation parameter $p$.}\label{fig_exp_SBmodel_AsFuncOf_perturbation}
\end{figure}
For the plots in the second row, Assumption~\eqref{bal_assum_theorem} in Theorem~\ref{theorem_SB_model} is not satisfied. We consider various scenarios of cluster sizes~$|C_l|$ and group sizes~$|V_s|$ (however,
we always have $|V_s\cap C_l|/|C_l|=|V_s|/n$, $s\in[h]$, $l\in[k]$, so that
$C_1\dot{\cup}\ldots \dot{\cup}C_k$
is as fair as possible). When the cluster sizes are different, but the group sizes are
all equal to each other (1st plot in the 2nd row) or Assumption~\eqref{bal_assum_theorem} is only slightly violated (2nd plot), both  Algorithm~\ref{fair_SC_alg} and Algorithm~\ref{fair_SC_alg_normalized} are still able to recover the fair ground-truth clustering. Compared to the plots in the first row, Algorithm~\ref{fair_SC_alg} requires a larger value of $n$ though, even though $k$ is smaller. Algorithm~\ref{fair_SC_alg_normalized} achieves (almost) zero error already for $n=1000$ in these scenarios. When Assumption~\eqref{bal_assum_theorem} is strongly violated (3rd and 4th plot), Algorithm~\ref{fair_SC_alg} fails to recover the fair ground-truth clustering, but Algorithm~\ref{fair_SC_alg_normalized} still succeeds.

In the experiments shown in  Figure~\ref{fig_exp_SBmodel_AsFuncOfK}, we study the error of Algorithm~\ref{fair_SC_alg} (left plot) and Algorithm~\ref{fair_SC_alg_normalized} (right plot)  as a function of $k$ when $n$ is roughly fixed. More precisely, for $k\in\{2,\ldots,8\}$ and $h=4$, we have $n=kh\lceil\frac{5000}{kh}\rceil$ (Alg.~\ref{fair_SC_alg}; left) or $n=kh\lceil\frac{2000}{kh}\rceil$ (Alg.~\ref{fair_SC_alg_normalized}; right), which allows for fair ground-truth clusterings satisfying \eqref{bal_assum_theorem}. We consider
connection
probabilities
$a=F\cdot \frac{25}{100}$, $b=F\cdot \frac{2}{10}$, $c=F\cdot \frac{15}{100}$, $d=F\cdot \frac{1}{10}$ for
$F\in\{1,1.5,2,2.5,3\}$.
Unsurprisingly, for both Algorithm~\ref{fair_SC_alg}  and Algorithm~\ref{fair_SC_alg_normalized} the error is monotonically increasing with $k$. The
rate of increase critically depends on $F$
(or the probabilities $a,b,c,d$).
For Algorithm~\ref{fair_SC_alg}, this is even more severe. There is only a small range in which the various curves
exhibit polynomial growth,
which
makes it impossible to empirically evaluate whether our error guarantees~\eqref{quant_guaran_theo} and \eqref{quant_guaran_theo_norm} are tight with respect to $k$.

\renewcommand{\sizeA}{0.205}
\newcommand{\sizeAF}{0.202}
\renewcommand{\sizeB}{2mm}

\begin{figure*}[th!]
\centering{
\begin{overpic}[scale=\sizeA]{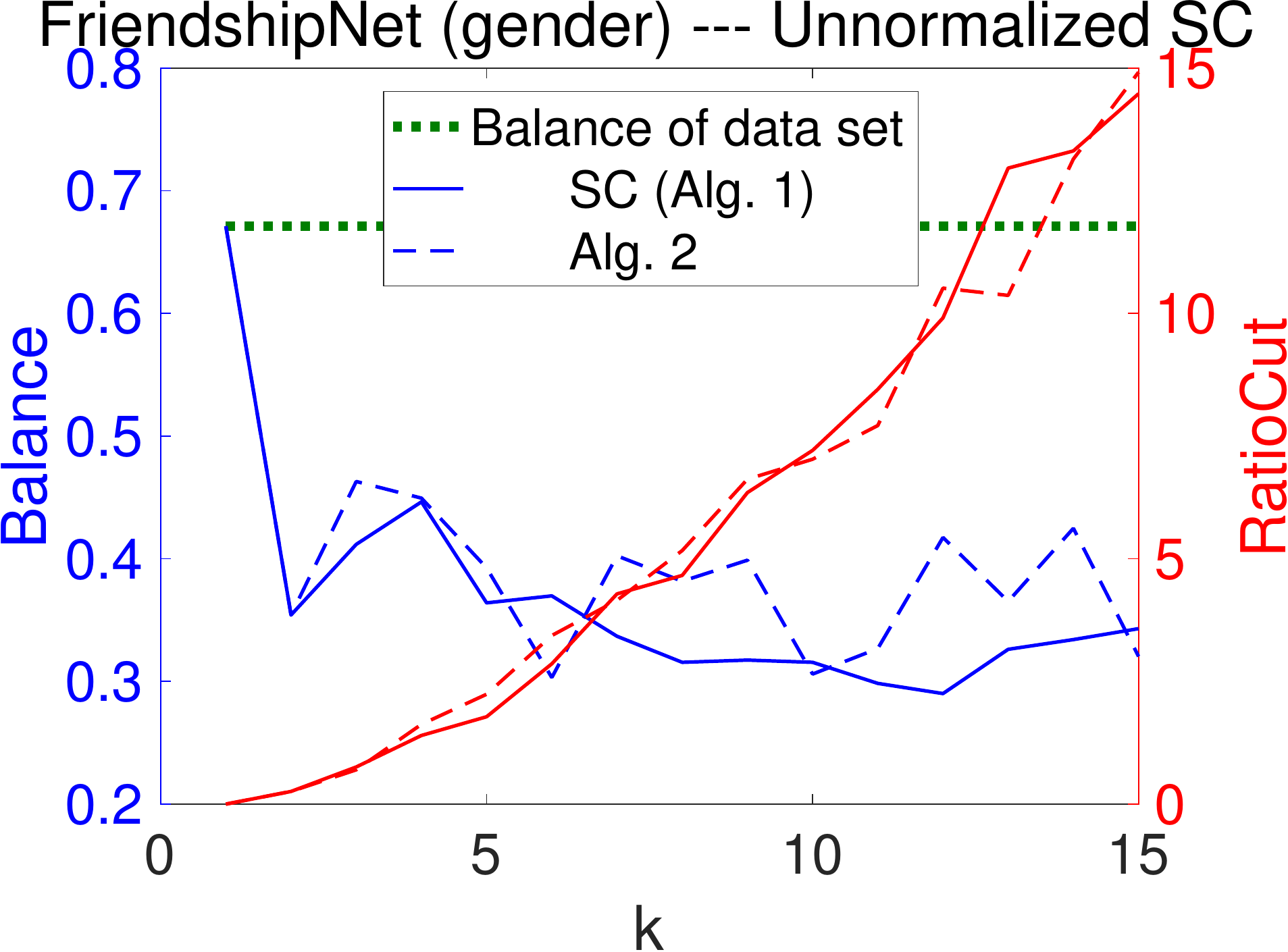}
\put(37,52.7){\includegraphics[scale=\sizeA]{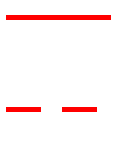}}
\end{overpic}
\hspace{\sizeB}
\begin{overpic}[scale=\sizeAF]{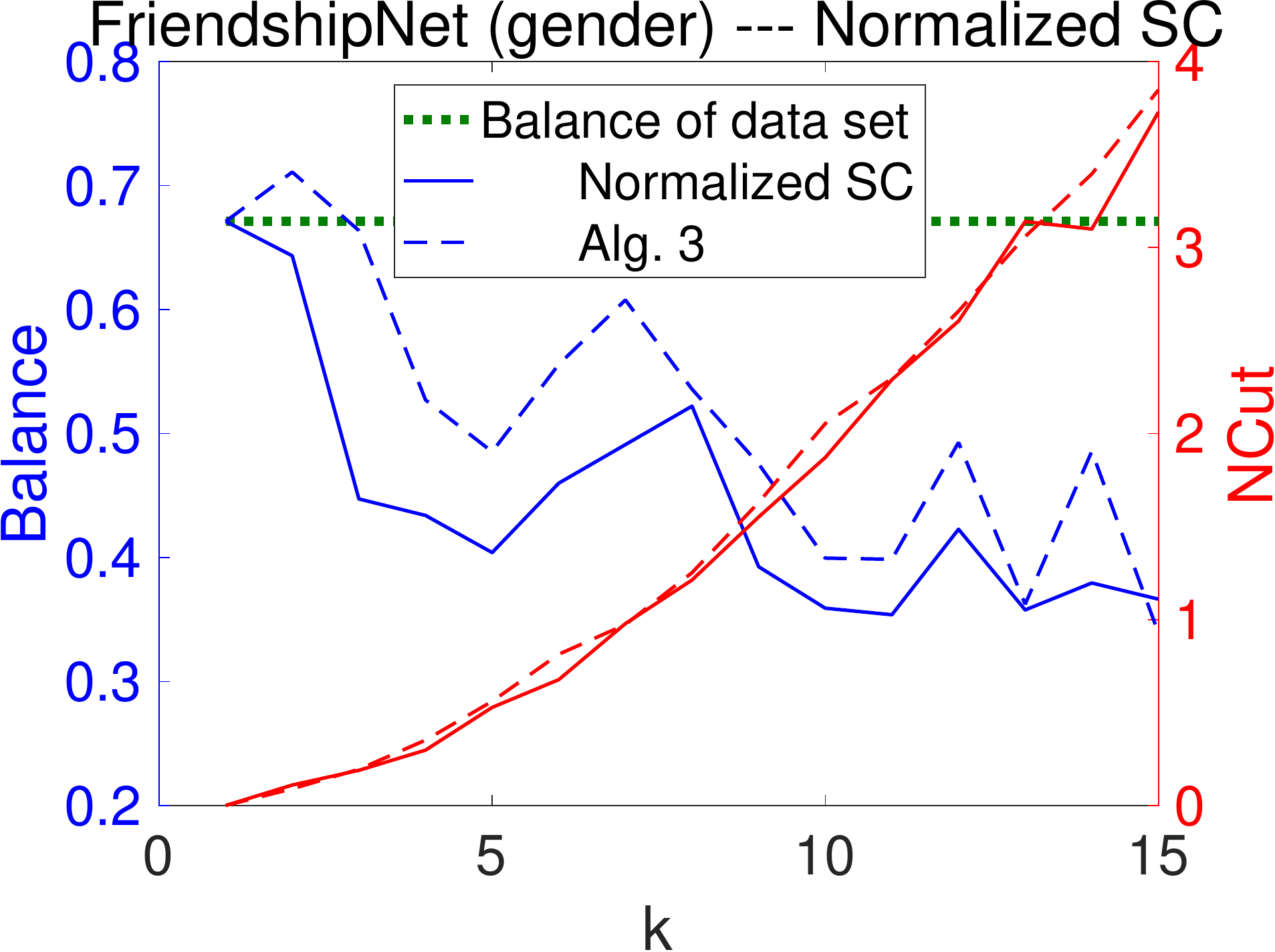}
\put(37.8,53.8){\includegraphics[scale=\sizeA]{Experiments_RealNetworks_NEW_NOTATION_2/legend_red_lines.pdf}}
\end{overpic}
\hspace{\sizeB}
\begin{overpic}[scale=\sizeA]{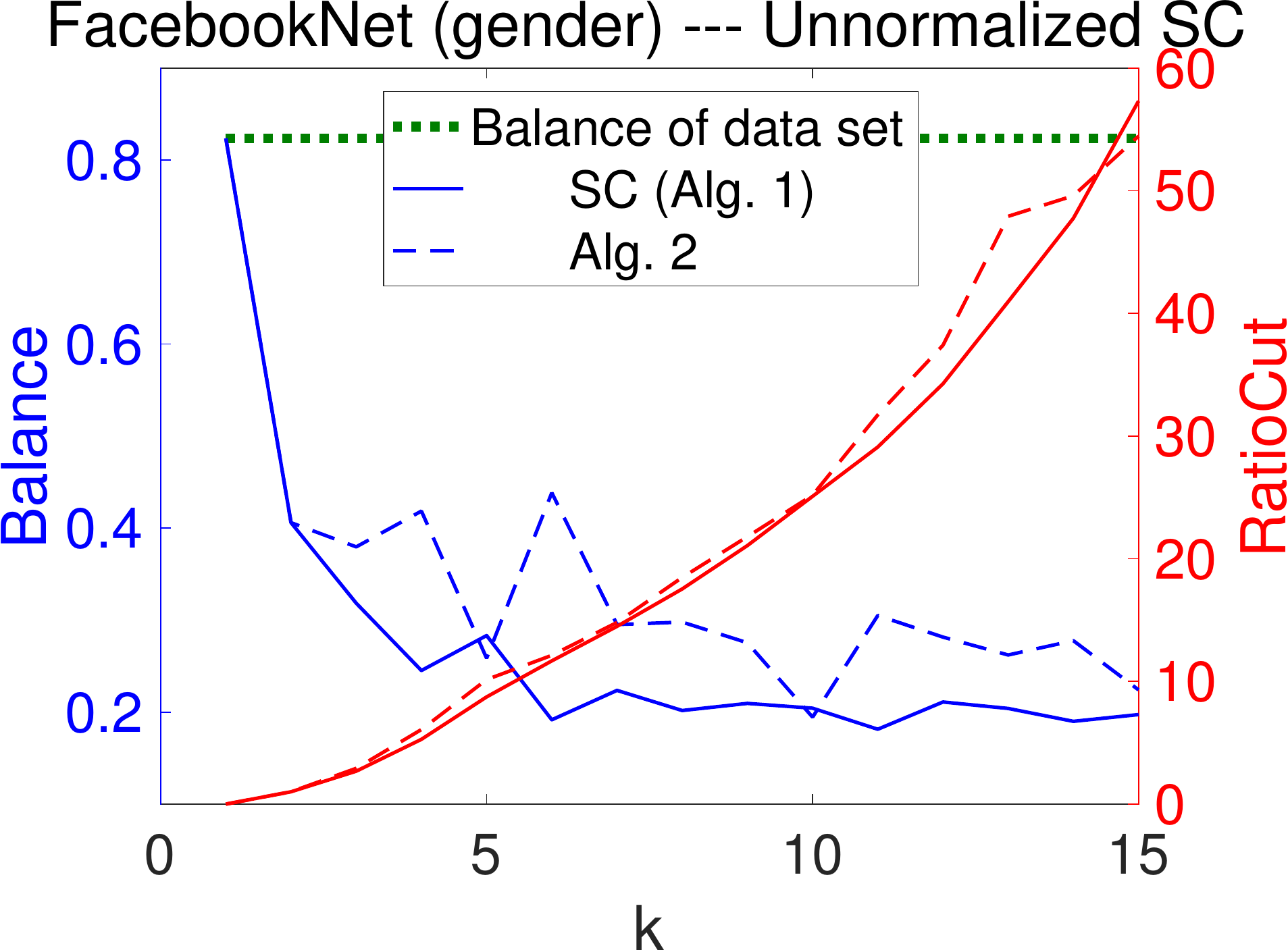}
\put(37,52.7){\includegraphics[scale=\sizeA]{Experiments_RealNetworks_NEW_NOTATION_2/legend_red_lines.pdf}}
\end{overpic}
\hspace{\sizeB}
\begin{overpic}[scale=\sizeA]{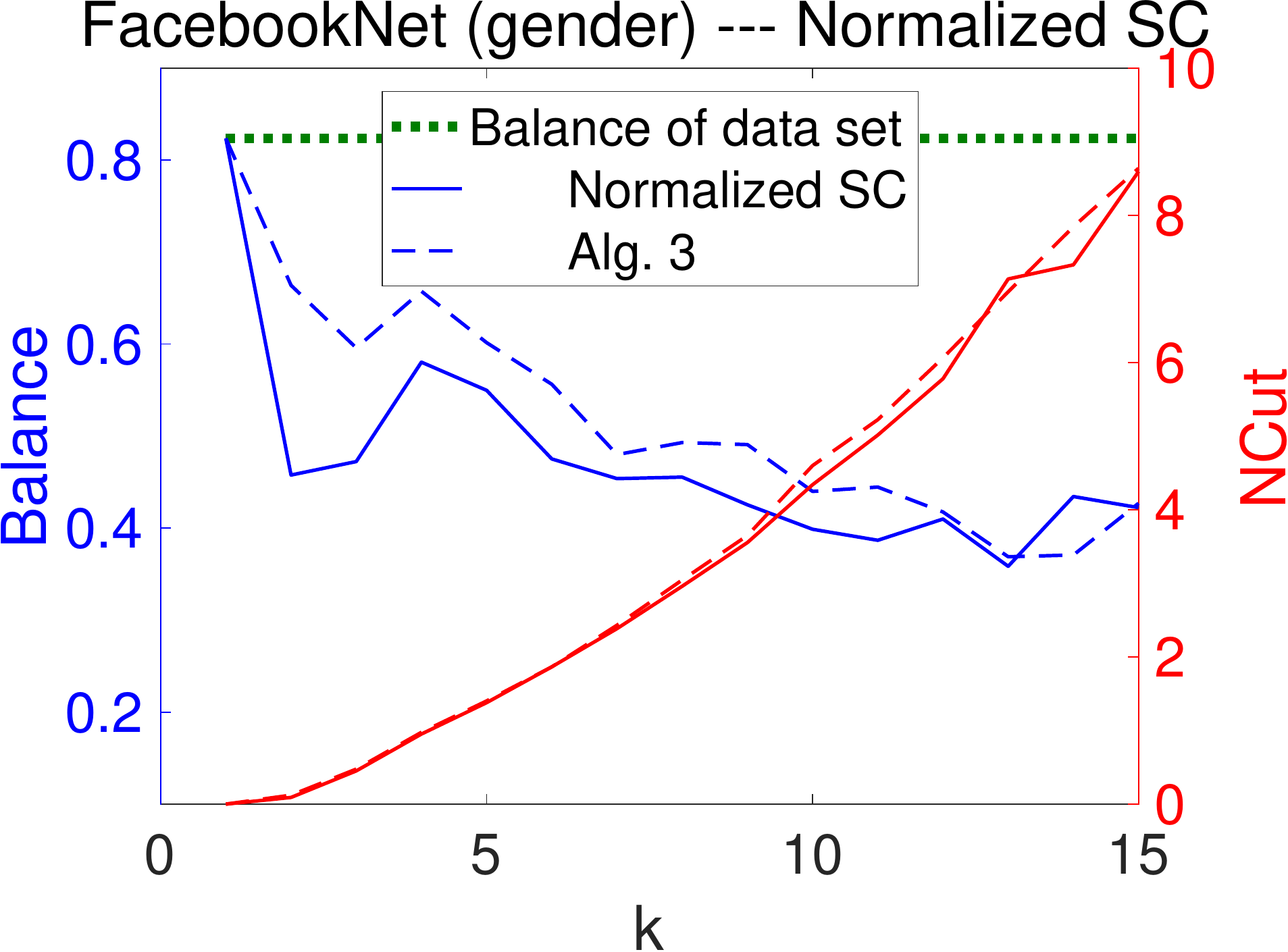}
\put(37,52.7){\includegraphics[scale=\sizeA]{Experiments_RealNetworks_NEW_NOTATION_2/legend_red_lines.pdf}}
\end{overpic}

\vspace{2mm}
\begin{overpic}[scale=\sizeA]{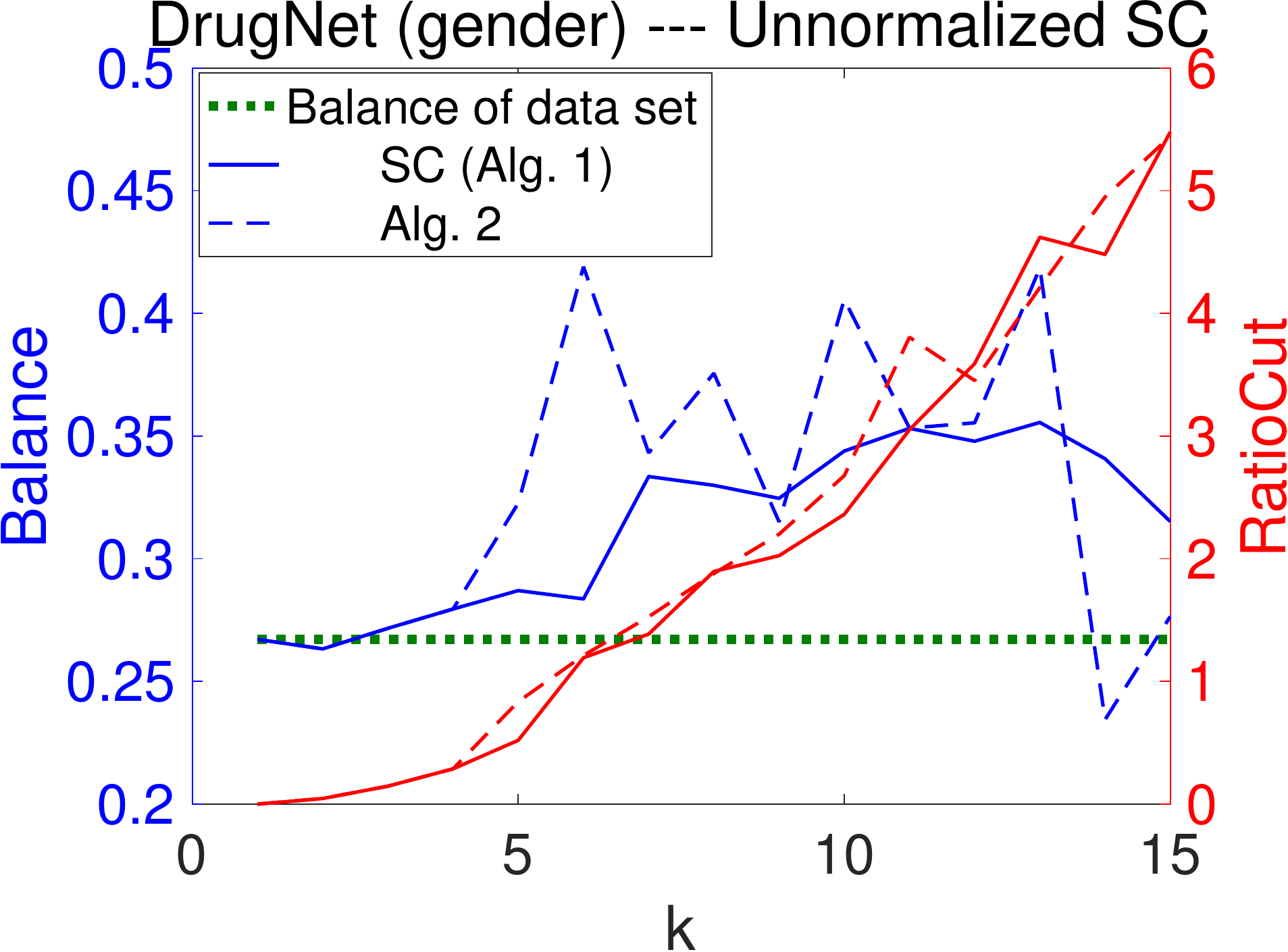}
\put(22.5,54.5){\includegraphics[scale=\sizeA]{Experiments_RealNetworks_NEW_NOTATION_2/legend_red_lines.pdf}}
\end{overpic}
\hspace{\sizeB}
\begin{overpic}[scale=\sizeA]{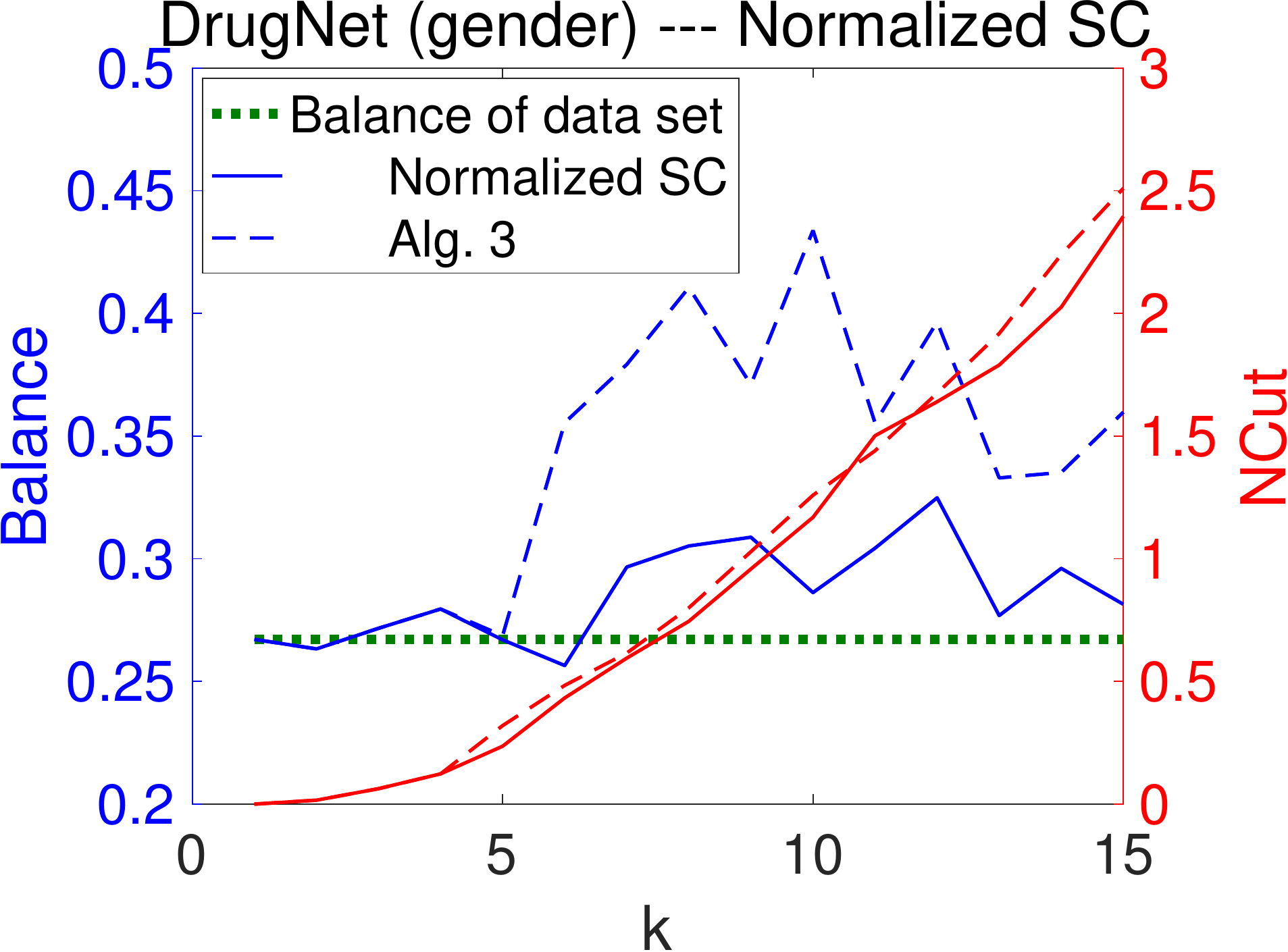}
\put(22.5,53.7){\includegraphics[scale=\sizeA]{Experiments_RealNetworks_NEW_NOTATION_2/legend_red_lines.pdf}}
\end{overpic}
\hspace{\sizeB}
\begin{overpic}[scale=\sizeA]{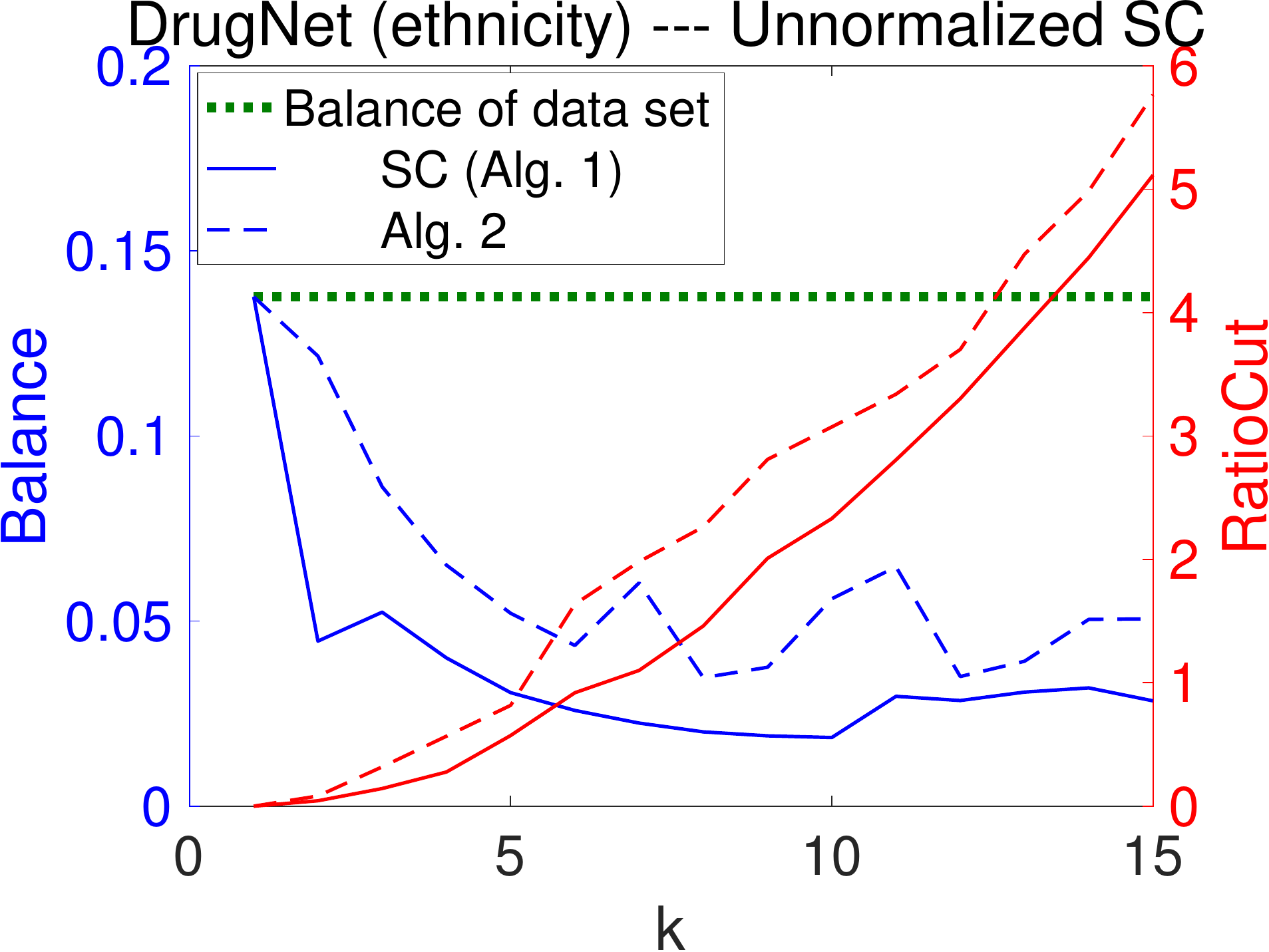}
\put(22.5,55.08){\includegraphics[scale=\sizeA]{Experiments_RealNetworks_NEW_NOTATION_2/legend_red_lines.pdf}}
\end{overpic}
\hspace{\sizeB}
\begin{overpic}[scale=\sizeA]{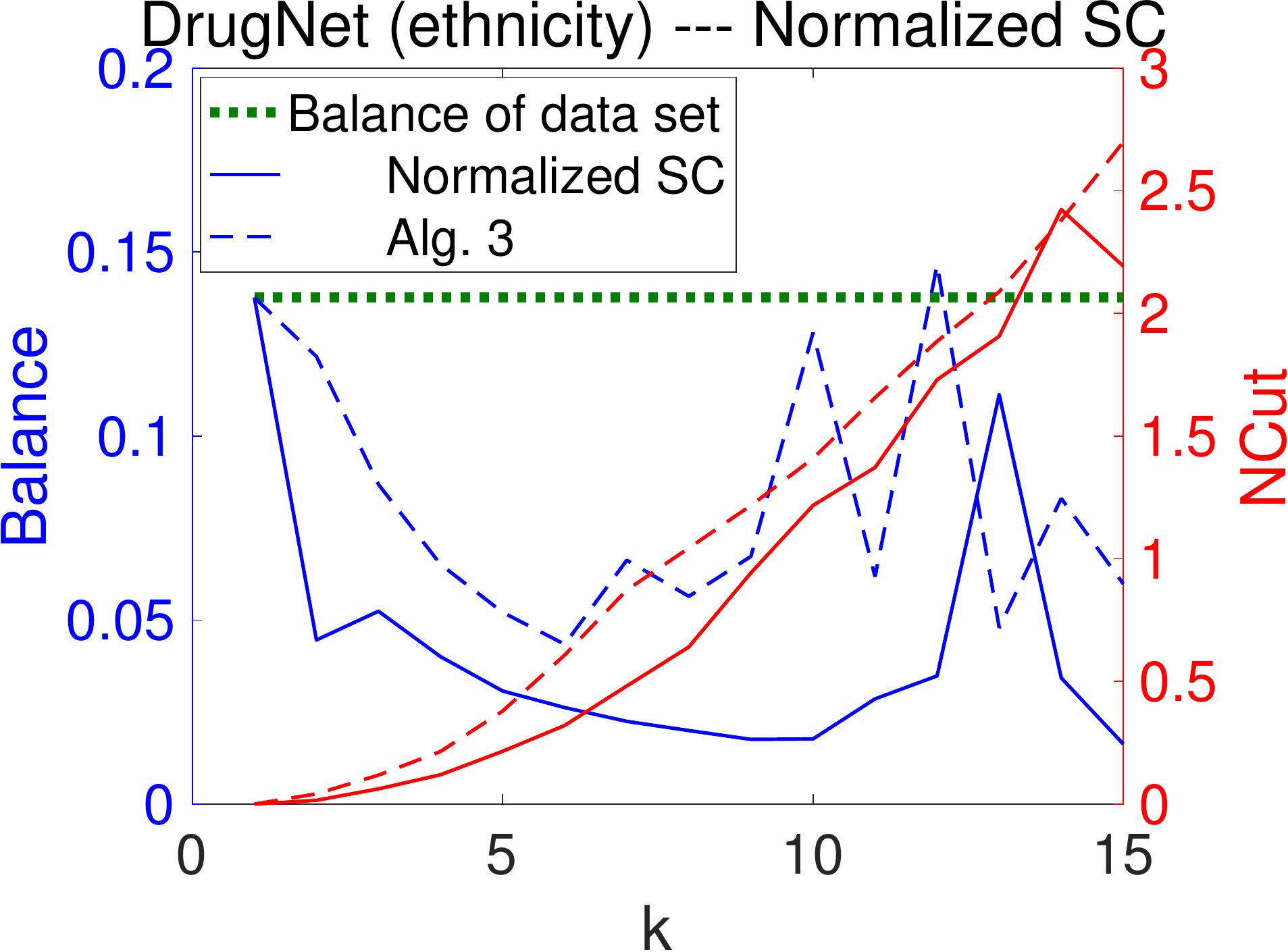}
\put(22.5,53.8){\includegraphics[scale=\sizeA]{Experiments_RealNetworks_NEW_NOTATION_2/legend_red_lines.pdf}}
\end{overpic}
}
\caption{Balance (left axis) and RatioCut / NCut value (right axis) of
standard SC and our fair versions
as a function of $k$ on
real networks.}\label{fig_exp_real_networks}
\end{figure*}

In the experiments of Figure~\ref{fig_exp_SBmodel_AsFuncOf_perturbation}, we consider a perturbation of our model as follows: first, for
$n=4000$ (left plot) or $n=2000$ (right plot),
$k=4$ and $h=2$ we generate a graph from our model
just
as before (Assumption~\eqref{bal_assum_theorem} is satisfied; in particular, the two groups have the same size), but then we assign some of the vertices in the first group to the other group. Concretely, for a perturbation parameter~$p\in[0,1]$,  each vertex in the first group is assigned to the second
one
with probability~$p$ independently of each other. The case $p=0$ is 
our model without any perturbation.
If
$p=1$, there is only one group and our algorithms technically coincide with standard unnormalized or normalized SC.
The
two plots
show the error of
our algorithms
and
standard SC
as a function of~$p$. Both our algorithms show the same behavior. They are robust against the perturbation up to $p=0.15$. They yield the same error as
standard SC
for $p\geq 0.7$.

\subsection{Real Data}\label{sec_real_experiments}

In the experiments of  Figure~\ref{fig_exp_real_networks},
we
evaluate
the performance of standard unnormalized and normalized SC versus our fair versions on real
network data. The quality of a clustering is measured through its
``Balance'' (defined as the average of the balance \eqref{def_balance} over all clusters; shown on left axis of the plots) and its
RatioCut~\eqref{def_ratio_cut}
or
NCut~\eqref{def_norm_cut}
value (right axis).
All networks that we are working with are the largest connected component of an originally unconnected network.

The first row of
Figure~\ref{fig_exp_real_networks} shows the results as a function of the number of clusters~$k$ for two high school friendship networks
\citep{mastrandrea2015contact}. Vertices correspond to students and are split into two groups of males and females.
\textsc{FriendshipNet}
has
127 vertices and  an edge
between two students
indicates that
one of them reported friendship with the other one.
\textsc{FacebookNet}
consists of
155 vertices and an edge between two students indicates friendship  on
Facebook. As we can see from the plots, compared to standard SC, our fair versions improve the output clustering's balance (by $10\%$ / $15\%$ / $34\%$ / $10\%$ on average over~$k$) while
almost not
changing its RatioCut or NCut~value.

The second row shows the results for
\textsc{DrugNet},
 a network encoding  acquaintanceship between drug users in Hartford,
CT
\citep{weeks2002social}.
In the left two plots, the network consists of  185
vertices split into two groups of males and females (we had to remove some vertices for which the gender was not known).
In the right two plots, the network has 193 vertices split into three ethnic groups of African Americans, Latinos and others. Again, our fair versions of SC quite 
significantly
improve the balance of the output clustering over standard SC (by $5\%$ / $18\%$ / $86\%$ / $167\%$ on average over~$k$). However, in the right two plots we also observe a  moderate increase of the RatioCut or NCut~value.


\section{Discussion}\label{sec_discussion}
In this work, we presented an
algorithmic
approach towards incorporating fairness constraints into the
SC framework. We provided a rigorous analysis of our algorithms and proved that they
can
recover fair
ground-truth
clusterings in a natural variant of the stochastic block model. Furthermore, we provided strong empirical evidence that often in real data sets, it is possible to achieve higher demographic proportionality at minimal additional cost in the clustering~objective.

An important direction for future work is to understand the price of fairness in the SC framework if one
 needs to satisfy the fairness constraints exactly. One way to achieve this would be to run the fair $k$-means
 algorithm of~\citet{sohler_kmeans} in the last step of our Algorithms~\ref{fair_SC_alg} or~\ref{fair_SC_alg_normalized}.
We want to point out that the algorithm of~\citeauthor{sohler_kmeans} currently does not extend beyond two
groups of the same size. Second, our experimental results on the stochastic block model provide evidence that our algorithms are robust to moderate levels of perturbations in the group assignments. Characterizing this robustness rigorously is an intriguing open problem.

\section*{Acknowledgements}

This research is supported by a Rutgers Research Council Grant and a Center for Discrete Mathematics and Theoretical Computer Science (DIMACS) postdoctoral fellowship.

\bibliography{mybibfile_fair_clustering}
\bibliographystyle{icml2019}


\clearpage

\icmltitlerunning{Appendix to Guarantees for Spectral Clustering with Fairness Constraints}
\onecolumn
\appendix

\section*{Appendix}\label{appendix}

\section{Adding Fairness Constraints to Normalized Spectral Clustering}\label{appendix_fair_SC_normalized}

In this section we derive a fair version of normalized spectral clustering (similarly to how we proceeded for unnormalized spectral clustering in Sections~\ref{sec_SC} and \ref{sec_Fair_SC} of the main paper).

Normalized spectral clustering aims at partitioning $V$ into $k$ clusters
with minimum value of
the NCut objective function
as follows \citep[see][for details]{Luxburg_tutorial}:
 for a clustering $V=C_1\dot{\cup}\ldots \dot{\cup}C_k$ we have
\begin{align}\label{def_norm_cut}
\NC(C_1,\ldots,C_k)=\sum_{l=1}^k \frac{\cut(C_l,V\setminus C_l)}{\vol(C_l)},
\end{align}
where
$\vol(C_l)=\sum_{i\in C_l} d_i=\sum_{i\in C_l,j\in[n]} W_{ij}$. Encoding a clustering $V=C_1\dot{\cup}\ldots \dot{\cup}C_k$ by a matrix
$H\in \R^{n\times k}$
with
\begin{align}\label{enc_clustering_normalized}
H_{il}= \begin{cases}
     1/\sqrt{\vol(C_l)}, & i\in C_l, \\
0,& i\notin C_l
   \end{cases},
\end{align}
we have $\NC(C_1,\ldots,C_k)=\trace(H^TLH)$. Note that any $H$ of the form \eqref{enc_clustering_normalized} satisfies $H^TDH=I_k$.
Normalized spectral clustering relaxes the problem of minimizing $\trace(H^TLH)$ over all $H$ of the form \eqref{enc_clustering_normalized} to
\begin{align}\label{relaxed_problem_normalized}
 \min_{H\in\R^{n\times k}}\trace(H^T L H) ~~ \text{subject to~}H^TDH=I_k.
\end{align}
Substituting $H=D^{-1/2}T$ for $T\in\R^{n\times k}$ (we need to assume that $G$ does not contain any isolated vertices since otherwise $D^{-1/2}$ does not exist), problem~\eqref{relaxed_problem_normalized} becomes
\begin{align*}
 \min_{T\in\R^{n\times k}}\trace(T^TD^{-1/2} LD^{-1/2} T) ~~ \text{subject to~}T^TT=I_k.
\end{align*}
Similarly to unnormalized spectral clustering, normalized spectral clustering computes an optimal $T$ by computing the $k$ smallest eigenvalues and some corresponding eigenvectors of $D^{-1/2} LD^{-1/2}$ and applies $k$-means clustering to the rows of $H=D^{-1/2}T$ (in practice, $H$ can be computed directly by solving the generalized eigenproblem $Lx=\lambda Dx$, $x\in\R^n$, $\lambda\in\R$; see \citealp{Luxburg_tutorial}).

\vspace{2mm}
Now we want to derive our fair version of normalized spectral clustering. The first step is to show that
Lemma~\ref{lemma_most_gen_const} holds true if we encode a clustering as in \eqref{enc_clustering_normalized}:

\begin{lemma}[Fairness constraint as linear constraint on $H$ for normalized spectral clustering]\label{lemma_most_gen_const_normalized}
For $s\in[h]$, let $f^{(s)}\in \{0,1\}^n$ be the group-membership vector of $V_s$, that is $f^{(s)}_i=1$ if $i\in V_s$ and $f^{(s)}_i=0$ otherwise. Let $V=C_1\dot{\cup}\ldots \dot{\cup}C_k$ be a clustering that is encoded as in \eqref{enc_clustering_normalized}.
We have,
for every $l\in[k]$,
 \begin{align*}
 \forall s\in[h-1]: \sum_{i=1}^n \left(f^{(s)}_i-\frac{|V_s|}{n}\right) H_{il}=0  ~~~~\Leftrightarrow~~~~\forall s\in[h]:\frac{|V_s\cap C_l|}{|C_l|}=\frac{|V_s|}{n}.
 \end{align*}
\end{lemma}

\begin{proof}
This simply follows from
\begin{align*}
 \sum_{i=1}^n \left(f^{(s)}_i-\frac{|V_s|}{n}\right) H_{il}=\frac{|V_s\cap C_l|}{\sqrt{\vol(C_l)}}-\frac{|V_s|\cdot|C_l|}{n\sqrt{\vol(C_l)}}
 \end{align*}
and
$|C_l|=\sum_{s=1}^h |V_s\cap C_l|$.
\end{proof}

\begin{algorithm}[t!]
   \caption{
   Normalized SC with fairness constraints
   }\label{fair_SC_alg_normalized}
\begin{algorithmic}
   \STATE {\bfseries Input:} weighted adjacency matrix $W\in\R^{n\times n}$
   \textbf{(the underlying graph must not contain any isolated vertices)};
   $k\in\N$; group-membership vectors  $f^{(s)}\in \{0,1\}^n$, $s\in [h]$

\vspace{1mm}
   \STATE {\bfseries Output:} a clustering of $[n]$ into $k$ clusters

   \begin{itemize}[leftmargin=*]
   \setlength{\itemsep}{-2pt}
   \item compute
the Laplacian matrix $L=D-W$ with the degree matrix $D$
\item build the matrix $F$ that has the vectors $f^{(s)}-\frac{|V_s|}{n}\cdot\mathbf{1}_n$, $s\in[h-1]$, as columns
\item compute a matrix $Z$  whose columns form an orthonormal basis of the nullspace of $F^T$
 \item compute the square root $Q$ of $Z^TDZ$
  \item compute some  orthonormal eigenvectors corresponding to the $k$ smallest eigenvalues (respecting multiplicities) of $Q^{-1}Z^TLZQ^{-1}$
  \item let $X$ be a matrix containing these eigenvectors as columns
  \item apply $k$-means clustering to the rows of $H=ZQ^{-1}X\in\R^{n\times k}$, which yields a clustering   of $[n]$ into $k$ clusters
   \end{itemize}
\end{algorithmic}
\end{algorithm}

Lemma \ref{lemma_most_gen_const_normalized} suggests that in a fair version of normalized spectral clustering, rather than solving \eqref{relaxed_problem_normalized}, we should solve
\begin{align}\label{relaxed_problem_normalized_fair}
 \min_{H\in\R^{n\times k}}\trace(H^T L H) ~~ \text{subject to~$H^TDH=I_k$~and $F^TH=0_{(h-1)\times k}$},
\end{align}
where $F\in\R^{n\times (h-1)}$ is the matrix that has the vectors $f^{(s)}-(|V_s|/n)\cdot\mathbf{1}_n$, $s\in[h-1]$, as columns (just as in Section~\ref{sec_Fair_SC}).
It is $\rank(F)=\rank(F^T)=h-1$ and
we need to assume that $k\leq n-h+1$ since otherwise \eqref{relaxed_problem_normalized_fair} does not have any solution.
 Let $Z\in \R^{n\times(n-h+1)}$ be a matrix whose columns form an orthonormal basis of the nullspace of $F^T$. We  substitute $H=ZY$ for  $Y\in \R^{(n-h+1)\times k}$, and then problem \eqref{relaxed_problem_normalized_fair} becomes
\begin{align}\label{relaxed_problem_normalized_fair_subst1}
 \min_{Y\in \R^{(n-h+1)\times k}}\trace(Y^TZ^T LZY) ~~ \text{subject to~$Y^TZ^TDZY=I_k$}.
\end{align}
Assuming that $G$ does not contain any isolated vertices, $Z^TDZ$ is
positive definite and hence has a
positive definite square root, that is there exists a
positive definite $Q\in \R^{(n-h+1)\times (n-h+1)}$ with  $Z^TDZ=Q^2$. We can substitute $Y=Q^{-1}X$ for   $X\in \R^{(n-h+1)\times k}$,
and then
problem  \eqref{relaxed_problem_normalized_fair_subst1} becomes
\begin{align}\label{relaxed_problem_normalized_fair_subst2}
 \min_{X\in \R^{(n-h+1)\times k}}\trace(X^TQ^{-1}Z^T LZQ^{-1}X) ~~ \text{subject to~$X^TX=I_k$}.
\end{align}
A  solution to \eqref{relaxed_problem_normalized_fair_subst2} is given by a matrix $X$ that contains some
orthonormal eigenvectors corresponding to the $k$ smallest eigenvalues
(respecting multiplicities) of $Q^{-1}Z^T LZQ^{-1}$ as columns. This gives rise to our fair version of normalized spectral clustering as stated in Algorithm~\ref{fair_SC_alg_normalized}.

%
%

\vspace{6mm}
\section{Computational Complexity of our Algorithms}\label{appendix_complexity}

The costs of standard
spectral clustering
(e.g., Algorithm~\ref{SC_alg}) are
dominated by the complexity of the eigenvector computations and are
 commonly stated to be, in general, in $\mathcal{O}(n^3)$ regarding time and $\mathcal{O}(n^2)$ regarding space for an arbitrary number of clusters~$k$, unless approximations are applied \citep{yan2009,li2011}. In addition to the computations performed in
 Algorithm~\ref{SC_alg},
 in 
 Algorithm~\ref{fair_SC_alg} and Algorithm~\ref{fair_SC_alg_normalized}
 we have to compute an orthonormal basis of the nullspace of $F^T$, perform some matrix multiplications, and (only for Algorithm~\ref{fair_SC_alg_normalized})
 compute the square root of an   $(n-h+1)\times  (n-h+1)$-matrix and the inverse of this square root.
 All these computations can be done in $\mathcal{O}(n^3)$ regarding time and $\mathcal{O}(n^2)$ regarding space (an orthonormal basis of the nullspace
of $F^T$
 can be computed by means of an SVD; see, e.g., \citealp{golub_book}), and hence
our algorithms have  the same worst-case complexity as standard
spectral clustering.
On the other hand, if the graph~$G$, and thus the Laplacian
matrix~$L$, is sparse or $k$ is small, then the eigenvector computations in Algorithm~\ref{SC_alg} can be done more efficiently than
with cubic running time \citep{bai2000}. This is not the
case for our algorithms as stated. However, one could apply one of the techniques suggested in the existing literature
on constrained spectral clustering to speed up computation (e.g., \citealp{stella_journal_2004}, or \citealp{xu2009}; see Section~\ref{sec_related_work} of the main paper).
With the implementations as stated, in our experiments in Section~\ref{sec_experiments} of the main paper we observe that Algorithm~\ref{fair_SC_alg} has a similar running time as standard normalized spectral clustering
while the running time of Algorithm~\ref{fair_SC_alg_normalized}
is significantly higher.

\vspace{4mm}
\section{Proof of Theorem \ref{theorem_SB_model}}\label{appendix_proofs}

We split the proof of Theorem~\ref{theorem_SB_model} into four parts. In the first part, we analyze the eigenvalues and eigenvectors of the expected adjacency matrix~$\mathcal{W}$ and of the matrix
$Z^T\mathcal{L}Z$, where $\mathcal{L}$ is the expected Laplacian matrix and $Z$ is the matrix computed
in the execution of Algorithm~\ref{fair_SC_alg} or Algorithm~\ref{fair_SC_alg_normalized}. In the second part, we study the
deviation of the observed matrix $Z^TLZ$ from the expected matrix $Z^T\mathcal{L}Z$.
In the third part, we use the results from the first and the second part to prove Theorem~\ref{theorem_SB_model} for Algorithm~\ref{fair_SC_alg} (unnormalized SC with fairness constraints).  In the fourth part, we prove Theorem~\ref{theorem_SB_model} for Algorithm~\ref{fair_SC_alg_normalized} (normalized SC with fairness constraints).

\paragraph{Notation}
For $x \in\R^n$, by $\|x\|$ we denote the Euclidean norm of $x$, that is $\|x\|=\sqrt{x_1^2+\ldots+x_n^2}$. For $A\in \R^{n\times m}$, by $\|A\|$ we denote the
operator norm (also known as spectral norm) and by $\|A\|_F$ the Frobenius norm of $A$.
It is
\begin{align}\label{op_norm}
 \|A\|=\max_{x\in\R^m:\|x\|=1}\|Ax\|=\sqrt{\lambda_{\text{max}}(A^TA)},
\end{align}
where $\lambda_{\text{max}}(A^TA)$ is the largest eigenvalue of $A^TA$, and
\begin{align}\label{frob_norm}
 \|A\|_F=\sqrt{\sum_{i=1}^n \sum_{j=1}^m A_{ij}^{\;2}}=\sqrt{\trace(A^TA)}.
\end{align}
Note that
for a symmetric matrix $A\in\R^{n\times n}$ with $A=A^T$ we have $\|A\|=\max\{|\lambda_i|:\lambda_i\text{~is an eigenvalue of~}A\}$. It follows from
\eqref{op_norm} and \eqref{frob_norm} that for any $A\in \R^{n\times m}$ with rank at most $r$ we have
\begin{align}\label{absch_frob_norm}
\|A\|\leq  \|A\|_F \leq \sqrt{r}\|A\|.
\end{align}
We use $\const(X)$ to denote a universal constant that only depends on $X$ and that may change from line to line.

\vspace{8mm}
\textbf{Part 1: Eigenvalues and eigenvectors of $\mathcal{W}$ and of $Z^T\mathcal{L}Z$}

 Assuming the $n$ vertices $1,\ldots,n$ are sorted in a way such that
\begin{align}\label{order_vertices}
\begin{split}
1,\ldots,\frac{n}{kh}\in C_1\cap V_1,~~~~\frac{n}{kh}+1,\ldots,\frac{2n}{kh} \in C_1\cap V_2,~~~~\ldots~\ldots~\ldots,~~~~\frac{(h-1)n}{kh}+1,\ldots,\frac{n}{k}\in C_1\cap V_h,\\
\frac{n}{k}+1,\ldots,\frac{n}{k}+\frac{n}{kh}\in C_2\cap V_1,~~~~\ldots~\ldots~\ldots,~~~~\frac{n}{k}+\frac{(h-1)n}{kh}+1,\ldots,\frac{2n}{k}\in C_2\cap V_h,\\
\vdots~~~~~~~~~~~~~~~~~~~~~~~~~~~~~~~~~~~~~~~~~~~~~~~~~~~~~~~~~~~~~~~~~~~~~~~~~~~~~~~~~~~\\
\frac{(k-1)n}{k}+1,\ldots,\frac{(k-1)n}{k}+\frac{n}{kh}\in C_k\cap V_1,~~~~\ldots~\ldots~\ldots,~~~~\frac{(k-1)n}{k}+\frac{(h-1)n}{kh}+1,\ldots,n\in C_k\cap V_h,\\
\end{split}
\end{align}
the expected adjacency matrix $\mathcal{W}\in\R^{n\times n}$ is given by the block matrix
\begin{align}\label{qwqwqwqwqw}
\renewcommand\arraystretch{1.5}
\mathcal{W}=\underbrace{\begin{pmatrix}
[R] & [S] & [S] & [S] & \cdots & [S] & [S]\\
[S] & [R] & [S] & [S] & \cdots & [S] & [S]\\
[S] & [S] & [R] & [S] & \cdots & [S] & [S]\\
&\vdots&& \ddots &&\vdots&\\
[S] & [S] & [S] & [S] & \cdots & [S] & [R]\\
\end{pmatrix}}_{=:\widetilde{\mathcal{W}}}-a I_n,
\end{align}
where $[R]\in\{a,c\}^{\frac{n}{k}\times \frac{n}{k}}$ and $[S]\in\{b,d\}^{\frac{n}{k}\times \frac{n}{k}}$ are themselves block matrices
\begin{align*}
\renewcommand\arraystretch{1.5}
[R]=\begin{pmatrix}
[a] & [c] & [c] & [c] & \cdots & [c] & [c]\\
[c] & [a] & [c] & [c] & \cdots & [c] & [c]\\
[c] & [c] & [a] & [c] & \cdots & [c] & [c]\\
&\vdots&& \ddots &&\vdots&\\
[c] & [c] & [c] & [c] & \cdots & [c] & [a]\\
\end{pmatrix},\qquad
[S]=\begin{pmatrix}
[b] & [d] & [d] & [d] & \cdots & [d] & [d]\\
[d] & [b] & [d] & [d] & \cdots & [d] & [d]\\
[d] & [d] & [b] & [d] & \cdots & [d] & [d]\\
&\vdots&& \ddots &&\vdots&\\
[d] & [d] & [d] & [d] & \cdots & [d] & [b]\\
\end{pmatrix}
\end{align*}
with
$[a]$, $[b]$, $[c]$ and $[d]$ being matrices of size $\frac{n}{kh}\times \frac{n}{kh}$ with all entries equaling $a$, $b$, $c$ and $d$, respectively. We denote the matrix $\mathcal{W}+aI_n$ with $\mathcal{W}$
as in \eqref{qwqwqwqwqw} by $\widetilde{\mathcal{W}}$.
 If the vertices are not ordered as in \eqref{order_vertices}, the expected adjacency matrix~$\mathcal{W}$ is rather given by
$\mathcal{W}=P^T\widetilde{\mathcal{W}}P-aI_n$ for some  permutation matrix $P\in\{0,1\}^{n\times n}$ with $PP^T=P^TP=I_n$.
Note that $v\in\R^n$ is an eigenvector of $\widetilde{\mathcal{W}}$ with eigenvalue $\lambda$ if and only if $P^Tv$ is an eigenvector of $P^T\widetilde{\mathcal{W}}P$ with eigenvalue $\lambda$.
Keeping track of the permutation matrices $P$ and $P^T$ throughout the proof of  Theorem~\ref{theorem_SB_model} does not impose any technical challenges, but makes the writing more complicated. Hence, for simplicity and without loss of generality, we assume in the following
that the vertices are ordered as in \eqref{order_vertices}.

The following lemma characterizes the eigenvalues and eigenvectors of $\widetilde{\mathcal{W}}$. Clearly, this also characterizes the eigenvalues and eigenvectors of $\mathcal{W}$:
$v\in\R^n$ is an eigenvector of $\widetilde{\mathcal{W}}$ with eigenvalue $\lambda$ if and only if $v$ is an eigenvector of $\mathcal{W}$ with eigenvalue $\lambda-a$.

\begin{lemma}\label{lemma_spectrum_h_groups}
Assuming that $a>b>c>d\geq 0$, the matrix $\widetilde{\mathcal{W}}$ has rank $kh$ or rank $k+h-1$ (the latter is true if and only if $a-c=b-d$). It has the following eigenvalues $\lambda_1,\ldots,\lambda_n$:
\begin{align}\label{eigenvalues_h_groups}
\begin{split}
\lambda_1&=\frac{n}{kh}\left[\big(a+(h-1)c\big)+(k-1)\big(b+(h-1)d\big)\right],\\
\lambda_2=\lambda_3=\ldots=\lambda_h&=\frac{n}{kh}\left[\big(a-c\big)+(k-1)\big(b-d\big)\right],\\
\lambda_{h+1}=\lambda_{h+2}=\ldots=\lambda_{h+k-1}&=\frac{n}{kh}\left[\big(a+(h-1)c\big)-\big(b+(h-1)d\big)\right],\\
\lambda_{h+k},\lambda_{h+k+1}=\ldots=\lambda_{hk}&=\frac{n}{kh}\left[\big(a-c\big)-\big(b-d\big)\right],\\
\lambda_{hk+1}=\lambda_{hk+2}=\ldots=\lambda_{n}&=0.
\end{split}
\end{align}
It is $\lambda_1>\lambda_2=\ldots=\lambda_h>0$ as well as  $\lambda_1>\lambda_{h+1}=\ldots=\lambda_{h+k-1}>0$ and  $\lambda_2=\ldots=\lambda_h>|\lambda_{h+k}|=\ldots=|\lambda_{hk}|$
as well as $\lambda_{h+1}=\ldots=\lambda_{h+k-1}>|\lambda_{h+k}|=\ldots=|\lambda_{hk}|$.

\vspace{2mm}
An eigenvector corresponding to $\lambda_1$ is given by $v_1=\mathbf{1}_n$. The eigenspace corresponding to the eigenvalue $\lambda_2=\ldots=\lambda_h$ is spanned by the vectors $v_2,\ldots,v_h$ with
\begin{align*}
\renewcommand\arraystretch{1.5}
v_2=\begin{pmatrix}
[1]\\
[-\frac{1}{h-1}]\\
[-\frac{1}{h-1}]\\
[-\frac{1}{h-1}]\\
\vdots\\
[-\frac{1}{h-1}]\\
\\
[1]\\
[-\frac{1}{h-1}]\\
[-\frac{1}{h-1}]\\
[-\frac{1}{h-1}]\\
\vdots\\
[-\frac{1}{h-1}]\\
\\
\vdots\\
\\
[1]\\
[-\frac{1}{h-1}]\\
[-\frac{1}{h-1}]\\
[-\frac{1}{h-1}]\\
\vdots\\
[-\frac{1}{h-1}]
\end{pmatrix},~~~
v_3=
\begin{pmatrix}
[-\frac{1}{h-1}]\\
[1]\\
[-\frac{1}{h-1}]\\
[-\frac{1}{h-1}]\\
\vdots\\
[-\frac{1}{h-1}]\\
\\
[-\frac{1}{h-1}]\\
[1]\\
[-\frac{1}{h-1}]\\
[-\frac{1}{h-1}]\\
\vdots\\
[-\frac{1}{h-1}]\\
\\
\vdots\\
\\
[-\frac{1}{h-1}]\\
[1]\\
[-\frac{1}{h-1}]\\
[-\frac{1}{h-1}]\\
\vdots\\
[-\frac{1}{h-1}]
\end{pmatrix},~~~~\cdots,~~~~
v_h=\begin{pmatrix}
[-\frac{1}{h-1}]\\
[-\frac{1}{h-1}]\\
\vdots\\
[-\frac{1}{h-1}]\\
[1]\\
[-\frac{1}{h-1}]\\
\\
[-\frac{1}{h-1}]\\
[-\frac{1}{h-1}]\\
\vdots\\
[-\frac{1}{h-1}]\\
[1]\\
[-\frac{1}{h-1}]\\
\\
\vdots\\
\\
[-\frac{1}{h-1}]\\
[-\frac{1}{h-1}]\\
\vdots\\
[-\frac{1}{h-1}]\\
[1]\\
[-\frac{1}{h-1}]
\end{pmatrix},\\
\end{align*}
where for $z\in\R$, by $[z]$ we denote a block of size $\frac{n}{kh}$ with all entries equaling $z$.
The eigenspace corresponding to the eigenvalue $\lambda_{h+1}=\ldots=\lambda_{h+k-1}$ is spanned by the vectors $v_{h+1},\ldots,v_{h+k-1}$ with
\begin{align*}
\renewcommand\arraystretch{1.5}
v_{h+1}=\begin{pmatrix}
[1]\\
[-\frac{1}{k-1}]\\
[-\frac{1}{k-1}]\\
\vdots\\
[-\frac{1}{k-1}]\\
[-\frac{1}{k-1}]\\
[-\frac{1}{k-1}]
\end{pmatrix},~~~
v_{h+2}\begin{pmatrix}
[-\frac{1}{k-1}]\\
[1]\\
[-\frac{1}{k-1}]\\
\vdots\\
[-\frac{1}{k-1}]\\
[-\frac{1}{k-1}]\\
[-\frac{1}{k-1}]
\end{pmatrix},~~~~\cdots,~~~~
v_{h+k-1}=\begin{pmatrix}
[-\frac{1}{k-1}]\\
[-\frac{1}{k-1}]\\
[-\frac{1}{k-1}]\\
\vdots\\
[-\frac{1}{k-1}]\\
[1]\\
[-\frac{1}{k-1}]
\end{pmatrix},\\
\end{align*}
where for $z\in\R$, by $[z]$ we denote a block of size $\frac{n}{k}$ with all entries equaling $z$.
The eigenspace corresponding to the eigenvalue $\lambda_{h+k}=\ldots=\lambda_{hk}$ is spanned by the vectors $v_{h+k},\ldots,v_{hk}$ with
\begin{align*}
\renewcommand\arraystretch{1.5}
\underbrace{\underbrace{
\underbrace{\begin{pmatrix}
[1]\\
[-\frac{1}{h-1}]\\
[-\frac{1}{h-1}]\\
\vdots\\
[-\frac{1}{h-1}]\\
[-\frac{1}{h-1}]\\
\\
[-1]\\
[\frac{1}{h-1}]\\
[\frac{1}{h-1}]\\
\vdots\\
[\frac{1}{h-1}]\\
[\frac{1}{h-1}]\\
\\
[0]\\
\vdots\\
[0]\\
\\
[0]\\
\vdots\\
[0]]\\
\\
\vdots
\end{pmatrix}}_{=v_{h+k}},
\underbrace{\begin{pmatrix}
[-\frac{1}{h-1}]\\
[1]\\
[-\frac{1}{h-1}]\\
\vdots\\
[-\frac{1}{h-1}]\\
[-\frac{1}{h-1}]\\
\\
[\frac{1}{h-1}]\\
[-1]\\
[\frac{1}{h-1}]\\
\vdots\\
[\frac{1}{h-1}]\\
[\frac{1}{h-1}]\\
\\
[0]\\
\vdots\\
[0]\\
\\
[0]\\
\vdots\\
[0]\\
\\
\vdots
\end{pmatrix}}_{=v_{h+k+1}},\cdots,
\underbrace{\begin{pmatrix}
[-\frac{1}{h-1}]\\
[-\frac{1}{h-1}]\\
\vdots\\
[-\frac{1}{h-1}]\\
[1]\\
[-\frac{1}{h-1}]\\
\\
[\frac{1}{h-1}]\\
[\frac{1}{h-1}]\\
\vdots\\
[\frac{1}{h-1}]\\
[-1]\\
[\frac{1}{h-1}]\\
\\
[0]\\
\vdots\\
[0]\\
\\
[0]\\
\vdots\\
[0]\\
\\
\vdots
\end{pmatrix}}_{=v_{h+k+(h-2)}}}_{h-1~\text{many}},~~~~~
\underbrace{\underbrace{\begin{pmatrix}
[1]\\
[-\frac{1}{h-1}]\\
[-\frac{1}{h-1}]\\
\vdots\\
[-\frac{1}{h-1}]\\
[-\frac{1}{h-1}]\\
\\
[0]\\
\vdots\\
[0]\\
\\
[-1]\\
[\frac{1}{h-1}]\\
[\frac{1}{h-1}]\\
\vdots\\
[\frac{1}{h-1}]\\
[\frac{1}{h-1}]\\
\\
[0]\\
\vdots\\
[0]\\
\\
\vdots
\end{pmatrix}}_{=v_{h+k+(h-1)}},
\cdots,
\underbrace{\begin{pmatrix}
[-\frac{1}{h-1}]\\
[-\frac{1}{h-1}]\\
\vdots\\
[-\frac{1}{h-1}]\\
[1]\\
[-\frac{1}{h-1}]\\
\\
[0]\\
\vdots\\
[0]\\
\\
[\frac{1}{h-1}]\\
[\frac{1}{h-1}]\\
\vdots\\
[\frac{1}{h-1}]\\
[-1]\\
[\frac{1}{h-1}]\\
\\
[0]\\
\vdots\\
[0]\\
\\
\vdots
\end{pmatrix}}_{=v_{h+k+(2h-3)}}}_{h-1~\text{many}},~~~~~~\cdots~~~~~~
\underbrace{\underbrace{\begin{pmatrix}
[1]\\
[-\frac{1}{h-1}]\\
[-\frac{1}{h-1}]\\
\vdots\\
[-\frac{1}{h-1}]\\
[-\frac{1}{h-1}]\\
\\
[0]\\
\vdots\\
[0]\\
\\
\vdots\\
\\
[0]\\
\vdots\\
[0]\\
\\
[-1]\\
[\frac{1}{h-1}]\\
[\frac{1}{h-1}]\\
\vdots\\
[\frac{1}{h-1}]\\
[\frac{1}{h-1}]\\
\end{pmatrix}}_{=v_{hk-h+2}},
\cdots,
\underbrace{\begin{pmatrix}
[-\frac{1}{h-1}]\\
[-\frac{1}{h-1}]\\
\vdots\\
[-\frac{1}{h-1}]\\
[1]\\
[-\frac{1}{h-1}]\\
\\
[0]\\
\vdots\\
[0]\\
\\
\vdots\\
\\
[0]\\
\vdots\\
[0]\\
\\
[\frac{1}{h-1}]\\
[\frac{1}{h-1}]\\
\vdots\\
[\frac{1}{h-1}]\\
[-1]\\
[\frac{1}{h-1}]\\
\end{pmatrix}}_{=v_{hk}}}_{h-1~\text{many}}}_{(k-1)(h-1)~\text{many}},
\end{align*}
where for $z\in\R$, by $[z]$ we denote a block of size $\frac{n}{kh}$ with all entries equaling $z$.
\end{lemma}

\begin{proof}
The matrix $\widetilde{\mathcal{W}}$ has only $kh$ different columns and hence $\rank\widetilde{\mathcal{W}}\leq kh$ and there are at most $kh$ many non-zero eigenvalues.
The statement about the rank of $\widetilde{\mathcal{W}}$ follows from the statement about the eigenvalues of $\widetilde{\mathcal{W}}$.

\vspace{2mm}
It is easy to verify that any of the vectors $v_i$ is actually an eigenvector of $\widetilde{\mathcal{W}}$ corresponding to eigenvalue $\lambda_i$, $i\in[hk]$. We need to show that the vectors
 $v_2,\ldots,v_h$,  the vectors $v_{h+1},\ldots,v_{h+k-1}$, as well as the vectors $v_{h+k},\ldots,v_{hk}$ are linearly independent.
 For example, let us show that $v_2,\ldots,v_h$ are linearly independent: assume that
$\sum_{j\in\{2,\ldots,h\}}\alpha_j v_j=0$ for some $\alpha_j\in\R$. We need to show that $\alpha_j=0$, $j\in\{2,\ldots,h\}$.
Looking at the $n$-th coordinate of $\sum_{j\in\{2,\ldots,h\}}\alpha_j v_j$, we infer that $0=-\frac{1}{h-1}\sum_{j\in\{2,\ldots,h\}}\alpha_j$. Looking at a coordinate where $v_i$ is $1$, we infer that
 $$0=\alpha_i-\frac{1}{h-1}\sum_{j\in\{2,\ldots,h\}\setminus\{i\}}\alpha_j=\alpha_i\left(1+\frac{1}{h-1}\right)-\frac{1}{h-1}\sum_{j\in\{2,\ldots,h\}}\alpha_j=\alpha_i\left(1+\frac{1}{h-1}\right)$$
and hence $\alpha_i=0$, $i\in\{2,\ldots,h\}$. Similarly, we can show that the vectors $v_{h+1},\ldots,v_{h+k-1}$ as well as the vectors $v_{h+k},\ldots,v_{hk}$ are linearly independent.

\vspace{2mm}
Since we assume that $a>b>c>d\geq 0$, we  have
\begin{align*}
\big(a+(h-1)c\big)+(k-1)\big(b+(h-1)d\big)>\big(a-c\big)+(k-1)\big(b-d\big)>0
\end{align*}
and
\begin{align*}
\big(a+(h-1)c\big)+(k-1)\big(b+(h-1)d\big)>\big(a+(h-1)c\big)-\big(b+(h-1)d\big)=(a-b)+(h-1)(c-d)>0,
\end{align*}
which shows that $\lambda_1>\lambda_2=\ldots=\lambda_h>0$ and  $\lambda_1>\lambda_{h+1}=\ldots=\lambda_{h+k-1}>0$.
It is
\begin{align*}
\big(a-c\big)+(k-1)\big(b-d\big)>(a-c)-(b-d) \qquad\text{and}\qquad
\big(a-c\big)+(k-1)\big(b-d\big)>-(a-c)+(b-d),
\end{align*}
and also
\begin{align*}
\big(a+(h-1)c\big)-\big(b+(h-1)d\big)&=(a-b)+(h-1)(c-d)\\
&\geq (a-b)+(c-d) >(a-b)+(d-c)=(a-c)-(b-d)
\end{align*}
and
\begin{align*}
\big(a+(h-1)c\big)-\big(b+(h-1)d\big)&=(a-b)+(h-1)(c-d)\\
&\geq (a-b)+(c-d)>(b-a)+(c-d)=-(a-c)+(b-d),
\end{align*}
which shows $\lambda_2=\ldots=\lambda_h>|\lambda_{h+k}|=\ldots=|\lambda_{hk}|$ and $\lambda_{h+1}=\ldots=\lambda_{h+k-1}>|\lambda_{h+k}|=\ldots=|\lambda_{hk}|$.
\end{proof}

\vspace{6mm}
Note that we have
\begin{align}\label{char_2nd_eigenvecs}
 f^{(s)}-\frac{|V_s|}{n}\cdot \mathbf{1}_n=  f^{(s)}-\frac{1}{h}\cdot \mathbf{1}_n=\frac{h-1}{h}v_{1+s}, \quad s\in [h-1],
\end{align}
where $f^{(s)}$ is the group-membership vector of $V_s$ and  $f^{(s)}-\frac{|V_s|}{n}$ is the vector  encountered in the second step of Algorithm~\ref{fair_SC_alg}
or Algorithm~\ref{fair_SC_alg_normalized}.

The next lemma provides an orthonormal basis of the eigenspace associated with eigenvalue $\lambda_{h+1}=\ldots=\lambda_{h+k-1}$.

\begin{lemma}\label{lemma_spectrum_h_groups_normalized}
An orthonormal basis of the eigenspace of  $\widetilde{\mathcal{W}}$ corresponding to the eigenvalue $\lambda_{h+1}=\ldots=\lambda_{h+k-1}$ is given by $\{n_1,\ldots,n_{k-1}\}$ with
\begin{align}\label{def_ni}
\renewcommand\arraystretch{1.5}
n_1=
\begin{pmatrix}
[(k-1)q_1]\\
[-q_1]\\
[-q_1]\\
[-q_1]\\
\vdots\\
[-q_1]\\
[-q_1]
\end{pmatrix},~~~
n_2=
\begin{pmatrix}
[0]\\
[(k-2)q_2]\\
[-q_2]\\
[-q_2]\\
\vdots\\
[-q_2]\\
[-q_2]
\end{pmatrix},~~~\ldots,~~~
n_i=
\begin{pmatrix}
[0]\\
\vdots\\
[0]\\
[(k-i)q_i]\\
[-q_i]\\
\vdots\\
[-q_i]
\end{pmatrix},~~~\ldots,~~~
n_{k-1}=\begin{pmatrix}
[0]\\
[0]\\
\vdots\\
[0]\\
[0]\\
[q_{k-1}]\\
[-q_{k-1}]
\end{pmatrix},
\end{align}
where for $z\in\R$, by $[z]$ we denote a block of size $\frac{n}{k}$ with all entries equaling $z$ and
\begin{align}\label{def_qi}
q_i=\frac{1}{\sqrt{\left(\frac{n}{k}(k-i)^2 +\frac{n}{k}(k-i)\right)}}, \quad i\in[k-1].
\end{align}
\end{lemma}

\begin{proof}
It is easy to verify
that any $n_i$ is indeed an eigenvector of $\widetilde{\mathcal{W}}$ with eigenvalue  $\lambda_{h+1}=\ldots=\lambda_{h+k-1}$, $i\in[k-1]$. Furthermore, we have
 \begin{align*}
 \|n_i\|^2=q_i^2\left(\frac{n}{k}(k-i)^2 +\frac{n}{k}(k-i)\right)=1, \quad i\in[k-1],
 \end{align*}
 and
  \begin{align*}
 \langle n_i,n_j \rangle=\frac{n}{k}\big(-q_i\cdot (k-j)q_j\big)+\frac{n}{k}(k-j)(q_i\cdot q_j)=0, \quad 1\leq i<j\leq n.
 \end{align*}
\end{proof}

\vspace{6mm}
Let $\mathcal{L}$ be the expected Laplacian matrix. We have $\mathcal{L}=\mathcal{D}-\mathcal{W}$, where $\mathcal{D}$ is the expected degree matrix.
The expected degree of  vertex $i$ in a random graph constructed according to our variant of the stochastic block model equals $\sum_{j\in[n]\setminus\{i\}}\mathcal{W}_{ij}= \lambda_1-a$ (with $\lambda_1$ defined in
\eqref{eigenvalues_h_groups}) and hence $\mathcal{D}=(\lambda_1-a)I_n$.

The following lemma characterizes the eigenvalues and eigenvectors of $Z^T\mathcal{L}Z$, where $Z\in\R^{n\times (n-h+1)}$ is the matrix computed
in the execution of Algorithm~\ref{fair_SC_alg}
or Algorithm~\ref{fair_SC_alg_normalized}.

\begin{lemma}\label{lemma_spectrum_ZtransLZ}
Let $Z\in\R^{n\times (n-h+1)}$ be any matrix whose columns form an orthonormal basis of the nullspace of $F^T$, where $F$ is the matrix that has the vectors
$f^{(s)}-\frac{|V_s|}{n}\cdot \mathbf{1}_n$, $s\in[h-1]$, as columns. Then the eigenvalues of $Z^T\mathcal{L}Z$ are
\begin{align*}
\lambda_1-\lambda_1,\lambda_1-\lambda_{h+1},\lambda_1-\lambda_{h+2},\ldots,\lambda_1-\lambda_{n}
\end{align*}
 with $\lambda_i$ defined in
\eqref{eigenvalues_h_groups}.
It is
\begin{align}\label{lemma_spectrum_ZtransLZ_cons1}
\begin{split}
&\lambda_1-\lambda_1=0,\\
&\lambda_1-\lambda_{h+1}=\lambda_1-\lambda_{h+2}=\ldots=\lambda_1-\lambda_{h+k-1},\\
&\lambda_1-\lambda_{h+k}=\lambda_1-\lambda_{h+k+1}=\ldots=\lambda_1-\lambda_{hk},\\
&\lambda_1-\lambda_{hk+1}=\lambda_1-\lambda_{hk+2}=\ldots=\lambda_1-\lambda_{n}=\lambda_1
\end{split}
\end{align}
with
\begin{align}\label{lemma_spectrum_ZtransLZ_cons2}
&\lambda_1-\lambda_1<\lambda_1-\lambda_{h+1}<\min\{\lambda_1-\lambda_{h+k},\lambda_1-\lambda_{hk+1}\},
\end{align}
so  that the $k$ smallest eigenvalues of $Z^T\mathcal{L}Z$ are $\lambda_1-\lambda_1,\lambda_1-\lambda_{h+1},\lambda_1-\lambda_{h+2},\ldots,\lambda_1-\lambda_{h+k-1}$.

Furthermore, there exists an orthonormal basis
$\{r_1,r_{h+1},r_{h+2},\ldots,r_{n}\}$
of eigenvectors of $Z^T\mathcal{L}Z$ with $r_i$ corresponding to eigenvalue $\lambda_1-\lambda_i$ such that
\begin{align*}
Zr_1=\mathbf{1}_n/\sqrt{n}\qquad\text{and}\qquad Zr_{h+i}=n_i, \quad i\in[k-1],
\end{align*}
with $n_i$ defined in \eqref{def_ni}.
\end{lemma}

\begin{proof}
Because of $Z^TZ=I_{(n-h+1)}$ we have
\begin{align*}
 Z^T\mathcal{L}Z=Z^T(\mathcal{D}-\mathcal{W})Z=Z^T\mathcal{D}Z-Z^T(\widetilde{\mathcal{W}}-aI_n)Z=(\lambda_1-a)I_n-Z^T\widetilde{\mathcal{W}}Z+aI_n=\lambda_1I_n-Z^T\widetilde{\mathcal{W}}Z.
\end{align*}
Let $\{u_1,\ldots,u_n\}$ be an orthonormal basis of eigenvectors of $\widetilde{\mathcal{W}}$ with $u_i$ corresponding to eigenvalue $\lambda_i$. According to Lemma~\ref{lemma_spectrum_h_groups} and Lemma~\ref{lemma_spectrum_h_groups_normalized} we can choose $u_1=\mathbf{1}_n/\sqrt{n}$ and $u_{h+i}=n_i$ for $i\in[k-1]$. We can write  $\widetilde{\mathcal{W}}$ as  $\widetilde{\mathcal{W}}=\sum_{i=1}^n\lambda_iu_iu_i^T$.

The nullspace of $F^T$, where $F$ is the matrix that has the vectors
$f^{(s)}-\frac{|V_s|}{n}\cdot \mathbf{1}_n$, $s\in[h-1]$, as columns, equals the orthogonal complement of $\{f^{(s)}-(|V_s|/n)\cdot \mathbf{1}_n, s\in[h-1]\}$. According to \eqref{char_2nd_eigenvecs}, the orthogonal complement of $\{f^{(s)}-(|V_s|/n)\cdot \mathbf{1}_n, s\in[h-1]\}$ equals the orthogonal complement of $\{v_{1+s}, s\in[h-1]\}$, with $v_i$ defined in  Lemma~\ref{lemma_spectrum_h_groups} and being an eigenvalue of  $\widetilde{\mathcal{W}}$ with eigenvalue $\lambda_i$. According to Lemma~\ref{lemma_spectrum_h_groups}, $\{v_{1+s}, s\in[h-1]\}$ is a basis of the eigenspace of $\widetilde{\mathcal{W}}$ corresponding to eigenvalue $\lambda_2=\lambda_3=\ldots=\lambda_h$, and hence the orthogonal complement of $\{v_{1+s}, s\in[h-1]\}$ equals the orthogonal complement of $\{u_2,\ldots,u_h\}$, which is the subspace spanned by $\{u_1,u_{h+1},u_{h+2},\ldots,u_n\}$. Let $U\in\R^{n\times (n-h+1)}$ be a matrix that has the vectors  $u_1,u_{h+1},u_{h+2},\ldots,u_n$ as columns (in this order). It follows that $U=ZR$ for some $R\in \R^{(n-h+1)\times (n-h+1)}$ with $R^TR=RR^T=I_{(n-h+1)}$.
It is
\begin{align*}
 Z^T\mathcal{L}Z=\lambda_1I_n-Z^T\widetilde{\mathcal{W}}Z=\lambda_1I_n-RU^T\left(\sum_{i=1}^n\lambda_iu_iu_i^T\right)UR^T.
\end{align*}

Let $r_1$ be the first column of $R$,  $r_{h+1}$ be the second column of $R$, $r_{h+2}$ be the third column of $R$, and so on. We have
\begin{align*}
 Z^T\mathcal{L}Zr_1&=\left[\lambda_1I_n-RU^T\left(\sum_{i=1}^n\lambda_iu_iu_i^T\right)UR^T\right]r_1\\
 &=\lambda_1 r_1-RU^T\left(\sum_{i=1}^n\lambda_iu_iu_i^T\right)Ue_1\\
 &=\lambda_1 r_1-RU^T\left(\sum_{i=1}^n\lambda_iu_iu_i^T\right)u_1\\
 &=\lambda_1 r_1-\lambda_1RU^Tu_1\\
 &=\lambda_1 r_1-\lambda_1Re_1\\
 &=(\lambda_1 -\lambda_1)r_1,
\end{align*}
where $e_1$ denotes the first natural basis vector. Similarly, we obtain $Z^T\mathcal{L}Zr_{h+i}=(\lambda_1 -\lambda_{h+i})r_{h+i}$, $i\in[n-h]$.
This proves that the eigenvalues of $Z^T\mathcal{L}Z$ are
$\lambda_1-\lambda_1,\lambda_1-\lambda_{h+1},\lambda_1-\lambda_{h+2},\ldots,\lambda_1-\lambda_{n}$. The claims in \eqref{lemma_spectrum_ZtransLZ_cons1} and \eqref{lemma_spectrum_ZtransLZ_cons2} immediately follow from Lemma~\ref{lemma_spectrum_h_groups}. Clearly, it is $Zr_1=u_1=\mathbf{1}_n/\sqrt{n}$ and $Zr_{h+i}=u_{h+i}=n_i$ for $i\in[k-1]$.
\end{proof}

\vspace{3mm}
We need one more simple lemma.

\begin{lemma}\label{lemma_distance_between_rows}
Let $T\in\R^{n\times k}$ be a matrix that contains the vectors $\mathbf{1}_n/\sqrt{n}, n_1,n_2,\ldots,n_{k-1}$, with $n_i$ defined in \eqref{def_ni},  as columns. For $i\in[n]$, let $t_i$ denote the $i$-th row of $T$. For all $i,j\in[n]$, we have $t_i=t_j$ if and only if the vertices $i$ and $j$ are in the same cluster~$C_l$. If the vertices $i$ and $j$ are not in the same cluster, then $\|t_i-t_j\|=\sqrt{2k/n}$.
\end{lemma}

\begin{proof}
This simply follows from the structure of the vectors $n_i$. It is, up to a permutation of the entries,
\begin{align*}
t_i=\left(\frac{1}{\sqrt{n}},-q_1,-q_2,\ldots,-q_{l-1},(k-l)q_l,0,0,\ldots,0\right),
\end{align*}
with $q_l$ defined in \eqref{def_qi},  for all $i\in[n]$ such that vertex $i$ is in cluster $C_l$, $l\in[k-1]$,  and
\begin{align*}
t_i=\left(\frac{1}{\sqrt{n}},-q_1,-q_2,\ldots,-q_{k-1}\right)
\end{align*}
for all $i\in[n]$ such that vertex $i$ is in cluster $C_k$.
It is easy to verify that $\|t_i-t_j\|^2=2k/n$ for all $i,j\in[n]$ such that the vertices~$i$ and $j$ are not in the same cluster.
\end{proof}

\vspace{8mm}
\textbf{Part 2: Deviation of $Z^TLZ$ from $Z^T\mathcal{L}Z$}

We want to obtain an upper bound on $\|Z^TLZ-Z^T\mathcal{L}Z\|$.
Because of $Z^TZ=I_{(n-h+1)}$,  it is $\|Z\|=\|Z^T\|=1$ and hence
\begin{align}\label{absch_getting_rid_Z}
\|Z^TLZ-Z^T\mathcal{L}Z\|\leq \|Z^T\|\cdot \|L-\mathcal{L}\|\cdot\|Z\|\leq \|L-\mathcal{L}\|.
\end{align}
We have
\begin{align*}
 \|L-\mathcal{L}\|=\|(D-W)-(\mathcal{D}-\mathcal{W})\|\leq \|D-\mathcal{D}\|+\|W-\mathcal{W}\|,
\end{align*}
with $\mathcal{D}=(\lambda_1-a)I_n$ as we have seen in Part 1. We
 bound both terms separately.

\begin{itemize}
\item Upper bound on $\|W-\mathcal{W}\|$:

\vspace{2mm}
Theorem 5.2 of \citet{lei2015} provides a bound on $\|W-\mathcal{W}\|$: assuming that $a\geq C \ln n/n$ for some $C>0$,
for every $r>0$
 there exists a constant $\const(C,r)$ such that
\begin{align}\label{absch_w}
\|W-\mathcal{W}\|\leq \const(C,r)\sqrt{a\cdot n}
\end{align}
with probability at least $1-n^{-r}$.

\item Upper bound on $\|D-\mathcal{D}\|$:

\vspace{2mm}
The matrix $D-\mathcal{D}$ is a diagonal matrix and hence $\|D-\mathcal{D}\|=\max_{i\in[n]}|D_{ii}-\mathcal{D}_{ii}|=\max_{i\in[n]}|D_{ii}-(\lambda_1-a)|$.
The random variable $D_{ii}=\sum_{j\in[n]\setminus\{i\}}\charfct[i\sim j]$, where $\charfct[i\sim j]$ denotes the indicator function of the event that there is an edge between vertices $i$ and $j$, is a sum of independent Bernoulli random variables. It is $\Ex[D_{ii}]=\lambda_1-a$.
For a fixed $i\in[n]$, we want to obtain an upper bound on $|D_{ii}-(\lambda_1-a)|=|D_{ii}-\Ex[D_{ii}]|$ and distinguish two cases:
\begin{enumerate}
 \item $a>\frac{1}{2}$:

 \vspace{2mm}
Hoeffding's inequality
\citep[e.g.,][Theorem 1]{lugosi_conc_ineq}
yields
\begin{align*}
\Pr[|D_{ii}-(\lambda_1-a)|\geq t]\leq  2\exp\left(-\frac{2t^2}{n}\right)
\end{align*}
for any $t> 0$. Choosing $t=\sqrt{2(r+1)}\sqrt{a\cdot n\ln n}$ for $r>0$, we have with $\const(r)=\sqrt{2(r+1)}$ that
\begin{align}\label{ungl_case_agr}
\Pr\left[|D_{ii}-(\lambda_1-a)|\geq \const(r)\cdot\sqrt{a\cdot n\ln n}\right]&\leq  2\exp\left(-4(r+1)a\ln n\right)
\leq n^{-(r+1)}.
\end{align}

 \item $a\leq \frac{1}{2}$:

 \vspace{2mm}
Bernstein's inequality \citep[e.g.,][Theorem 3]{lugosi_conc_ineq}
yields
\begin{align*}
\Pr\left[|D_{ii}-(\lambda_1-a)|> tn\right]\leq 2\exp\left(-\frac{nt^2}{2\left(\frac{\Var[D_{ii}]}{n}+\frac{t}{3}\right)}\right)
\end{align*}
for any $t> 0$. It is
\begin{align*}
\Var[D_{ii}]=\sum_{j\in[n]\setminus\{i\}}\Var\left[\charfct[i\sim j]\right]
=\sum_{j\in[n]\setminus\{i\}} \Pr[\charfct[i\sim j]](1-\Pr[\charfct[i\sim j]])\leq n a(1-a)\leq na
\end{align*}
since the function $x\mapsto x(1-x)$ is monotonically increasing on
$[0,1/2]$.
If we choose $t=\const\cdot\frac{\sqrt{a\cdot n \ln n}}{n}$ for some $\const>0$, assuming that $a\geq C \ln n/n$ for some $C>0$, we
have
\begin{align*}
\frac{\Var[D_{ii}]}{n}+\frac{t}{3}\leq a\left(1+\frac{\const}{3\sqrt{C}}\right)
\end{align*}
and hence
\begin{align*}
\Pr\left[|D_{ii}-(\lambda_1-a)|> \const\cdot\sqrt{a\cdot n \ln n}\right]\leq
 2\exp\left(-\frac{\const^2\cdot \ln n}{2+\frac{2\const}{3\sqrt{C}}}\right).
\end{align*}
Because of
\begin{align*}
\frac{\const^2}{2+\frac{2\const}{3\sqrt{C}}}\rightarrow \infty\quad \text{as}\const\rightarrow\infty,
\end{align*}
for every $r>0$ we can choose $\const=\const(C,r)$ large enough such that
$\const^2/\big(2+\frac{2\const}{3\sqrt{C}}\big)\geq 2(r+1)$ and
\begin{align}\label{ungl_case_asm}
\Pr\left[|D_{ii}-(\lambda_1-a)|> \const(C,r)\cdot\sqrt{a\cdot n \ln n}\right]\leq
n^{-(r+1)}.
\end{align}
\end{enumerate}

Choosing  $\const(C,r)$ as the maximum of $\const(r)$ encountered in \eqref{ungl_case_agr} and $\const(C,r)$ encountered in \eqref{ungl_case_asm}, we see that there exists   $\const(C,r)$ such that
\begin{align*}
\Pr\left[|D_{ii}-(\lambda_1-a)|> \const(C,r)\cdot\sqrt{a\cdot n \ln n}\right]\leq
n^{-(r+1)},
\end{align*}
no matter whether $a>1/2$ or $1/2\geq a \geq C \ln n/n$.
Applying a union bound we obtain
\begin{align*}
\Pr\left[\max_{i\in[n]}|D_{ii}-(\lambda_1-a)|>  \const(C,r)\cdot\sqrt{a\cdot n\ln n}\right]\leq n\cdot n^{-(r+1)}=n^{-r},
\end{align*}
and hence with probability at least $1-n^{-r}$ we have
\begin{align}\label{absch_d}
\|D-\mathcal{D}\|\leq \const(C,r)\sqrt{a\cdot n\ln n}.
\end{align}
\end{itemize}

From \eqref{absch_w} and \eqref{absch_d} we see that for every $r>0$ there
exists   $\const(C,r)$ such that
 with probability at least $1-n^{-r}$ we have
\begin{align}\label{absch_w_und_d}
\|W-\mathcal{W}\|\leq \const(C,r)\sqrt{a\cdot n}\qquad\text{and}\qquad \|D-\mathcal{D}\|\leq \const(C,r)\sqrt{a\cdot n\ln n}
\end{align}
and hence
\begin{align}\label{absch_l}
\|Z^TLZ-Z^T\mathcal{L}Z\|\leq\|L-\mathcal{L}\|\leq\|D-\mathcal{D}\|+\|W-\mathcal{W}\|\leq  \const(C,r)\sqrt{a\cdot n\ln n}.
\end{align}

\begin{figure}[t]
\centering
\includegraphics[scale=0.6]{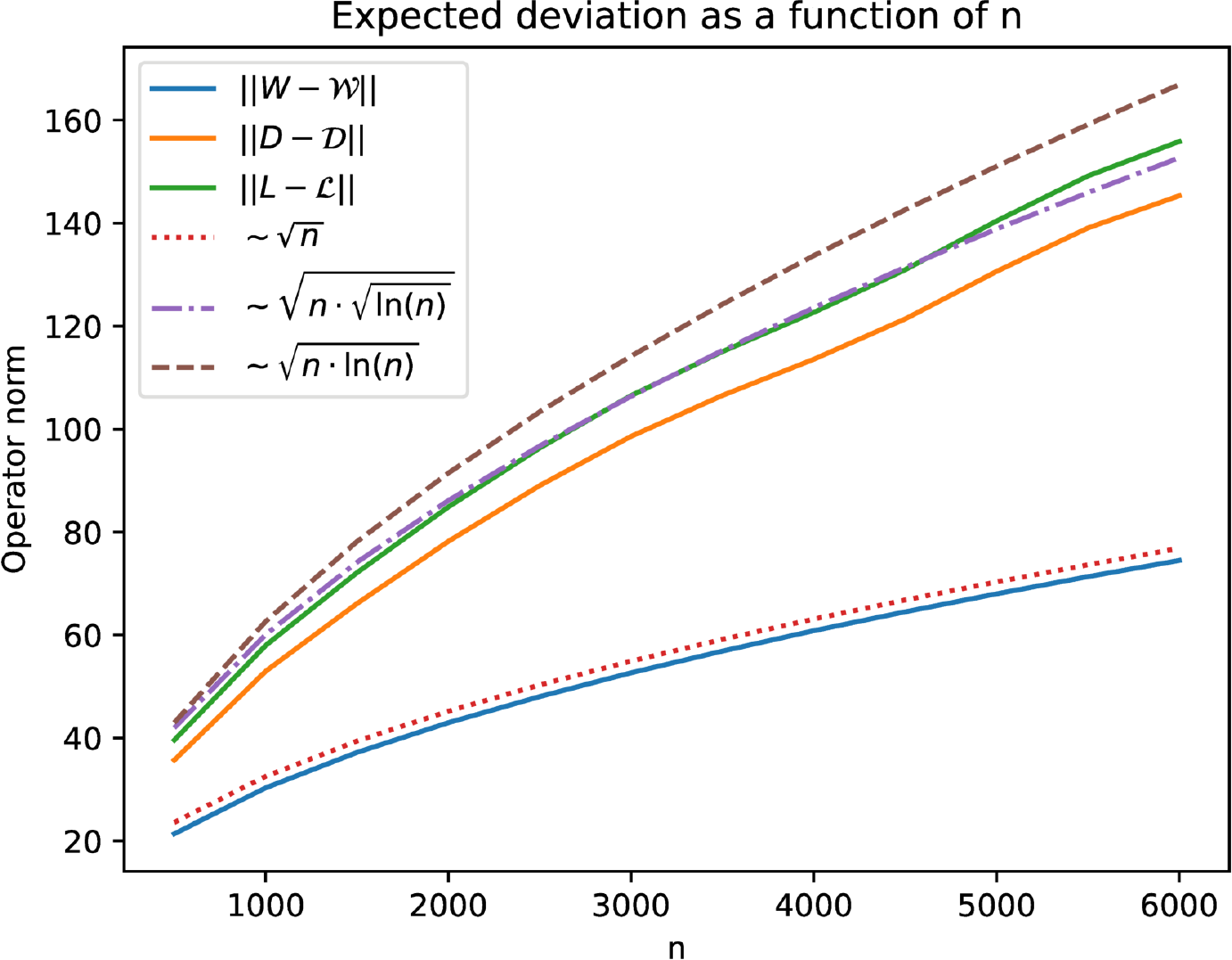}
\caption{Average  deviations  $\|W-\mathcal{W}\|$, $\|D-\mathcal{D}\|$ and $\|L-\mathcal{L}\|$  as a function of $n$ when $a=0.6,b=0.5,c=0.4,d=0.3$ are constant, $k=5$ and $h=2$. The average is computed over sampling the graph for 100 times.}\label{plot_deviation}
\end{figure}

\vspace{5mm}
For illustrative purposes, we show empirically that, in general,  our upper  bounds on  $\|W-\mathcal{W}\|$, $\|D-\mathcal{D}\|$ and $\|L-\mathcal{L}\|$ in  \eqref{absch_w_und_d} and \eqref{absch_l}, respectively, are tight, up to a factor of
at most
$\sqrt[4]{\ln n}$ in case of $\|D-\mathcal{D}\|$ and $\|L-\mathcal{L}\|$. The plot in Figure~\ref{plot_deviation} shows the observed deviations  $\|W-\mathcal{W}\|$, $\|D-\mathcal{D}\|$ and $\|L-\mathcal{L}\|$  as a function of $n$ when $a=0.6,b=0.5,c=0.4,d=0.3$ are constant, $k=5$ and $h=2$. The shown curves are average results, obtained from  sampling the graph for 100 times.

\vspace{8mm}
\textbf{Part 3: Proving Theorem \ref{theorem_SB_model} for Algorithm~\ref{fair_SC_alg} (unnormalized SC with fairness constraints)}

In the last step of Algorithm~\ref{fair_SC_alg} we apply $k$-means clustering to the rows of the matrix $ZY$, where $Y\in\R^{(n-h+1)\times k}$  contains some orthonormal eigenvectors corresponding to the $k$ smallest eigenvalues of $Z^TLZ$ as columns. We want to show that up to some orthogonal transformation, the rows of $ZY$  are close to the rows of $Z\mathcal{Y}$, where $\mathcal{Y}\in\R^{(n-h+1)\times k}$  contains some orthonormal eigenvectors corresponding to the $k$ smallest eigenvalues of $Z^T\mathcal{L}Z$ as columns. According to Lemma~\ref{lemma_spectrum_ZtransLZ}, we can choose $\mathcal{Y}$ in such a way that $Z\mathcal{Y}=T$ with $T$ as in Lemma~\ref{lemma_distance_between_rows}, that is $T$ contains the vectors  $\mathbf{1}_n/\sqrt{n}, n_1,n_2,\ldots,n_{k-1}$, with $n_i$ defined in \eqref{def_ni},  as columns.

We want to obtain an upper bound on $\min_{U\in\R^{k\times k}: U^TU=UU^T=I_k}\|Z\mathcal{Y}-ZYU\|_{F}$.
For any $U\in\R^{k\times k}$ with  $U^TU=UU^T=I_k$, because of $Z^TZ=I_{(n-h+1)}$ we have
\begin{align*}
\|Z\mathcal{Y}-ZYU\|^2_{F}=\|Z(\mathcal{Y}-YU)\|^2_{F}=\trace((\mathcal{Y}-YU)^TZ^TZ(\mathcal{Y}-YU))=\|\mathcal{Y}-YU\|^2_{F}
\end{align*}
and hence
\begin{align}\label{absch_popoloplo}
\min_{U\in\R^{k\times k}: U^TU=UU^T=I_k}\|Z\mathcal{Y}-ZYU\|_{F}=\min_{U\in\R^{k\times k}: U^TU=UU^T=I_k}\|\mathcal{Y}-YU\|_{F}.
\end{align}
We proceed similarly to \citet{lei2015}.
According to Proposition 2.2 and Equation (2.6) in \citet{vu2013} we have (note that the set of all orthogonal matrices $U\in\R^{k\times k}$ is a compact subset of $\R^{k\times k}$ and hence the infimum is indeed a minimum)
\begin{align}\label{absch_first_step}
\min_{U\in\R^{k\times k}: U^TU=UU^T=I_k}\|\mathcal{Y}-YU\|_{F}\leq\sqrt{2}\|\mathcal{Y}\mathcal{Y}^T(I_{(n-h+1)}-YY^T)\|_F\stackrel{\eqref{absch_frob_norm}}{\leq}\sqrt{2}\sqrt{k}\,\|\mathcal{Y}\mathcal{Y}^T(I_{(n-h+1)}-YY^T)\|.
\end{align}
According to Lemma~\ref{lemma_spectrum_ZtransLZ} the eigenvalues of $Z^T\mathcal{L}Z$ are
$\lambda_1-\lambda_1,\lambda_1-\lambda_{h+1},\lambda_1-\lambda_{h+2},\ldots,\lambda_1-\lambda_{n}$.
The $k$ smallest eigenvalues are
$\lambda_1-\lambda_1,\lambda_1-\lambda_{h+1},\lambda_1-\lambda_{h+2},\ldots,\lambda_1-\lambda_{h+k-1}$ and  the $(k+1)$-th smallest eigenvalue is either $\lambda_1-\lambda_{h+k}$ or $\lambda_1$. Hence, for the eigengap~$\gamma$ between the $k$-th and the $(k+1)$-th smallest eigenvalue we have
\begin{align*}
\gamma=\min\left\{(\lambda_1-\lambda_{h+k})-(\lambda_1-\lambda_{h+k-1}),\lambda_1-(\lambda_1-\lambda_{h+k-1})\right\}
=\min\left\{\frac{n}{k}(c-d),\frac{n}{kh}((a-b)+(h-1)(c-d))\right\}.
\end{align*}
It is
\begin{align}\label{absch_gamma}
\frac{n}{2k}(c-d)\leq \frac{n(h-1)}{hk}(c-d) \leq \gamma\leq \frac{n}{k}(c-d).
\end{align}
We want to show that
\begin{align}\label{absch_poiul}
\|\mathcal{Y}\mathcal{Y}^T(I_{(n-h+1)}-YY^T)\|\leq \frac{4}{\gamma}\|Z^T\mathcal{L}Z-Z^TLZ\|.
\end{align}
If $\|Z^T\mathcal{L}Z-Z^TLZ\|> \frac{\gamma}{4}$, then \eqref{absch_poiul} holds trivially because of
\begin{align*}
\|\mathcal{Y}\mathcal{Y}^T(I_{(n-h+1)}-YY^T)\|\leq \|\mathcal{Y}\mathcal{Y}^T\|\cdot\|I_{(n-h+1)}-YY^T\| =1\cdot 1=1.
\end{align*}
Assume that $\|Z^T\mathcal{L}Z-Z^TLZ\|\leq \frac{\gamma}{4}$ and let $\mu_1\leq \mu_2\leq\ldots\leq\mu_{n-h+1}$ be the eigenvalues of $Z^TLZ$.
Since $L$ is
positive semi-definite, so is $Z^TLZ$, and hence $\mu_1\geq 0$. Let $\lambda'_1\leq \lambda'_2\leq \ldots,\lambda'_{n-h+1}$ be the eigenvalues $\lambda_1-\lambda_1,\lambda_1-\lambda_{h+1},\lambda_1-\lambda_{h+2},\ldots,\lambda_1-\lambda_{n}$ of $Z^T\mathcal{L}Z$ in ascending order.
According to Weyl's Perturbation Theorem
\citep[e.g.,][Corollary III.2.6]{bhatia_book}
it is
\begin{align*}
|\mu_i-\lambda'_i|\leq\|Z^T\mathcal{L}Z-Z^TLZ\|\leq \frac{\gamma}{4}, \quad i\in[n-h+1].
\end{align*}
In particular, we have
\begin{align*}
\mu_1,\ldots,\mu_k\in \left[0,\lambda'_k+\frac{\gamma}{4}\right], \quad \mu_{k+1},\ldots,\mu_n \in \left[\lambda'_{k+1}-\frac{\gamma}{4},\infty\right)
\end{align*}
with $\left(\lambda'_{k+1}-\frac{\gamma}{4}\right)-\left(\lambda'_k+\frac{\gamma}{4}\right)=\frac{\gamma}{2}$. The Davis-Kahan sin$\Theta$ 
Theorem \citep[e.g.,][Theorem VII.3.1]{bhatia_book}
yields that
\begin{align*}
\|\mathcal{Y}\mathcal{Y}^T(I_{(n-h+1)}-YY^T)\|\leq \frac{2}{\gamma}\|Z^T\mathcal{L}Z-Z^TLZ\|
\end{align*}
and hence \eqref{absch_poiul}.
Combining \eqref{absch_popoloplo} to \eqref{absch_poiul}, we end up with
\begin{align}\label{absch_after_sin_kahan}
\min_{U\in\R^{k\times k}: U^TU=UU^T=I_k}\|Z\mathcal{Y}-ZYU\|_{F}\leq \frac{16\sqrt{k^3}}{n(c-d)}\|Z^T\mathcal{L}Z-Z^TLZ\|.
\end{align}

\vspace{3mm}
Using \eqref{absch_l} from Part~2, we see that with probability at least $1-n^{-r}$
\begin{align}\label{absch_ev1}
\min_{U\in\R^{k\times k}: U^TU=UU^T=I_k}\|Z\mathcal{Y}-ZYU\|_{F}\leq \const(C,r)\cdot\frac{\sqrt{k^3}}{c-d}\cdot\sqrt{\frac{a\cdot  \ln n}{n}}.
\end{align}
We use Lemma~5.3 in  \citet{lei2015} to complete the proof of Theorem~\ref{theorem_SB_model} for Algorithm~\ref{fair_SC_alg}. Assume that \eqref{absch_ev1} holds and let $U\in\R^{k\times k}$ be an orthogonal matrix attaining the minimum, that is we have
\begin{align}\label{absch_ev1_f}
\|Z\mathcal{Y}U^T-ZY\|_{F}=\|Z\mathcal{Y}-ZYU\|_{F}\leq \const(C,r)\cdot\frac{\sqrt{k^3}}{c-d}\cdot\sqrt{\frac{a\cdot  \ln n}{n}}.
\end{align}
As we have noted above, we can choose $\mathcal{Y}$ in such a way that $Z\mathcal{Y}=T$ with $T$ as in Lemma~\ref{lemma_distance_between_rows}. According to Lemma~\ref{lemma_distance_between_rows}, if we denote the $i$-th row of $T$ by $t_i$, then $t_i=t_j$ if the vertices $i$ and $j$ are in the same cluster and $\|t_i-t_j\|=\sqrt{2k/n}$ if the vertices $i$ and $j$ are not in the same cluster. Since multiplying $T$ by $U^T$ from the right side has the effect of applying an orthogonal transformation to the rows of $T$, the same properties are true for the matrix $TU^T$. Lemma~5.3 in  \citet{lei2015} guarantees that for any $\delta\leq \sqrt{2k/n}$, if
\begin{align}\label{condition_lemma53_lirin}
\frac{16+8M}{\delta^2}\|TU^T-ZY\|_F^2<\underbrace{|C_l|}_{=\frac{n}{k}},\quad l\in[k],
\end{align}
with $|C_l|$ being the size of cluster $C_l$, then a $(1+M)$-approximation algorithm for $k$-means clustering applied to the rows of the matrix $ZY$ returns  a clustering that misclassifies at most
\begin{align}\label{guaran_lemma53_lirin}
\frac{4(4+2M)}{\delta^2}\|TU^T-ZY\|_F^2
\end{align}
many vertices. If we choose $\delta= \sqrt{2k/n}$, then for a small enough $\widehat{C}_1=\widehat{C}_1(C,r)$
the condition~\eqref{condition_on_probabilities} implies
\eqref{condition_lemma53_lirin} because of  \eqref{absch_ev1_f}. Also, for a large enough $\widetilde{C}_1=\widetilde{C}_1(C,r)$, the expression~\eqref{guaran_lemma53_lirin} is upper bounded by the expression~\eqref{quant_guaran_theo}.

\vspace{8mm}
\textbf{Part 4: Proving Theorem \ref{theorem_SB_model} for Algorithm~\ref{fair_SC_alg_normalized} (normalized SC with fairness constraints)}

According to Part~2, for every $r>0$ there
exists   $\const(C,r)$ such that
 with probability at least $1-n^{-r}$ we have
\begin{align}\label{absch_first_part4}
 \|D-\mathcal{D}\|\leq \const(C,r)\sqrt{a\cdot n\ln n}.
\end{align}
Condition~\eqref{condition_on_probabilities_norm}, with a suitable $\widehat{C}_2=\widehat{C}_2(C,r)$, implies that in this case we also have
\begin{align}\label{absch_tretorn}
 \|D-\mathcal{D}\|\leq \frac{\lambda_1-a}{2}.
\end{align}
Let $\mu'_1,\ldots,\mu'_{n-h+1}$ denote the eigenvalues of $Z^TDZ$. It is $\mathcal{D}=(\lambda_1-a)I_n$ (see Part~1) and because of $Z^TZ=I_{(n-h+1)}$ we have $Z^T\mathcal{D}Z=(\lambda_1-a)I_{(n-h+1)}$. According to Weyl's Perturbation Theorem
\citep[e.g.,][Corollary III.2.6]{bhatia_book}
it is
\begin{align}\label{absch_mu_ZTDZ}
|\mu'_i-(\lambda_1-a)|\leq\|Z^TDZ-Z^T\mathcal{D}Z\|\leq \|D-\mathcal{D}\|, \quad i\in[n-h+1],
\end{align}
where the second inequality follows analogously to \eqref{absch_getting_rid_Z}. It follows from \eqref{absch_tretorn} that
\begin{align}\label{absch_ew_ZTDZ}
\mu'_i\geq \frac{\lambda_1 -a}{2}\stackrel{\eqref{condition_on_probabilities_norm}}{>}0, \quad i\in[n-h+1],
\end{align}
In particular, this shows that $Z^TDZ$ is positive definite and hence  Algorithm~\ref{fair_SC_alg_normalized} is well-defined.

Now we proceed similarly to Part~3.
In the last step of Algorithm~\ref{fair_SC_alg_normalized} we apply $k$-means clustering to the rows of the matrix $ZQ^{-1}X$, where $Q\in\R^{(n-h+1)\times(n-h+1)}$ is the
positive definite square root of $Z^TDZ$ and $X\in\R^{(n-h+1)\times k}$  contains some orthonormal eigenvectors corresponding to the $k$ smallest eigenvalues of $Q^{-1}Z^TLZQ^{-1}$ as columns.
We want to show that up to some orthogonal transformation, the rows of $ZQ^{-1}X$  are close to the rows of $Z\mathcal{Q}^{-1}\mathcal{X}$,
where $\mathcal{Q}\in\R^{(n-h+1)\times(n-h+1)}$ is the
positive definite square root of $Z^T\mathcal{D}Z$ and $\mathcal{X}\in\R^{(n-h+1)\times k}$  contains some orthonormal eigenvectors corresponding to the $k$ smallest eigenvalues of $\mathcal{Q}^{-1}Z^T\mathcal{L}Z\mathcal{Q}^{-1}$ as columns.
%
It is  $Z^T\mathcal{D}Z=(\lambda_1-a)I_{(n-h+1)}$.
Consequently, $\mathcal{Q}=\sqrt{\lambda_1-a}\cdot I_{(n-h+1)}$ and $\mathcal{Q}^{-1}=\frac{1}{\sqrt{\lambda_1-a}}\cdot I_{(n-h+1)}$ and it is  $\mathcal{Q}^{-1} Z^T\mathcal{L} Z\mathcal{Q}^{-1}=\frac{1}{\lambda_1-a}\cdot Z^T\mathcal{L} Z$. Hence, the eigenvalues of $\mathcal{Q}^{-1} Z^T\mathcal{L} Z\mathcal{Q}^{-1}$ are the eigenvalues of $Z^T\mathcal{L} Z$ rescaled by $(\lambda_1-a)^{-1}$ with the same eigenvectors as for $Z^T\mathcal{L} Z$.
 According to Lemma~\ref{lemma_spectrum_ZtransLZ}, we can choose $\mathcal{X}$ in such a way that $Z\mathcal{Q}^{-1}\mathcal{X}=\frac{1}{\sqrt{\lambda_1-a}}\cdot Z\mathcal{X}=\frac{1}{\sqrt{\lambda_1-a}}\cdot T$ with $T$ as in Lemma~\ref{lemma_distance_between_rows}, that is $T$ contains the vectors  $\mathbf{1}_n/\sqrt{n}, n_1,n_2,\ldots,n_{k-1}$, with $n_i$ defined in \eqref{def_ni},  as columns.

 We want to obtain an upper bound on $\min_{U\in\R^{k\times k}: U^TU=UU^T=I_k}\|Z\mathcal{Q}^{-1}\mathcal{X}-ZQ^{-1}XU\|_{F}$. Analogously to
\eqref{absch_popoloplo} we obtain
\begin{align*}
\min_{U\in\R^{k\times k}: U^TU=UU^T=I_k}\|Z\mathcal{Q}^{-1}\mathcal{X}-ZQ^{-1}XU\|_{F}=\min_{U\in\R^{k\times k}: U^TU=UU^T=I_k}\|\mathcal{Q}^{-1}\mathcal{X}-Q^{-1}XU\|_{F}.
\end{align*}
The rank of both $\mathcal{Q}^{-1}\mathcal{X}$ and $Q^{-1}XU$ equals $k$ and hence the rank of $\mathcal{Q}^{-1}\mathcal{X}-Q^{-1}XU$ is not greater than $2k$. We have
\begin{align*}
\|\mathcal{Q}^{-1}\mathcal{X}-Q^{-1}XU\|_F\stackrel{\eqref{absch_frob_norm}}{\leq} \sqrt{2k}\cdot \|\mathcal{Q}^{-1}\mathcal{X}-Q^{-1}XU\|\leq \sqrt{2k}\cdot\|\mathcal{Q}^{-1}\|\cdot \|\mathcal{X}-XU\|+\sqrt{2k}\cdot\|\mathcal{Q}^{-1}-Q^{-1}\|\cdot \|XU\|
\end{align*}
with
$\|\mathcal{Q}^{-1}\|=\frac{1}{\sqrt{\lambda_1-a}}$
and
$\|XU\|=1$ because of $X^TX=I_k$ and $U^TU=I_k$.
Hence
\begin{align}\label{central_absch_normalized}
\begin{split}
\min_{U:\, U^TU=UU^T=I_k}\|Z\mathcal{Q}^{-1}\mathcal{X}-ZQ^{-1}XU\|_{F}
\leq  \frac{\sqrt{2k}}{\sqrt{\lambda_1-a}}\cdot \min_{U: \,U^TU=UU^T=I_k}\|\mathcal{X}-XU\|+\sqrt{2k}\cdot\|\mathcal{Q}^{-1}-Q^{-1}\|.
  \end{split}
 \end{align}
 Because of \eqref{absch_frob_norm} we have
 \begin{align}\label{apply_absch_frob_in_normproof}
 \min_{U\in\R^{k\times k}:U^TU=UU^T=I_k}\|\mathcal{X}-XU\|\leq \min_{U\in\R^{k\times k}:U^TU=UU^T=I_k}\|\mathcal{X}-XU\|_F
 \end{align}
and similarly to how we obtained the bound \eqref{absch_after_sin_kahan} in Part~2, we can show that
 \begin{align}\label{absch_sin_kahan_norm}
  \min_{U\in\R^{k\times k}:U^TU=UU^T=I_k}\|\mathcal{X}-XU\|_F\leq \frac{16\sqrt{k^3}(\lambda_1-a)}{n(c-d)}\|\mathcal{Q}^{-1}Z^T\mathcal{L}Z\mathcal{Q}^{-1}-Q^{-1}Z^TLZQ^{-1}\|.
 \end{align}

Before looking at $\|\mathcal{Q}^{-1}Z^T\mathcal{L}Z\mathcal{Q}^{-1}-Q^{-1}Z^TLZQ^{-1}\|$
let us first look at the second term in \eqref{central_absch_normalized}. Because $Q^{-1}$ is symmetric and $\mathcal{Q}^{-1}=\frac{1}{\sqrt{\lambda_1-a}}\cdot I_{(n-h+1)}$ we have
 \begin{align*}
 \|\mathcal{Q}^{-1}-Q^{-1}\|=\max\left\{\left|\nu_i-\frac{1}{\sqrt{\lambda_1-a}}\right|:\nu_i\text{ is an eigenvalue of }Q^{-1}\right\}.
 \end{align*}
 It is $Q^2=Z^TDZ$. Denoting the eigenvalues of $Z^TDZ$ by $\mu'_1,\ldots,\mu'_{n-h+1}$ (note that all of them are greater than zero according to \eqref{absch_ew_ZTDZ}),  the eigenvalues of $Q^{-1}$ are $1/\sqrt{\mu'_1},\ldots,1/\sqrt{\mu'_{n-h+1}}$.
For any $z_1,z_2>0$ we have
\begin{align}\label{absch_wurzel}
|\sqrt{z_1}-\sqrt{z_2}|=\frac{|(\sqrt{z_1}-\sqrt{z_2})(\sqrt{z_1}+\sqrt{z_2})|}{\sqrt{z_1}+\sqrt{z_2}}=\frac{|z_1-z_2|}{\sqrt{z_1}+\sqrt{z_2}}\leq\frac{|z_1-z_2|}{\sqrt{z_2}}
\end{align}
and
\begin{align}\label{absch_wurzel2}
\left| \frac{1}{\sqrt{z_1}}-\frac{1}{\sqrt{z_2}}\right|=\frac{|\sqrt{z_1}-\sqrt{z_2}|}{\sqrt{z_1}\sqrt{z_2}}\stackrel{\text{ for  $z_1\geq \frac{z_2}{2}$ }}{\leq} \frac{\sqrt{2}\cdot|\sqrt{z_1}-\sqrt{z_2}|}{z_2}\stackrel{\eqref{absch_wurzel}}{\leq}\frac{\sqrt{2}\cdot |z_1-z_2|}{\sqrt{z_2^{\,3}}}.
\end{align}
According to \eqref{absch_ew_ZTDZ} we have $\mu'_i\geq \frac{\lambda_1 -a}{2}>0$, $i\in[n-h+1]$, and hence
\begin{align*}
\left| \frac{1}{\sqrt{\mu'_i}}-\frac{1}{\sqrt{\lambda_1-a}}\right|\stackrel{\eqref{absch_wurzel2}}{\leq} \frac{\sqrt{2}\cdot |\mu'_i-(\lambda_1 -a)|}{\sqrt{(\lambda_1 -a)^{3}}}\stackrel{\eqref{absch_mu_ZTDZ}}{\leq} \frac{\sqrt{2}\cdot \|D-\mathcal{D}\|}{\sqrt{(\lambda_1 -a)^{3}}}, \quad i\in[n-h+1],
\end{align*}
and
\begin{align}\label{absch_q_inverse}
 \|\mathcal{Q}^{-1}-Q^{-1}\|\leq \frac{\sqrt{2}\cdot \|D-\mathcal{D}\|}{\sqrt{(\lambda_1 -a)^{3}}}.
 \end{align}

 Let us now look at $\|\mathcal{Q}^{-1}Z^T\mathcal{L}Z\mathcal{Q}^{-1}-Q^{-1}Z^TLZQ^{-1}\|$.
It is
\begin{align}\label{retw2}
\begin{split}
 &\|\mathcal{Q}^{-1} Z^T\mathcal{L} Z\mathcal{Q}^{-1}-Q^{-1} Z^TL ZQ^{-1}\|
 \leq  \|\mathcal{Q}^{-1} -Q^{-1}\|\cdot \| Z^T\mathcal{L} Z\|\cdot\|\mathcal{Q}^{-1}\|+\\
 &~~~~~~~~~~~~~~~~~~~~~~~~~~~~~~~~~~~\|Q^{-1}\|\cdot \| Z^T\mathcal{L} Z- Z^TL Z\|\cdot\|\mathcal{Q}^{-1}\|
+ \|Q^{-1}\|\cdot\| Z^TL Z\|\cdot \|\mathcal{Q}^{-1}-Q^{-1}\|.
\end{split}
\end{align}
It is $\|\mathcal{Q}^{-1}\|=\frac{1}{\sqrt{\lambda_1-a}}$. According to Lemma~\ref{lemma_spectrum_ZtransLZ}, the largest eigenvalue of $Z^T\mathcal{L} Z$ is $\lambda_1$ or $\lambda_1-\lambda_{hk}$, where $\lambda_1-\lambda_{hk}\leq 2\lambda_1$ according to Lemma~\ref{lemma_spectrum_h_groups}. Consequently,  $\| Z^T\mathcal{L} Z\|\leq 2\lambda_1$.
It is
\begin{align*}
\|Q^{-1}\|\leq \|\mathcal{Q}^{-1}-Q^{-1}\|+\|\mathcal{Q}^{-1}\|\stackrel{\eqref{absch_q_inverse}}{\leq}\frac{\sqrt{2}\cdot \|D-\mathcal{D}\|}{\sqrt{(\lambda_1 -a)^{3}}}+\frac{1}{\sqrt{\lambda_1-a}}
\end{align*}
and
\begin{align*}
  \|Z^TL Z\|\leq   \| Z^TL Z-Z^T\mathcal{L} Z\|+    \|Z^T\mathcal{L} Z\|\stackrel{\eqref{absch_getting_rid_Z}}{\leq} \|L-\mathcal{L}\|+2\lambda_1.
\end{align*}
 It follows that
 \begin{align}\label{retw324242}
\begin{split}
 &\|\mathcal{Q}^{-1} Z^T\mathcal{L} Z\mathcal{Q}^{-1}-Q^{-1} Z^TL ZQ^{-1}\|
 \leq  \frac{4\lambda_1\cdot \|D-\mathcal{D}\|}{(\lambda_1 -a)^2}+\left(\frac{\sqrt{2}\cdot \|D-\mathcal{D}\|}{(\lambda_1 -a)^{2}}+\frac{1}{\lambda_1-a}\right)\cdot\|\mathcal{L} -L \|+\\
 &~~~~~~~~~~~~~~~~~~~~~~~~~~~~~~~~~~~~~~~~~~~~\left( \frac{2\cdot \|D-\mathcal{D}\|^2}{(\lambda_1 -a)^{3}}+\frac{\sqrt{2}\cdot \|D-\mathcal{D}\|}{(\lambda_1 -a)^{2}}\right)\cdot(\|L-\mathcal{L}\|+2\lambda_1)\\
 &~~~~~~~~~~\leq \frac{8\lambda_1\cdot \|D-\mathcal{D}\|}{(\lambda_1 -a)^2}+\left(\frac{4\cdot \|D-\mathcal{D}\|}{(\lambda_1 -a)^{2}}+\frac{1}{\lambda_1-a}\right)\cdot\|\mathcal{L} -L \|+\frac{2\cdot \|D-\mathcal{D}\|^2}{(\lambda_1 -a)^{3}}\cdot(\|L-\mathcal{L}\|+2\lambda_1).
\end{split}
\end{align}
 If \eqref{absch_w_und_d} and \eqref{absch_l} hold, then, after combining \eqref{central_absch_normalized}, \eqref{apply_absch_frob_in_normproof}, \eqref{absch_sin_kahan_norm}, \eqref{absch_q_inverse},  \eqref{retw324242} and using that $\lambda_1-a>\lambda_1/2$, which follows from  \eqref{condition_on_probabilities_norm}, we end up with
 \begin{align*}
 \min_{U\in\R^{k\times k}: U^TU=UU^T=I_k}\|Z\mathcal{Q}^{-1}\mathcal{X}-ZQ^{-1}XU\|_{F}
 \leq \frac{\const(C,r) \cdot k^2}{n(c-d)\sqrt{\lambda_1 -a}}\left(\sqrt{a\cdot n\ln n}+\frac{a\cdot n\ln n}{\lambda_1 -a}+\frac{(a\cdot n\ln n)^{3/2}}{(\lambda_1 -a)^{2}}\right)+\\
 \frac{\const(C,r)}{\sqrt{\lambda_1 -a}}\cdot\frac{\sqrt{k}\cdot \sqrt{a\cdot n\ln n}}{\lambda_1 -a}
 \end{align*}
 for some
 $\const(C,r)$.
Using that  $\sqrt{a\cdot n\ln n}\leq\sqrt{k\cdot a\cdot n\ln n}\leq \frac{\widehat{C}_2}{1+M}(\lambda_1-a)\leq \widehat{C}_2(\lambda_1-a)$ due to \eqref{condition_on_probabilities_norm}, for some $\widehat{C}_2=\widehat{C}_2(C,r)$ that we will specify shortly (we will choose it smaller than $1$), we can simplify this bound such that
\begin{align}\label{absch_ev_1_norm}
\min_{U\in\R^{k\times k}: U^TU=UU^T=I_k}\|Z\mathcal{Q}^{-1}\mathcal{X}-ZQ^{-1}XU\|_{F}
 \leq \frac{\const(C,r) \cdot k^2}{n(c-d)\sqrt{\lambda_1 -a}}\cdot\sqrt{a\cdot n\ln n}+ \frac{\const(C,r)}{\sqrt{\lambda_1 -a}}\cdot\frac{\sqrt{k}\cdot \sqrt{a\cdot n\ln n}}{\lambda_1 -a}.
 \end{align}

 \vspace{3mm}
Similarly to Part~3, we use Lemma~5.3 in  \citet{lei2015} to complete the proof of Theorem~\ref{theorem_SB_model} for Algorithm~\ref{fair_SC_alg_normalized}. Assume that \eqref{absch_ev_1_norm} holds and let $U\in\R^{k\times k}$ be an orthogonal matrix attaining the minimum, that is we have
\begin{align}\label{absch_ev1_f_norm}
\begin{split}
\|Z\mathcal{Q}^{-1}\mathcal{X}U^T-ZQ^{-1}X\|_{F}&=\|Z\mathcal{Q}^{-1}\mathcal{X}-ZQ^{-1}XU\|_{F} \\
& \leq \frac{\const(C,r) \cdot k^2}{n(c-d)\sqrt{\lambda_1 -a}}\cdot\sqrt{a\cdot n\ln n}+ \frac{\const(C,r)}{\sqrt{\lambda_1 -a}}\cdot\frac{\sqrt{k}\cdot \sqrt{a\cdot n\ln n}}{\lambda_1 -a}.
\end{split}
\end{align}
As we have noted above, we can choose $\mathcal{X}$ in such a way that $Z\mathcal{Q}^{-1}\mathcal{X}=\frac{1}{\sqrt{\lambda_1-a}}\cdot T$ with $T$ as in Lemma~\ref{lemma_distance_between_rows}. According to Lemma~\ref{lemma_distance_between_rows}, if we denote the $i$-th row of $\frac{1}{\sqrt{\lambda_1-a}}\cdot T$ by $\tilde{t}_i$, then $\tilde{t}_i=\tilde{t}_j$ if the vertices $i$ and $j$ are in the same cluster and
$\|\tilde{t}_i-\tilde{t}_j\|=\sqrt{\frac{2k}{n(\lambda_1-a)}}$
if the vertices $i$ and $j$ are not in the same cluster. Since multiplying $\frac{1}{\sqrt{\lambda_1-a}}\cdot T$ by $U^T$ from the right side has the effect of applying an orthogonal transformation to the rows of $\frac{1}{\sqrt{\lambda_1-a}}\cdot T$, the same properties are true for the matrix $\frac{1}{\sqrt{\lambda_1-a}}\cdot TU^T$. Lemma~5.3 in  \citet{lei2015} guarantees that for any $\delta\leq \sqrt{\frac{2k}{n(\lambda_1-a)}}$, if
\begin{align}\label{condition_lemma53_lirin_norm}
\frac{16+8M}{\delta^2}\left\|\frac{1}{\sqrt{\lambda_1-a}}\cdot TU^T-ZQ^{-1}X\right\|_F^2<\underbrace{|C_l|}_{=\frac{n}{k}},\quad l\in[k],
\end{align}
with $|C_l|$ being the size of cluster $C_l$, then a $(1+M)$-approximation algorithm for $k$-means clustering applied to the rows of the matrix $ZQ^{-1}X$ returns  a clustering that misclassifies at most
\begin{align}\label{guaran_lemma53_lirin_norm}
\frac{4(4+2M)}{\delta^2}\left\|\frac{1}{\sqrt{\lambda_1-a}}\cdot TU^T-ZQ^{-1}X\right\|_F^2
\end{align}
many vertices. If we choose $\delta= \sqrt{\frac{2k}{n(\lambda_1-a)}}$, then for a small enough $\widehat{C}_2=\widehat{C}_2(C,r)$
(chosen smaller than 1 and also so small that \eqref{absch_first_part4} implies \eqref{absch_tretorn}),  the condition~\eqref{condition_on_probabilities_norm} implies
\eqref{condition_lemma53_lirin_norm} because of \eqref{absch_ev1_f_norm}. Also, for a large enough $\widetilde{C}_2=\widetilde{C}_2(C,r)$ the expression~\eqref{guaran_lemma53_lirin_norm} is upper bounded by the expression~\eqref{quant_guaran_theo_norm}.

 \vspace{6mm}

\section{Why Running Standard Spectral Clustering on Each Group $V_s$ Separately is not a Good Idea}\label{appendix_baseline_strategy}

\begin{figure}[t]
\centering
\includegraphics[scale=0.6]{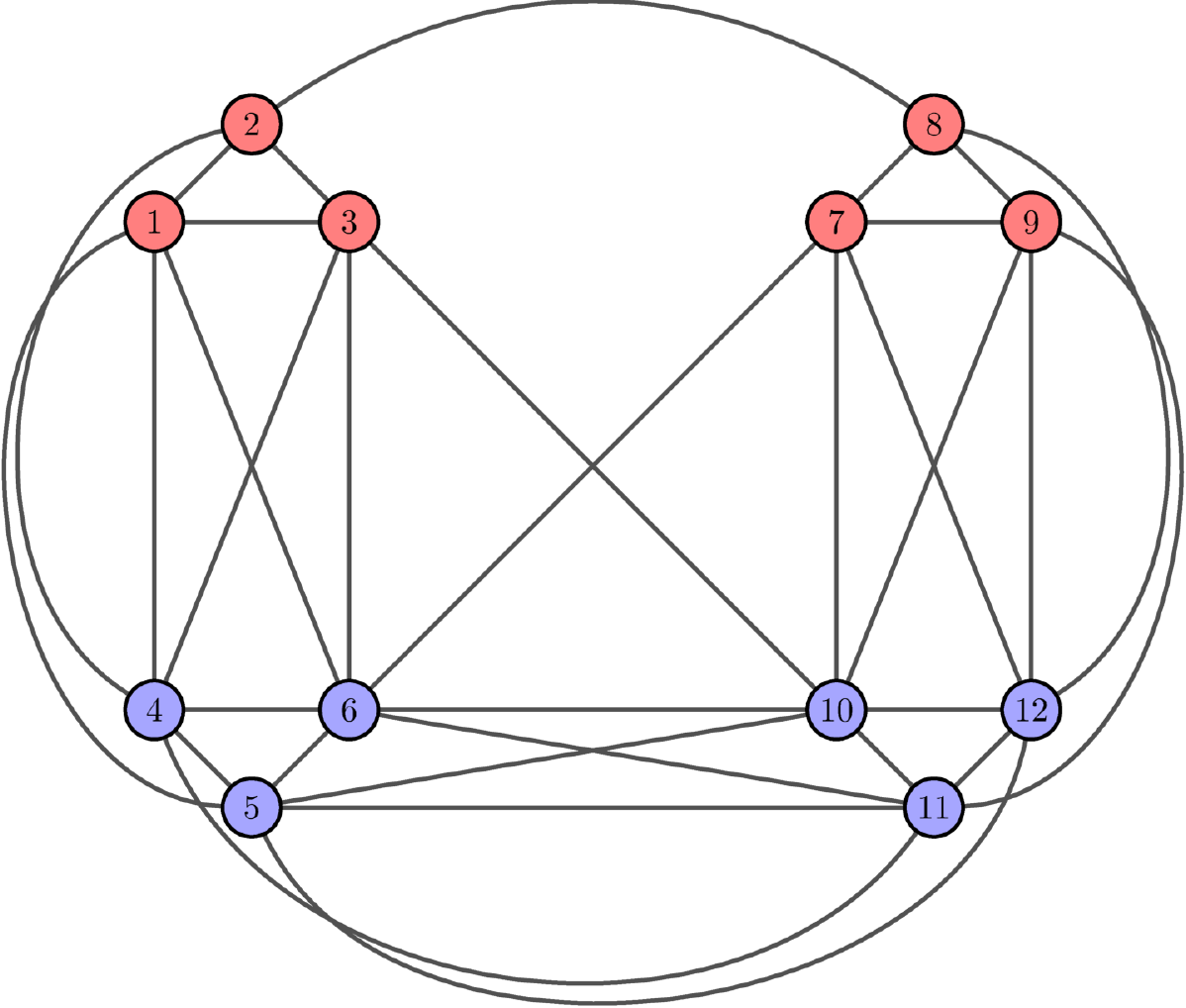}
\caption{Example of a graph for which both standard spectral clustering and our fair versions are able to recover the fair meaningful ground-truth clustering while a naive approach that runs standard spectral clustering on each group separately fails to do so. It is $V=[12]$, $V_1=\{1,2,3,7,8,9\}$,  $V_2=\{4,5,6,10,11,12\}$ and the fair ground-truth clustering is $V=\{1,2,3,4,5,6\}\dot{\cup}\{7,8,9,10,11,12\}$.}\label{pic_counterex}
\end{figure}

One might think that the following was a good  idea for partitioning $V=V_1\dot{\cup}\ldots \dot{\cup}V_h$ into $k$ clusters such that every cluster has a high balance value: we could try to  run  standard spectral clustering with $k$ clusters on each of the groups $V_s$, $s\in[h]$, separately and then to merge the $k\cdot h$ many clusters to end up with $k$  clusters.

The graph shown in Figure~{\ref{pic_counterex} illustrates that such an approach, in general, fails to recover an underlying fair ground-truth clustering, even when standard spectral clustering succeeds. We have $V=[12]$ and two  groups $V_1=\{1,2,3,7,8,9\}$ (shown in red) and $V_2=\{4,5,6,10,11,12\}$ (shown in blue). We want to partition $V$ into two clusters. It can be verified that a clustering with minimum RatioCut value is given by $V=\{1,2,3,4,5,6\}\dot{\cup}\{7,8,9,10,11,12\}$ and that this clustering is found by running standard spectral clustering. This clustering is perfectly fair with $\bal(\{1,2,3,4,5,6\})=\bal(\{7,8,9,10,11,12\})=1$ and is also returned by our fair versions of spectral clustering.
Let us now look at the idea of running  standard spectral clustering on $V_1$ and $V_2$ separately:
when running spectral clustering
on the subgraph induced by $V_1$,  we obtain the clustering $V_1=\{1,2,3\}\dot{\cup}\{7,8,9\}$ as we would hope for.
However, in the subgraph induced by $V_2$ the clustering $V_2=\{4,5,6\}\dot{\cup}\{10,11,12\}$  does not have minimum RatioCut value and is not returned  by spectral clustering. Consequently, no matter how we merge the two clusters for $V_1$ and the two clusters for $V_2$, we do not end up with the  clustering $V=\{1,2,3,4,5,6\}\dot{\cup}\{7,8,9,10,11,12\}$.

Note that for these findings to hold we do not require the specific graph shown in Figure~{\ref{pic_counterex}. The key is its structure: let
$V_1=\{1,2,3,7,8,9\}$,  $V_2=\{4,5,6,10,11,12\}$, $C_1=\{1,2,3,4,5,6\}$ and $C_2=\{7,8,9,10,11,12\}$. Then the graph looks like a realization of the following random graph model: as in our variant of the stochastic block model introduced in Section~\ref{section_SBmodel}, two vertices~$i$ and $j$ are connected with an edge with a certain probability $\Pr(i,j)$, which is now given by
\begin{align*}
\Pr(i,j)=
\begin{cases}
  a, & i,j\in C_1~\vee~ i,j\in C_2 ~\vee~ i,j\in V_2, \\
b, & \text{else},
    \end{cases}
    \end{align*}
with $a$ large and $b$ small.

\end{document}